\pdfoutput=1
\documentclass[11pt]{article}

\usepackage[margin=1in]{geometry}
\usepackage{amsmath}
\usepackage{amssymb}
\usepackage{amsthm}
\usepackage{mathtools}
\usepackage{mathrsfs}
\usepackage{lmodern}
\usepackage[T1]{fontenc}
\usepackage{microtype}
\usepackage{enumitem}
\usepackage{booktabs}
\usepackage{graphicx}
\usepackage{xcolor}
\usepackage[round,authoryear]{natbib}
\usepackage[colorlinks=true,linkcolor=blue!60!black,citecolor=blue!60!black,urlcolor=blue!60!black]{hyperref}
\usepackage{cleveref}
\usepackage{pgfplots}
\usepackage{pgfplotstable}
\pgfplotsset{compat=1.17}

\usepackage{float}
\newfloat{algorithm}{t}{loa}[section]
\floatname{algorithm}{Algorithm}
\crefname{algorithm}{Algorithm}{Algorithms}
\Crefname{algorithm}{Algorithm}{Algorithms}

\providecommand{\Description}[1]{}

\theoremstyle{plain}
\newtheorem{theorem}{Theorem}[section]

\newtheorem{proposition}[theorem]{Proposition}
\newtheorem{lemma}[theorem]{Lemma}
\newtheorem{corollary}[theorem]{Corollary}

\theoremstyle{definition}
\newtheorem{definition}[theorem]{Definition}
\newtheorem{example}[theorem]{Example}
\theoremstyle{plain}
\newtheorem{remark}[theorem]{Remark}
\crefname{fact}{Fact}{Facts}
\Crefname{fact}{Fact}{Facts}

\newcommand{\Xcal}{\mathcal{X}}
\newcommand{\Ycal}{\mathcal{Y}}
\newcommand{\Fcal}{\mathcal{F}}
\newcommand{\Sel}{\mathsf{Sel}}
\newcommand{\SNL}{S_{\mathrm{NL}}}
\newcommand{\JNL}{J_{\mathrm{NL}}}
\newcommand{\fNL}{f_{\mathrm{NL}}}
\newcommand{\err}{\mathrm{err}}
\newcommand{\egora}{\epsilon_g}
\newcommand{\Diam}{\mathrm{diam}}
\renewcommand{\Pr}{\mathbb{P}}
\newcommand{\E}{\mathbb{E}}
\newcommand{\eps}{\varepsilon}
\newcommand{\LLM}{\mathrm{LLM}_\theta}

\newcounter{stassumption}

\newtheorem{assumption}{Assumption}

\crefname{assumption}{Assumption}{Assumptions}
\Crefname{assumption}{Assumption}{Assumptions}

\begin{document}

\title{NL-PAC: Specification Ambiguity and Certified Minimax Risk Floors
       in LLM-Mediated Supervision}
\author{Berkay Anahtarci\\
  \small Department of Mathematical Engineering, \"Ozye\u{g}in University,
  Istanbul, T\"urkiye\\
  \small \texttt{berkay.anahtarci@ozyegin.edu.tr}}
\date{}

\maketitle

\begin{abstract}
Large language models increasingly provide labels, evaluations, and feedback
for tasks specified in natural language. When a specification admits multiple
readings but the supervision channel does not reveal which is operative,
additional labels reduce sampling error without resolving the resulting
identification problem. We introduce Natural Language PAC (NL-PAC), a
framework that uses a fixed model's thresholded decoding law to define
admissible labels and candidate targets. The probability that multiple labels
are admissible equals the diameter of the pointwise-admissible target class,
and under target-blind supervision every learner incurs worst-case risk of at
least half this diameter, at every sample size; the exact randomized minimax
risk over this class is attained by a data-independent strategy. Finite-sample
confidence bounds make these quantities certifiable from held-out unlabeled
inputs. In a frozen Qwen~2.5--3B audit, one prespecified prompt yields a
positive model-relative certificate, whereas a paraphrase and exact-rule
controls yield zero. A held-out bridge audit finds that supplied candidate
reading clauses fail the admissibility condition needed to transfer the
certificate to coherent readings. The guarantee is specific to the audited
model, prompt, threshold, and input distribution; extending it to human
interpretations requires external validation.
\end{abstract}

\section{Introduction}
\label{sec:intro}

Modern learning pipelines increasingly use large-language-model (LLM)
judgments in place of human evaluation or labeling, including in
LLM-as-a-judge systems~\citep{liu2023geval,kim2023prometheus,gu2024survey}.
Prompts and scoring rubrics specify these tasks in natural language and may
admit multiple interpretations. The supervising model can adopt one
interpretation without revealing which one governs its judgments. We call the supervision channel \emph{target-blind}
when its observation law does not reveal which reading is operative.
This regime is common in practice: judge models are typically served through
proprietary APIs whose internals are opaque to the
analyst~\citep{lamalfa2024lmaas}, so the operative reading is unobservable by
construction. Target blindness refers to this hidden reading; whether the model's
decoding probabilities are exposed is a separate matter that later governs the
audit's access mode. The resulting error is interpretive rather than statistical: additional
observations from the same unresolved channel do not identify the operative
reading, although a more informative channel can. We ask how much minimax risk
remains when the channel defining the admissible targets is also the learner's
only source of supervision, and whether that risk floor can be certified from
unlabeled inputs. We formalize
the setting as \emph{Natural Language PAC} (NL-PAC), characterize when such a
channel creates an irreducible risk floor, and turn that floor into a
finite-sample certificate.

Classical statistical learning theory does not model this coupling between
task specification and the supervision channel: in the
Probably Approximately Correct (PAC) framework of \citet{valiant1984theory},
the analyst fixes the instance space, distribution, target, and hypothesis
class, and the learner pays only the statistical price of estimating a target
inside that specification. This is the separation underlying VC theory and
computational learning theory generally~\citep{kearns1994introduction}.
Language-model systems collapse this separation: the task is given as a
natural-language instruction, and the supervision signal is produced by a
model interpreting that same instruction~\citep{liu2023ambiguity}. NL-PAC
makes this specification--channel pair the object of analysis.

Formally, a fixed model, prompt, and threshold induce at each input an
\emph{admissible label set} $A_\tau(x)$; the overlap mass
$D^\star_\tau=\Pr(|A_\tau(X)|\ge2)$ is the diameter of the pointwise-admissible
core of the induced candidate class, and under target-blind supervision every learner incurs
worst-case risk at least $D^\star_\tau/2$, with the exact obstruction given by a
multiplicity-weighted value $V^\star_\tau$ over the admissible core. This bound
does not vanish with sample size.
Because this floor is a functional of the model-induced admissible sets, it can
be certified from unlabeled deployment inputs whenever those sets are observed
exactly or estimated with a certified decoding radius. The rest of the paper
develops and audits these quantities.

\paragraph{Contributions.}
We establish four results that together characterize how a model-induced
ambiguity set creates an irreducible learning floor and how that floor can be
audited from held-out unlabeled inputs.

\begin{enumerate}[leftmargin=*, itemsep=4pt, topsep=4pt]
  \item \textbf{Representation.}
      Thresholded model outputs define a pointwise-admissible core and its open
      $\zeta$-tolerance class. The core diameter equals the overlap mass
      $D^\star_\tau=\Pr(|A_\tau(X)|\ge2)$, while the tolerance-class diameter
      lies in $[D^\star_\tau,D^\star_\tau+2\zeta]$, with both endpoints sharp
      (\Cref{thm:representation}; sharpness in \Cref{app:proofs}).
  \item \textbf{Blind-channel value.} Target blindness yields a half-diameter
      minimax floor over the tolerance class and the exact value
      $V^\star_\tau$ over the core, uniformly in the learner's sample size
      (\Cref{thm:semantic_floor,thm:deployed_exact_minimax}). The latter is
      attained by a data-independent randomized learner and witnessed by an
      explicit finite least-favorable family. It equals
      $D^\star_\tau/2$ when at most two labels are admissible at almost every input and
      is strictly sharper when higher-order overlap has positive mass.
  \item \textbf{Coherent readings.} An explicit $\eta$-coverage
        condition bounds the gap between $V^\star_\tau$ and the blind minimax
        value over a finite family of global readings
        (\Cref{thm:coherent_bridge}). This condition addresses the fact that
        pointwise selectors need not represent coherent global interpretations.
        A fit--holdout audit finds that the bound is uninformative for the
        supplied two-reading pool (\Cref{ssec:qwen_eta_bridge}).
  \item \textbf{Certification.} Observed admissible sets give Hoeffding
        certificates for the pairwise floor, diameter, and exact core value:
        $N$ held-out inputs give a pairwise lower certificate at a Hoeffding
        radius, and under sampled decoding the radius additionally pays
        threshold-margin and finite-depth terms
        (\Cref{prop:master_statistic}; \Cref{cor:plugin_dstar};
        \Cref{thm:end_to_end};
        \Cref{prop:vstar_certificate,prop:plugin_vstar_certificate}).
\end{enumerate}

\paragraph{Adjacent literatures.}
Two adjacent literatures bracket the problem without containing it. Noisy and
weak supervision assume an operative target and analyze corrupted, aggregated,
or heuristic observations of
it~\citep{natarajan2013learning,ratner2017snorkel,dawid1979maximum}. In
classical malicious- and nasty-noise models, the resulting lower bounds scale
with the corruption budget and vanish as that budget
vanishes~\citep{kearnsli1993malicious,bshouty2002nasty}. Instance-dependent
noise further illustrates the role of structural restrictions: bounded
instance- and label-dependent noise is learnable under additional conditions,
as shown by \citet{cheng2020bounded}, whereas an unrestricted noise law leaves
the target--noise decomposition unidentified. Truth-tracking analyses
similarly ask when reports from imperfect sources identify an externally
defined state~\citep{singleton2024truth}.

LLM-as-a-judge research instead makes language part of the evaluation
interface: work on rating indeterminacy documents that the criteria supplied
to a judge may themselves admit multiple valid
readings~\citep{guerdan2025rating}. Quantitative analyses
bound how far judge labels can substitute for ground truth under
\emph{exogenous} judge error or
bias~\citep{dorner2025limits,feuer2026biasbounded}. NL-PAC isolates the
intersection of these concerns: the same specification-mediated channel both
defines the admissible targets and supplies the observations, so its inability
to reveal the operative reading creates an identification obstruction. That
obstruction is set by the indistinguishability of admissible readings, not by
a noise budget, and is endogenous to the model's own interpretation.

On the
alignment side, an information-theoretic argument holds that no specification
can pin down an external target while conditional entropy
remains~\citep{young2025harm}, and reward-misspecification studies show that
optimizing a proxy for an underspecified objective can diverge sharply from the
intended one~\citep{pan2022reward}; the ambiguity diameter certified here is an
operational, model-relative witness of that gap. The closest prior quantity is
the \emph{ambiguity degree} of partial-label
learning~\citep{cour2011partial,liu2014learnability}, which marks exactly the
condition under which candidate-label ambiguity does \emph{not} create an
irreducible floor; NL-PAC characterizes the complementary regime, in which the
channel cannot disambiguate. Further bordering literatures, including learning
from disagreement, perspectivist and prompt-underspecification research, and
inference on partially identified parameters, are treated in \Cref{app:related}.

\paragraph{Relation to classical theory.}
NL-PAC instantiates classical statistical decision theory in a
language-mediated supervision problem. The two-target and
least-favorable-prior reductions follow standard
arguments~\citep{wald1950statistical,blackwell1954games,ferguson1967decision},
and the per-input value coincides with the finite zero--one subset game from
adversarial multiclass classification~\citep{fathony2016adversarial}. The
NL-PAC contribution is to derive the target set and observation channel from a
fixed model--prompt configuration, construct an explicit finite
least-favorable family of admissible selectors, and make the resulting minimax
value certifiable from held-out unlabeled inputs. The proofs combine classical
two-point testing and concentration
arguments~\citep{tsybakov2009introduction,hoeffding1963probability} with this
model-induced structure.

\paragraph{Empirical validation.}
The principal experiment audits a frozen Qwen judge's declared-label
conditional first-token probabilities, where the certificate applies with its
correction-free radius: one prespecified prompt yields a positive
model-relative certificate, while a second paraphrase and the exact-rule
controls yield zero (\Cref{sec:experiments}). Two audits delimit the claim:
the bridge passes its fitted mixture-coverage check but fails its held-out
admissibility check for the supplied two-reading pool, and the sampled-decoding
mode is inconclusive at feasible depth because its finite-depth and
threshold-margin corrections are vacuous. Controlled
tasks recovering the prescribed zero and positive certificate cases, and a
ChaosNLI calculation illustrating the gain from multiclass multiplicity,
appear in \Cref{app:experiment_diagnostics}.

\paragraph{Scope of the guarantee.}
The certificate quantifies ambiguity relative to the audited tuple
$(\LLM,\Pi,\tau,P)$; it does not by itself identify the designer's intended
target or establish human task ambiguity. Extending the guarantee to coherent
human readings or deployment settings therefore requires coverage and external
construct-validity evidence.

\paragraph{Roadmap.}
\Cref{sec:setup} defines the admissible geometry; \Cref{sec:oracle} derives the
blind-channel values and coherent-reading bridge; \Cref{sec:certificate}
constructs finite-sample certificates; and \Cref{sec:experiments,sec:discussion}
report the empirical probes, limitations, and extensions.

\section{NL-PAC Setup and Model-Admissible Geometry}
\label{sec:setup}

This section distinguishes four sources of uncertainty: ambiguity in the task
description, stochasticity in model decoding, error relative to a target, and
indistinguishability induced by the supervision channel. The resulting
quantities describe the audited model and channel rather than ambiguity in
human interpretation.

\subsection{Supervision and admissibility channels}
\label{ssec:basic_objects}

We adopt the standard PAC setting. Let $(\Xcal,\mathscr F_\Xcal)$ be a
measurable instance space, let $\Ycal$ be a finite label space with the
discrete $\sigma$-algebra, and let $P$ be a probability measure on $\Xcal$.
The target is a measurable labeling function $f\colon\Xcal\to\Ycal$. In the
classical setting, supervision consists of direct observations of $f(x)$.

Here, by contrast, the target is specified through natural language. Let
$\mathcal T$ denote the model's token alphabet and equip the set
$\mathcal T^*$ of finite token strings with the discrete $\sigma$-algebra.
The spaces of task descriptions and justifications are measurable subsets
$\SNL,\JNL\subseteq\mathcal T^*$, endowed with the corresponding subspace
$\sigma$-algebras $\mathscr F_{\SNL}$ and $\mathscr F_{\JNL}$. All sets of strings
used below are assumed measurable. A fixed task description
$\fNL\in\SNL$, supplied independently of the learner, is intended to specify
the target $f$.

Throughout, the admissibility channel interprets the description with a
pretrained language model $\LLM$ whose parameters $\theta$ are frozen.
Since $\LLM$ is generative, $\LLM(\cdot\mid u)$ is a distribution over
token-string \emph{continuations} of a string $u\in\mathcal{T}^*$, not
over the label set $\Ycal$. Two fixed maps bridge the gap: a
\emph{label-extraction kernel} (verbalizer)
\[
    V \colon \mathcal{T}^* \to \Delta(\Ycal)
\]
taking a generated string to a distribution over labels, where $\Delta(\Ycal)$ is
the probability simplex on $\Ycal$ (e.g.\ constrained decoding onto a verbalizer
set $\{v_y\}_{y\in\Ycal}$, or parsing an answer field), and
a \emph{prompt template}
\[
    \Pi \colon \SNL \times \Xcal \to \mathcal{T}^*,
\]
taking a description and an instance to the prompt $\Pi(\fNL, x)$ presented
to $\LLM$. The kernel induces the label probability
\[
    p_{\LLM}(y \mid u) \;:=\; \sum_{s\in\mathcal{T}^*} V(y\mid s)\,\LLM(s\mid u).
\]
For the fixed description, define the prompt-conditioned
\emph{admissibility score}
\[
    \pi_{\LLM}(y\mid\fNL,x)
    \;:=\;
    p_{\LLM}\bigl(y\mid\Pi(\fNL,x)\bigr).
\]
This score is a decoding probability, not a calibrated posterior over human
interpretations. All label probabilities below refer to this quantity. The
description $\fNL$ and distribution $P$ are fixed throughout; we suppress them
from $A_\tau$, $D^\star_\tau$, and $\Sel_\tau$, displaying them only where a
result varies them (as in $\Fcal_{\tau,\zeta}(\fNL)$ and
$\Diam_{\tau,\zeta}(\fNL,P)$).

\begin{assumption}[Measurability]
\label{a:measurable}
For every $y \in \Ycal$, the map
$x \mapsto \pi_{\LLM}(y \mid \fNL,x)$ is
$\mathscr{F}_{\Xcal}$-measurable.
\end{assumption}

Abstractly, the learner's supervision is an \emph{observed-label kernel}
$x\mapsto Q(\cdot\mid x)\in\Delta(\Ycal)$, the conditional law of the label it
receives at input $x$. For the fixed-description base model studied here we
restrict attention to a kernel $Q$ that carries \emph{no} operative-target
argument; \Cref{sec:oracle} later embeds this base model in a target-indexed
family $Q_f$ in which target dependence, and hence partial distinguishability, is
possible. Target-independence is thus a modeling restriction of the base channel,
not a property discovered later. The headline blind-channel results depend only
on $Q$, not on how it is realized. We realize $Q$ concretely through a
justification oracle $g$ and a decoder $\rho$, introduced next, so that
$Q(\cdot\mid x)=\mathrm{Law}(\rho(J,x)\mid\fNL,x)$ with $J\sim g(\cdot\mid\fNL,x)$;
readers interested only in the floors may read $(g,\rho)$ as one instantiation of
$Q$.

\begin{definition}[Justification oracle]
\label{def:oracle}
The justification oracle $g$ is a Markov kernel from
$\bigl(\SNL\times\Xcal,\,\mathscr{F}_{\SNL}\otimes\mathscr{F}_{\Xcal}\bigr)$
to $\bigl(\JNL,\mathscr{F}_{\JNL}\bigr)$. On input $(\fNL,x)$ it draws a
justification $J\sim g(\cdot\mid\fNL,x)$, with realizations $j\in\JNL$.
\end{definition}

\begin{assumption}[Conditional independence of oracle draws]
\label{a:kernel}
The inputs are i.i.d., $X_{1:m}\sim P^m$, and given $X_{1:m}$ the draws
$J_i \sim g(\cdot \mid \fNL, X_i)$ are conditionally independent across $i$.
\end{assumption}

This excludes dependence among repeated oracle outputs beyond that induced by the
sampled inputs. We impose full conditional independence and do not analyze weaker
dependence conditions such as martingale-difference or mixing assumptions.

\begin{definition}[Justification decoder]
\label{def:decoder}
The \emph{justification decoder} is a deterministic map
$\rho\colon\JNL\times\Xcal\to\Ycal$ that reads a label off a justification--input
pair. It is required to be measurable with respect to the product
$\sigma$-algebra $\mathscr F_{\JNL}\otimes\mathscr F_{\Xcal}$ on its domain and
the discrete $\sigma$-algebra on $\Ycal$, so that the composed observed label
$\rho(J,X)$ is a well-defined $\Ycal$-valued random variable.
\end{definition}

The tuple $(\Xcal,\Ycal,P,f,\fNL,g,\rho)$ defines an \emph{NL-PAC learning
problem}. The target labels are unobserved; instead, for $i=1,\ldots,m$ the
learner receives a sample generated by drawing $X_i\sim P$, then a justification
$J_i\sim g(\cdot\mid\fNL,X_i)$, and setting $Y_i=\rho(J_i,X_i)$.
A \emph{learner} is any measurable, possibly randomized rule that maps
$(X_i,Y_i)_{i=1}^m$ to a measurable classifier
$h\colon\Xcal\to\Ycal$. We impose no hypothesis-class restriction, and measure
its performance against the latent target by the standard risk
$\err_P(h)=\Pr_{X\sim P}[h(X)\neq f(X)]$.

The framework therefore separates two channels. The \emph{admissibility channel}
$x\mapsto\pi_{\LLM}(\cdot\mid\fNL,x)$ determines which labels the model regards
as compatible with the description, whereas the \emph{supervision channel}
$(g,\rho)$ determines the labels available to the learner. The oracle $g$ may
be instantiated by the same model or by a distinct supervisor, so the two
induced label distributions need not coincide. NL-PAC asks whether the
observed channel supports low target risk or leaves an irreducible gap that
cannot be closed by changing the learner or increasing the sample size.

\begin{example}[Content moderation]
\label{ex:moderation}
Let $\Xcal$ be user comments and $\Ycal=\{0,1\}$ denote \emph{safe} and
\emph{toxic}, with $f(x)=1$ exactly when $x$ contains a literal lexical insult.
The description $\fNL=\text{\emph{``label the comment toxic ($1$) or safe ($0$)''}}$
is silent on sarcasm, admitting two readings: (i)~pragmatic personal attacks,
including sarcasm, are toxic and (ii)~only literal lexical insults are toxic, with
the target following reading~(ii). On $x=\text{\emph{``That was a brilliant move, genius''}}$ we have
$f(x)=0$, yet the oracle may return
$j=\text{\emph{``a sarcastic personal attack''}}$, yielding
$\rho(j,x)=1\neq f(x)$: not a decoding failure but the label under
reading~(i).
\end{example}

For a fixed description $\fNL$, the \emph{supervision channel} $(g,\rho)$ is
conditioned on the input $x$, but not on the operative target. The resulting
observation law therefore cannot distinguish among candidate targets associated
with the same description.

\begin{proposition}[Fixed-description target blindness]
\label{prop:deployed_blindness}
Fix $P$, the description $\fNL$, the oracle $g$, and the decoder $\rho$, and
draw $X\sim P$ and $J\sim g(\cdot\mid\fNL,X)$. The law of the observed pair
$(X,\rho(J,X))$ does not depend on the operative target: it is identical under
any two candidates $f_1,f_2\colon\Xcal\to\Ycal$.
\end{proposition}

\begin{proof}
The law of $(X,\rho(J,X))$ factors as the input marginal $P$ composed with the
observed-label kernel $Q(\cdot\mid x)=\mathrm{Law}(\rho(J,x)\mid\fNL,x)$ with
$J\sim g(\cdot\mid\fNL,x)$. This kernel is determined by $g(\cdot\mid\fNL,x)$ and
the deterministic decoder $\rho(\cdot,x)$ alone; neither depends on the operative
target, so $Q$, and hence the joint law, is identical under any two candidates
$f_1,f_2$. This is the sense in which the blind-channel results depend only on
$Q$.
\end{proof}

We call an observation channel with this property \emph{target-blind}, and all subsequent blind-channel guarantees refer to this fixed-description
observation law. \emph{Oracle accuracy}, measured by the error rate
\[
    \egora:=\Pr_{X\sim P,\,J\sim g(\cdot\mid\fNL,X)}[\rho(J,X)\neq f(X)],
\]
is a
separate property of the supervision channel, required neither for target
blindness nor for the admissibility geometry; its compatibility with a
localized unresolved floor is discussed in \Cref{rem:floor_vs_epssys}. Thus the
channel can be accurate on average yet uninformative exactly where candidate
targets disagree.

This pins down the four sources of uncertainty announced above: description
ambiguity surfaces as the admissible-overlap mass $D^\star_\tau$
(\Cref{def:overlap_mass}); decoding stochasticity is the sampling controlled
later by the plug-in analysis (\Cref{ssec:plugin}); error relative to a target is
$\egora$ above; and
channel-induced indistinguishability is target blindness
(\Cref{prop:deployed_blindness}), the property that turns $D^\star_\tau$ into an
irreducible floor.

\subsection{Model-admissible labeling geometry}
\label{ssec:ambiguity}

We quantify model-admissible ambiguity through the diameter of a class of
labelings supported by the model's decoded probabilities. Two parameters enter
the construction: the probability threshold $\tau$ determines which labels are
admissible at an input, while the tolerance $\zeta$ controls the fraction of
inputs on which a labeling may violate pointwise admissibility. Throughout, we
take $\tau,\zeta\in(0,1/2)$. The resulting diameter will reduce to an observable
per-instance overlap statistic that drives both the risk floors and their
empirical certificates.

\begin{definition}[Admissible label set]
\label{def:admissible_set}
For each $x\in\Xcal$, the \emph{admissible label set} at threshold $\tau$ is
\[
    A_\tau(x) \;:=\; \bigl\{\, y \in \Ycal : \pi_{\LLM}(y \mid \fNL, x) \ge \tau \,\bigr\}.
\]
Thus a label is admissible precisely when its decoded probability is at least
$\tau$.
\end{definition}

\begin{assumption}[Finite labels and nonempty admissible sets]
\label{a:finite_Y}
The label space satisfies $2 \le |\Ycal| < \infty$, and
$A_\tau(x)$ is nonempty for $P$-almost every $x$.
\end{assumption}

This assumption ensures that pointwise admissible selections exist. Because
every distribution on $\Ycal$ assigns probability at least $1/|\Ycal|$ to some
label, nonemptiness is automatic when $\tau\le1/|\Ycal|$. In particular, it is
automatic for binary labels under the standing condition $\tau<1/2$.

\begin{definition}[Admissible-overlap mass]
\label{def:overlap_mass}
The \emph{admissible-overlap mass} is
\[
    D^\star_\tau:=\Pr_{x\sim P}\!\left[|A_\tau(x)|\ge2\right].
\]
\end{definition}

\begin{remark}[Role of the threshold $\tau$]
\label{rem:tau}
The threshold $\tau$ controls the strength of evidence required for a label to
be retained. Multiple labels can be admissible only when $\tau\le1/2$;
choosing $\tau<1/2$ therefore allows overlap without requiring an exact tie.
The geometric identities remain valid at other thresholds whenever
\Cref{a:finite_Y} holds, but $D_\tau^\star=0$ for $\tau>1/2$. Moreover,
admissible sets contract monotonically with the threshold:
\[ \tau_1\le\tau_2 \quad\Longrightarrow\quad A_{\tau_2}(x)\subseteq A_{\tau_1}(x)
   \quad\Longrightarrow\quad D^\star_{\tau_2}\le D^\star_{\tau_1}. \]
Hence increasing $\tau$ can only decrease the certified ambiguity floor. Our
experiments avoid post hoc threshold selection by reporting a prespecified
threshold sweep together with the corresponding certificate curve
(\Cref{sec:experiments,appssec:threshold_sensitivity}).
\end{remark}

\begin{definition}[Model-admissible labeling class and diameter]
\label{def:ambiguity_set}
\emph{(i) Disagreement.} Call a measurable labeling $f'\colon\Xcal\to\Ycal$
\emph{admissible at} $x$ when $f'(x)\in A_\tau(x)$, and measure disagreement
between two such labelings by the \emph{disagreement pseudometric}
\[
    d_P(f_1,f_2) \;:=\; \Pr_{x\sim P}\!\left[f_1(x)\neq f_2(x)\right].
\]

\emph{(ii) Class and core.} The \emph{model-admissible labeling class} at
tolerance $\zeta$ collects the labelings admissible off a set of mass at most
$\zeta$,
\[
    \Fcal_{\tau,\zeta}(\fNL)
    :=
    \left\{\, f'\colon\Xcal\to\Ycal\text{ measurable} \;\middle|\;
        \Pr_{x\sim P}\!\left[f'(x)\in A_\tau(x)\right] > 1-\zeta \,\right\},
\]
and its \emph{almost-everywhere admissible core} is the zero-tolerance class,
recovered as the intersection $\Sel_\tau=\bigcap_{\zeta>0}\Fcal_{\tau,\zeta}(\fNL)$,
\[
    \Sel_\tau
    :=
    \left\{\, f'\colon\Xcal\to\Ycal\text{ measurable} \;\middle|\;
        f'(x)\in A_\tau(x)\text{ for $P$-a.e.\ }x \,\right\}.
\]

\emph{(iii) Diameter.} For a family $\mathcal G$ of labelings, write
$\Diam_P(\mathcal G):=\sup_{f_1,f_2\in\mathcal G}d_P(f_1,f_2)$, with
$\Diam_P(\mathcal G)=0$ when $|\mathcal G|\le1$. The \emph{model-admissible
diameter} is
\[
    \Diam_{\tau,\zeta}(\fNL,P)
    :=\Diam_P\bigl(\Fcal_{\tau,\zeta}(\fNL)\bigr).
\]
\end{definition}

The disagreement pseudometric is standard in disagreement-based active
learning~\citep{hanneke2014theory}. Since $\err_P(h)=d_P(h,f)$, bounds expressed
in $d_P$ translate directly into PAC risk bounds.

These objects are well defined. For every $y\in\Ycal$,
\[ \{x:y\in A_\tau(x)\}=\{x:\pi_{\LLM}(y\mid\fNL,x)\ge\tau\} \]
is measurable by \Cref{a:measurable}. Since $\Ycal$ is finite,
$x\mapsto A_\tau(x)$ is a measurable multifunction and every admissibility and
disagreement event below is measurable; hence $\Fcal_{\tau,\zeta}(\fNL)$ and its
diameter are well defined.

Conceptually, $\Fcal_{\tau,\zeta}(\fNL)$ is a model-relative identified
set~\citep{manski2003partial} formed from approximate measurable selections of
$A_\tau$~\citep{molchanov2018random}, and its core $\Sel_\tau$ is a version space
in the sense of \citet{mitchell1982generalization}, with model admissibility
replacing labeled data as the defining constraint. What the analysis below
exploits is that this identified set is model-relative and that its diameter
reduces to a single auditable statistic.

For a measurable $u\colon\Xcal\to\Ycal$, write
$d_P(u,\Sel_\tau):=\inf_{v\in\Sel_\tau}d_P(u,v)$ for its distance to the
admissible core $\Sel_\tau$.

\begin{definition}[Maximal-spread pair]
\label{def:maximal_spread}
Fix a total order $\prec$ on $\Ycal$, and set $y_0:=\min_{\prec}\Ycal$. When
$A_\tau(x)\neq\varnothing$, enumerate its elements
$y_{(1)}(x),\ldots,y_{(|A_\tau(x)|)}(x)$ in decreasing order of
$\pi_{\LLM}(\cdot\mid\fNL,x)$, with ties broken by $\prec$. The
\emph{maximal-spread pair} takes the two top-ranked admissible labels,
\[
    f_\tau^{+}(x):=y_{(1)}(x),
    \qquad
    f_\tau^{-}(x):=y_{\bigl(\min\{2,\,|A_\tau(x)|\}\bigr)}(x),
\]
and both default to $y_0$ where $A_\tau(x)=\varnothing$. The two therefore
coincide exactly when $|A_\tau(x)|\le1$.
\end{definition}

In \Cref{ex:moderation}, the maximal-spread pair splits exactly on the
sarcastic borderline comments: there both labels clear the threshold,
$A_\tau(x)=\{0,1\}$, and $f_\tau^+$ takes the higher-scoring one, $f_\tau^-$
the other, while the two coincide on unambiguous comments;
$D^\star_\tau$ is the $P$-mass of the borderline region.

\begin{lemma}[Maximal-spread identity]
\label{lem:diameter_elicited}
Under \Cref{a:measurable,a:finite_Y}, the pair $f_\tau^{+},f_\tau^{-}$ lies in
the admissible core $\Sel_\tau$ and disagrees exactly on the overlap region,
\[
    \{x:f_\tau^{+}(x)\neq f_\tau^{-}(x)\}=\{x:|A_\tau(x)|\ge2\}.
\]
In particular, $d_P(f_\tau^{+},f_\tau^{-})=D^\star_\tau$.
\end{lemma}

\begin{proof}
By \Cref{a:measurable} and finiteness of $\Ycal$ the ranked selectors are
measurable, and both pick admissible labels wherever $A_\tau(x)\neq\varnothing$,
which holds $P$-a.e.\ by \Cref{a:finite_Y}; hence $f_\tau^+,f_\tau^-\in\Sel_\tau$.
By construction they differ at $x$ precisely when $A_\tau(x)$ holds at least two
labels, so the displayed identity holds and its $P$-measure is $D^\star_\tau$.
\end{proof}

The theorem below makes the diameter tractable: $\Fcal_{\tau,\zeta}(\fNL)$ is
the open $d_P$-neighborhood of its core $\Sel_\tau$, and its diameter is
bracketed by the overlap mass.

\begin{theorem}[Geometry of the model-admissible labeling class]
\label{thm:representation}
Under \Cref{a:measurable,a:finite_Y},
\[
    \Fcal_{\tau,\zeta}(\fNL)
    =
    \left\{
        u\colon\Xcal\to\Ycal\text{ measurable}
        :d_P(u,\Sel_\tau)<\zeta
    \right\}.
\]
Moreover,
\[
    \Diam_P(\Sel_\tau)=D^\star_\tau,
    \qquad
    D^\star_\tau
    \le \Diam_{\tau,\zeta}(\fNL,P)
    \le D^\star_\tau+2\zeta.
\]
\end{theorem}

\begin{proof}[Proof sketch; full proof in \Cref{app:proofs}]
If $u$ is within $\zeta$ of some $v\in\Sel_\tau$, admissibility transfers from $v$
to $u$ off the small disagreement set, placing $u\in\Fcal_{\tau,\zeta}(\fNL)$.
Conversely, an admissible $u\in\Fcal_{\tau,\zeta}(\fNL)$ violates pointwise
admissibility only on a set of mass below $\zeta$; reassigning $u$ there to a
top-ranked admissible label $y_{(1)}(x)$ yields a core member $v\in\Sel_\tau$
that differs from $u$ on that same set, so $d_P(u,v)<\zeta$. Hence
$\Fcal_{\tau,\zeta}(\fNL)$ is exactly the open $\zeta$-neighborhood of
$\Sel_\tau$. Within the core, two selections can disagree only where
$|A_\tau(x)|\ge2$, and the maximal-spread pair attains that overlap exactly,
giving $\Diam_P(\Sel_\tau)=D^\star_\tau$. For the upper bracket, apply this same
reassignment to an arbitrary pair $u,v\in\Fcal_{\tau,\zeta}(\fNL)$, obtaining
core members $\tilde u,\tilde v\in\Sel_\tau$ with $d_P(u,\tilde u),d_P(v,\tilde
v)<\zeta$; the triangle inequality then gives
$d_P(u,v)\le d_P(\tilde u,\tilde v)+d_P(u,\tilde u)+d_P(v,\tilde v)<D^\star_\tau+2\zeta$.
\end{proof}

The additive $2\zeta$ in the bracket is not slack of the analysis: both
endpoints are attained, so no distribution-free bound improves it.

\begin{proposition}[Tightness of the coverage bracket]
\label{prop:bracket_tightness}
Both endpoints of the bracket $D^\star_\tau\le\Diam_{\tau,\zeta}(\fNL,P)
\le D^\star_\tau+2\zeta$ are sharp: each is realized as the diameter for some
choice of $P$ and admissible-set geometry:
\begin{enumerate}[label=(\alph*), itemsep=2pt]
    \item if $P$ restricted to $\{x:|A_\tau(x)|=1\}$ is nonatomic and this set
    has mass at least $2\zeta$, then
    $\Diam_{\tau,\zeta}(\fNL,P)=D^\star_\tau+2\zeta$;
    \item if $P$ is purely atomic and every atom has probability at least
    $\zeta$, then $\Diam_{\tau,\zeta}(\fNL,P)=D^\star_\tau$.
\end{enumerate}
\end{proposition}

\begin{proof}[Proof sketch; full proof in \Cref{app:proofs}]
For (a), nonatomicity supplies disjoint subsets $S_1,S_2$ of the singleton
region, each of mass $\zeta-\eps$. Reassigning the maximal-spread pair to a
distinct label on $S_1$ and $S_2$ keeps both labelings within coverage $\zeta$
yet adds $2(\zeta-\eps)$ of disagreement beyond the overlap mass $D^\star_\tau$;
letting $\eps\downarrow0$ pushes the diameter to $D^\star_\tau+2\zeta$. For (b),
an admissible labeling violates coverage on a set of mass below $\zeta$, which
can contain no atom once every atom has mass at least $\zeta$; that set is then
$P$-null, so the $\zeta$-neighborhood collapses onto the core and the diameter
equals $D^\star_\tau$.
\end{proof}

The identity $\Diam_P(\Sel_\tau)=D^\star_\tau$ equates the diameter of the
almost-everywhere admissible core with the mass of inputs carrying multiple
admissible labels. This ambiguity is a property of the model, prompt, and
threshold, not of the learner's sample: it persists under arbitrarily many
further observations from the same target-blind channel.

\section{Blind-Channel Risk Floors and Exact Minimax Values}
\label{sec:oracle}

This section characterizes the risk imposed by target-blind supervision. We
first derive pairwise lower bounds that hold uniformly over the sample size,
then identify the exact minimax value over the pointwise-admissible core and
relate it to globally coherent readings. The admissible-overlap mass
$D^\star_\tau$ is the central geometric quantity.

\subsection{Two-point blind-channel floor}
\label{ssec:semantic_floor}

We work in the \emph{oracle-label model}, in which the learner observes labels
decoded from oracle justifications rather than the target labels $f(X_i)$.

\begin{definition}[Oracle-labeled sample]
\label{def:oracle_sample}
An \emph{oracle-labeled sample} of size $m$ is the sequence
$S_m^{\mathrm{oracle}}:=((X_i,\widehat Y_i))_{i=1}^m$ obtained by drawing
$X_{1:m}\sim P^m$, then drawing the justifications $J_1,\ldots,J_m$
conditionally independently with $J_i\sim g(\cdot\mid\fNL,X_i)$, and setting
$\widehat Y_i:=\rho(J_i,X_i)$.
\end{definition}

The operative target enters the learner's risk but not the sampling law above,
so distinct targets can define the same statistical experiment. We turn this
into a lower bound by Le Cam's two-point method \citep{tsybakov2009introduction}:
if two model-admissible
labelings disagree on a set of mass $D$ and the oracle cannot distinguish them
there, the learner observes the same sample law under both targets, so no
measurable rule can be correct for both.

Fix $P$, $\fNL$, $g$, and $\rho$, and consider two measurable candidate targets
$f_1,f_2\colon\Xcal\to\Ycal$. To accommodate both target-blind and
target-dependent supervision, we momentarily allow the oracle to depend on the
operative target: write $g(\cdot\mid\fNL,x,f)$ for a target-indexed family of
Markov kernels (in the sense of \Cref{def:oracle} for each fixed $f$) that
specializes to the target-blind oracle $g(\cdot\mid\fNL,x)$ of
\Cref{def:oracle} exactly when it is constant in $f$. Let $P_j$ denote the law
of one observed example $(X,\rho(J_j,X))$ under target $f_j$, where
$X\sim P$ and $J_j\sim g(\cdot\mid\fNL,X,f_j)$. Under \Cref{a:kernel}, a sample
of size $m$ has law $P_j^{\otimes m}$. Write
$B:=\{x:f_1(x)\neq f_2(x)\}$ for the disagreement region and
$D:=P(B)=d_P(f_1,f_2)$ for its probability. This pairwise construction does
not require either target to be model-admissible.

\begin{theorem}[Two-point blind-channel floor]
\label{thm:semantic_floor}
For every sample size $m$ and every measurable, possibly randomized learner
$\hat h$ mapping an oracle-labeled sample to a classifier, the following
bounds hold.
\begin{enumerate}[label=(\alph*), itemsep=2pt]
    \item \emph{(Le Cam bound.)}
          \[
              \max_{j\in\{1,2\}}
              \mathbb{E}_{S\sim P_j^{\otimes m}}\bigl[d_P(\hat h(S), f_j)\bigr]
              \;\geq\;
              \frac{D}{2}\Bigl(1-\mathrm{TV}\bigl(P_1^{\otimes m},P_2^{\otimes m}\bigr)\Bigr),
          \]
          where the expectation averages over the sample and any internal
          randomization.
    \item \emph{(Unresolved-ambiguity floor.)} Suppose the oracle does not
          distinguish the two targets: for $P$-almost every $x$,
          \begin{equation}
          \label{eq:unresolved}
              \mathrm{Law}\bigl(\rho(J_1,x)\mid \fNL,x,f_1\bigr)
              = \mathrm{Law}\bigl(\rho(J_2,x)\mid \fNL,x,f_2\bigr).
          \end{equation}
          Then the two targets induce the same one-example observation law,
          $P_1=P_2$, so $P_1^{\otimes m}=P_2^{\otimes m}$ and the bound of
          part~(a) is independent of the sample size:
          \[
              \max_{j\in\{1,2\}}
              \mathbb{E}_{S\sim P_j^{\otimes m}}
              \bigl[d_P(\hat h(S), f_j)\bigr]
              \;\geq\; \frac{D}{2}
              \qquad\text{for every }m.
          \]
\end{enumerate}
\end{theorem}

\begin{proof}[Proof sketch; full proof in \Cref{app:proofs}]
For (a), any returned classifier satisfies
$d_P(h,f_1)+d_P(h,f_2)\ge D$ pointwise on $B$, since $h(x)$ can match at most
one of two distinct labels; integrating this against the common part
$\min(p_1,p_2)$ of the two sample laws reduces the problem to testing
$P_1^{\otimes m}$ against $P_2^{\otimes m}$ and yields the
$\tfrac{D}{2}(1-\mathrm{TV})$ bound. For (b), condition~\Cref{eq:unresolved}
equates the one-example observation laws, $P_1=P_2$, hence
$P_1^{\otimes m}=P_2^{\otimes m}$ and the total-variation term vanishes
\emph{for every $m$}. The bound of part~(a) therefore reduces to the
sample-independent floor $D/2$: no sample size recovers what the channel never
transmits.
\end{proof}

This blind-channel bound is exact for the two-target problem: a data-independent
learner that returns $f_1$ or $f_2$ with equal probability has risk $D/2$
against either target, so the floor of part~(b) is attained. The exactness is
for the two-target minimax, not necessarily for the entire admissible class;
\Cref{cor:floor_tightness} records the general least-favorable-prior identity.

\begin{corollary}[Diameter floor for the base channel]
\label{cor:diameter_floor}
Suppose that, for every target
$f\in\Fcal_{\tau,\zeta}(\fNL)$, each observed example is generated by
$X\sim P$, $J\sim g(\cdot\mid\fNL,X)$, and $\widehat Y=\rho(J,X)$. Then, for
every sample size $m$,
\[
    \inf_{\hat h} \sup_{f\in\Fcal_{\tau,\zeta}(\fNL)}
        \mathbb{E}\bigl[d_P(\hat h(S_m^{\mathrm{oracle}}), f)\bigr]
    \;\geq\; \tfrac{1}{2}\,\Diam_{\tau,\zeta}(\fNL,P).
\]
The infimum ranges over all measurable, possibly randomized learners, and the
expectation averages over the sample and the learner's internal randomization.
\end{corollary}

\begin{proof}
Every pair $f_1,f_2\in\Fcal_{\tau,\zeta}(\fNL)$ is admissible and, on the
fixed-description base channel, unresolved (\Cref{prop:deployed_blindness}), so
\Cref{thm:semantic_floor}(b) gives
$\max_{j\in\{1,2\}}\E[d_P(\hat h,f_j)]\ge\tfrac12 d_P(f_1,f_2)$ uniformly in
$m$. Since $\sup_{f}\E[d_P(\hat h,f)]$ dominates this two-point maximum,
optimizing the pair so that $d_P(f_1,f_2)$ approaches
$\Diam_{\tau,\zeta}(\fNL,P)$ yields the bound; the supremum need not be attained
(\Cref{prop:bracket_tightness}).
\end{proof}

We call a pair \emph{full-law unresolved} if its one-example observation laws
coincide. Every pair has this property under the fixed-description channel,
because the sampling law does not depend on the operative target. In the
deterministic extreme, the target-blind channel always emits labels from one
fixed reading $f'$, whose error against a target $f$ is $d_P(f,f')$. The resulting
$D/2$ floor is uniform over the sample size but remains a pairwise statement;
it does not determine the minimax risk over the entire admissible class. We
compute that full-class value next.

\subsection{Exact blind minimax values}
\label{ssec:exact_blind_values}

We now pass from a single unresolved pair to the full pointwise-admissible core
$\Sel_\tau$. The resulting minimax value is determined by the number of
admissible labels at each input. To define this multiplicity everywhere, fix
$y_0\in\Ycal$ and set
\[
    \bar A_\tau(x):=
    \begin{cases}
        A_\tau(x), & A_\tau(x)\neq\emptyset,\\
        \{y_0\}, & A_\tau(x)=\emptyset,
    \end{cases}
    \qquad
    k(x):=|\bar A_\tau(x)|.
\]
By \Cref{a:finite_Y}, $\bar A_\tau(x)=A_\tau(x)$ for $P$-almost every $x$, so
the completion is immaterial for all population quantities below.

\begin{definition}[Admissible-multiplicity functional]
\label{def:exact_blind_value}
The \emph{expected non-modal admissible mass} is
\[
    V^\star_\tau
    :=
    \E_{x\sim P}\!\left[1-\frac{1}{k(x)}\right].
\]
\end{definition}

At each input $x$, uniform randomization over $\bar A_\tau(x)$ incurs error
$1-1/k(x)$ against every admissible target label. Thus $V^\star_\tau$ is the
population risk achieved by the uniform admissible strategy.

The three central quantities play different roles.
The overlap mass $D^\star_\tau$ \emph{counts} the inputs carrying more than one
admissible label; the exact value $V^\star_\tau$ is the actual minimax error a
blind learner must incur; and the diameter $\Diam_{\tau,\zeta}$ measures the
spread of the admissible class. They coincide up to a factor of one half,
$V^\star_\tau=\tfrac12 D^\star_\tau=\tfrac12\Diam_P(\Sel_\tau)$, exactly when at
most two labels are admissible almost everywhere, and separate once higher-order
overlap has positive mass.

We show next that the uniform admissible strategy is minimax optimal
under target blindness.

\begin{theorem}[Exact blind-channel minimax over admissible selections]
\label{thm:deployed_exact_minimax}
Under \Cref{a:measurable,a:finite_Y}, suppose the supervision channel is
\emph{target-blind on $\Sel_\tau$} (automatic for the fixed-description channel,
\Cref{prop:deployed_blindness}): the one-example observation law $P_f$ it induces
under a target $f$ is common to all $f\in\Sel_\tau$. Then, for every sample size $m$,
\[
    \inf_{\hat h}\ \sup_{f\in\Sel_\tau}
        \E_{S\sim P_f^{\otimes m}}\!\left[d_P(\hat h(S),f)\right]
    = V^\star_\tau,
\]
the infimum ranging over all randomized measurable learners and the expectation
including their internal randomness. The common value is attained by a
data-independent learner.
\end{theorem}

\begin{proof}[Proof sketch; full proof in \Cref{app:proofs}]
\emph{Achievability.} Consider the data-independent learner given by the
measurable Markov kernel $x\mapsto\mathrm{Uniform}(\bar A_\tau(x))$ (a uniform
draw from the finite cyclic selector family of the full proof, hence measurable
since $\bar A_\tau$ is a measurable finite-valued multifunction). It errs against
any fixed $f\in\Sel_\tau$ with probability $1-1/k(x)$ at $x$; its risk is
therefore $V^\star_\tau$ against every target.

\emph{Lower bound.} Target-blindness on $\Sel_\tau$ makes the sample law
identical across the class, so no learner's output depends on which target is
operative. Averaging the risk over a finite cyclic family of admissible
selections whose uniform mixture places uniform mass on $\bar A_\tau(x)$ at
every $x$ shows that no learner improves on $\E_x[1-1/k(x)]=V^\star_\tau$. The
family is thus least favorable, and the two bounds meet for every sample
size~$m$.
\end{proof}

For any target family $\mathcal G$, define its sample-size-$m$ minimax risk by
\[
    R_m(\mathcal G)
    :=
    \inf_{\hat h}\ \sup_{f\in\mathcal G}\
        \E_{S\sim P_f^{\otimes m}}\!\left[d_P(\hat h(S),f)\right].
\]
The theorem above establishes $R_m(\Sel_\tau)=V^\star_\tau$ for every $m$,
and its pointwise term $1-1/k(x)$ coincides with the subset value of the
adversarial multiclass zero-one game of \citet{fathony2016adversarial}.
The equality admits a classical decision-theoretic reading as well: the
cyclic selector family is least favorable for an uninformative experiment
in the sense of \citet{blackwell1954games}. Together these two viewpoints
furnish an explicit finite measurable selector family for the
model-induced sets $\bar A_\tau(x)$, together with an aggregate value
$V^\star_\tau$ estimable from held-out unlabeled inputs.

The equality $R_m(\Sel_\tau)=V^\star_\tau$ is exact on the
pointwise-admissible core $\Sel_\tau$, but transfers automatically neither
to the wider $\zeta$-neighborhood $\Fcal_{\tau,\zeta}(\fNL)$ nor to a class
of globally coherent readings: the former tolerates admissibility
violations on a set of small probability, while the latter constrains how
labels are selected jointly across inputs. These are genuinely distinct
decision problems and are treated separately.

\subsection{From pointwise selectors to coherent readings}
\label{ssec:coherent_bridge}

The defining condition of $\Fcal_{\tau,\zeta}(\fNL)$ constrains admissibility
pointwise, so a labeling in $\Fcal_{\tau,\zeta}(\fNL)$ may vary its
interpretation across inputs without corresponding to any single global
reading. We therefore separate labelings induced by one reading clause from
those assembled by pointwise selection. \Cref{tab:labeling_classes} summarizes
the three labeling classes used in this paper.

\begin{table}[t]
\centering
\small
\setlength{\tabcolsep}{4.5pt}
\begin{tabular}{@{}llll@{}}
\toprule
\textbf{Object} & \textbf{Construction} & \textbf{Interpretation} & \textbf{Governs} \\
\midrule
$\Sel_\tau$ (Def.~\ref{def:ambiguity_set}) & pointwise admissible ($P$-a.e.) & exact theoretical core & exact minimax value $V^\star_\tau$ \\
$\Fcal_{\tau,\zeta}(\fNL)$ (Thm.~\ref{thm:representation}) & within $\zeta$ of the core & tolerance class & $\zeta$-bracketed neighborhood minimax \\
$\Fcal^{\mathrm{read}}_{\tau,\zeta}(\fNL)$ (Def.~\ref{def:coherent_subclass}) & induced by a global reading clause & semantic candidate pool & coherent-reading bridge \\
\bottomrule
\end{tabular}
\caption{The three labeling classes and the result each governs.}
\label{tab:labeling_classes}
\end{table}

\begin{definition}[Coherent-reading subclass]
\label{def:coherent_subclass}
Let $\mathcal R_{\mathrm{read}}$ be a finite collection of natural-language
\emph{reading clauses}, each specifying an interpretation of $\fNL$. For
$a\in\mathcal R_{\mathrm{read}}$, let $\iota(\fNL,a)$ denote the description
obtained by appending clause $a$ to $\fNL$, and define its induced labeling by
\[
    f_a(x)
    :=\arg\max_{y\in\Ycal}
      p_{\LLM}\!\bigl(y\mid\Pi(\iota(\fNL,a),x)\bigr),
\]
with ties resolved by a fixed ordering of $\Ycal$. We extend \Cref{a:measurable}
to every appended clause: for each $a\in\mathcal R_{\mathrm{read}}$ and
$y\in\Ycal$, the map $x\mapsto p_{\LLM}(y\mid\Pi(\iota(\fNL,a),x))$ is
$\mathscr F_{\Xcal}$-measurable. With $\Ycal$ finite and the tie-break fixed,
each $f_a$ is then a measurable labeling, so membership in
$\Fcal_{\tau,\zeta}(\fNL)$ is well posed. The \emph{coherent-reading
subclass} consists of the induced labelings that satisfy the admissibility
criterion:
\[
    \Fcal^{\mathrm{read}}_{\tau,\zeta}(\fNL)
    :=\{f_a:a\in\mathcal R_{\mathrm{read}}\}
      \cap\Fcal_{\tau,\zeta}(\fNL).
\]
\end{definition}

The coherent-reading subclass is relative to the chosen clause collection and
need not exhaust all coherent interpretations of $\fNL$. Enlarging
$\mathcal R_{\mathrm{read}}$ can only enlarge
$\Fcal^{\mathrm{read}}_{\tau,\zeta}(\fNL)$.

\begin{remark}[What the diameter does and does not measure]
\label{rem:coherent_scope}
The inclusion
$\Fcal^{\mathrm{read}}_{\tau,\zeta}(\fNL)
\subseteq\Fcal_{\tau,\zeta}(\fNL)$ gives
\[
    \Diam_P\bigl(\Fcal^{\mathrm{read}}_{\tau,\zeta}(\fNL)\bigr)
    \;\le\;
    \Diam_{\tau,\zeta}(\fNL,P)
    \;\le\;
    D^\star_\tau+2\zeta.
\]
The identity $D^\star_\tau=\Diam_P(\Sel_\tau)$ concerns the pointwise-admissible
core. It does not imply that $D^\star_\tau$ bounds the coherent-reading
diameter, because members of the latter class may use their $\zeta$ coverage
slack. Conversely, the maximal-spread pair attaining $D^\star_\tau$ may switch
labels across inputs and need not arise from any single global reading. Thus
$D^\star_\tau$ measures pointwise model-admissible variation, not necessarily
the spread among coherent interpretations.
\end{remark}

We next compare the pointwise value $V^\star_\tau$ with the minimax value of a
finite coherent family. If a channel is target-blind on a finite target class
$\mathcal T$, then $R_m(\mathcal T)$ is independent of $m$; denote this common
value by $V_{\mathrm{blind}}(\mathcal T):=R_m(\mathcal T)$. For a prior
$\lambda$ on a coherent family $\mathcal C=\{f_a\}$, let
$\mu_x^\lambda(y):=\sum_{a:f_a(x)=y}\lambda_a$ denote the induced label
distribution at input $x$.

\begin{definition}[$\eta$-uniform coverage]
\label{def:eta_coverage}
Fix $\eta\ge0$. A finite family
$\mathcal C\subseteq\Fcal^{\mathrm{read}}_{\tau,\zeta}(\fNL)$ is
\emph{$\eta$-uniformly covering} for $\Sel_\tau$ if there exists a distribution
$\lambda^\star\in\Delta(\mathcal C)$ such that
\[
    \E_{x\sim P}\!\left[
        \max_{y\in\Ycal}\mu_x^{\lambda^\star}(y)-\frac{1}{k(x)}
    \right]
    \le \eta.
\]
\end{definition}

Here $\eta$ bounds the \emph{average} gap between the mixture's modal
concentration $\max_y\mu_x^{\lambda^\star}(y)$ and the uniform-admissible
benchmark $1/k(x)$. This per-input gap can be negative where a reading places
mass outside $A_\tau(x)$, so it is not a pointwise excess; only its average,
together with $\zeta$, controls the distance between the coherent and pointwise
blind values. Smaller $\eta$ indicates that the coherent family more closely
reproduces the local spread of the pointwise-admissible core.

\begin{theorem}[Coherent-reading bridge]
\label{thm:coherent_bridge}
Let $\mathcal C\subseteq\Fcal^{\mathrm{read}}_{\tau,\zeta}(\fNL)$ be finite
and nonempty, and let the supervision channel be target-blind on
$\mathcal C$. Then the coherent blind value never exceeds the core value by
more than $\zeta$,
\[
    V_{\mathrm{blind}}(\mathcal C)\le V^\star_\tau+\zeta,
\]
and if $\mathcal C$ is moreover $\eta$-uniformly covering for $\Sel_\tau$,
the two values agree up to $\max\{\eta,\zeta\}$:
\[
    V^\star_\tau-\eta
    \;\le\; V_{\mathrm{blind}}(\mathcal C)
    \;\le\; V^\star_\tau+\zeta,
    \qquad
    \bigl|V^\star_\tau-V_{\mathrm{blind}}(\mathcal C)\bigr|
    \;\le\; \max\{\eta,\zeta\}.
\]
\end{theorem}

\begin{proof}[Proof sketch; full proof in \Cref{app:proofs}]
For the upper bound, consider the data-independent strategy that predicts
uniformly on $\bar A_\tau(x)$. Against any $f_a\in\mathcal C$, its pointwise
error is $1-1/k(x)$ wherever $f_a(x)\in A_\tau(x)$ and is at most one
elsewhere. Since each $f_a$ violates admissibility on a set of probability
less than $\zeta$, its risk is at most $V^\star_\tau+\zeta$.

For the lower bound, place the covering distribution $\lambda^\star$ on
$\mathcal C$. Target blindness makes the learner's observation law independent
of the sampled target. At input $x$, the smallest Bayes error against the
induced label distribution $\mu_x^{\lambda^\star}$ is
$1-\max_y\mu_x^{\lambda^\star}(y)$. Therefore
\[
    V_{\mathrm{blind}}(\mathcal C)
    \ge \E_x\!\left[1-\max_y\mu_x^{\lambda^\star}(y)\right]
    = V^\star_\tau
      -\E_x\!\left[\max_y\mu_x^{\lambda^\star}(y)-\frac1{k(x)}\right]
    \ge V^\star_\tau-\eta.
\]
\end{proof}

\begin{remark}[Reading the bridge]
\label{rem:bridge_reading}
The bound separates two approximation errors. The tolerance $\zeta$ measures
departures from pointwise admissibility, whereas $\eta$ measures how closely a
mixture of coherent readings matches the balanced label distribution of the
pointwise core. If both vanish, the coherent and pointwise blind values agree.
The theorem does not ensure that a given reading pool has small coverage
error. Both quantities are pool-relative and should be reported together: a
small value of $\eta$ is informative only when the corresponding admissibility
tolerance $\zeta$ is also small.
\end{remark}

The lower bound is operational only if the coverage error of the supplied
reading pool can be bounded from data. Sample splitting provides such a
certificate. On a fitting sample
$z^{\mathrm{fit}}_{1:N}$, choose
\[
    \widehat\lambda
    \in\arg\min_{\lambda\in\Delta(\mathcal C)}
    \frac1N\sum_{i=1}^N
    \left(
        \max_y\mu_{z_i^{\mathrm{fit}}}^{\lambda}(y)
        -\frac1{k(z_i^{\mathrm{fit}})}
    \right).
\]
Evaluate the same bounded loss at $\widehat\lambda$ on an independent audit
sample. Conditional on the fitting sample, $\widehat\lambda$ is fixed, so a
standard one-sided concentration bound yields a valid population upper bound
on its coverage error. Reusing the fitting sample would instead require a
uniform bound over $\lambda$.

\section{Finite-Sample Certification of Blind-Channel Values}
\label{sec:certificate}

We construct finite-sample certificates from held-out inputs drawn from the
deployment distribution $P$. Empirical admissible overlap certifies a pairwise
risk floor and brackets the model-admissible diameter, while admissible-set
multiplicities certify the exact core value $V^\star_\tau$. The latter can be
strictly sharper when more than two labels overlap. These guarantees are
specific to the audited deployment distribution.

The risk certificates are one-sided. A positive lower bound establishes a
minimax floor; a small or zero bound is inconclusive, because sampling
uncertainty may dominate the observed disagreement or the audited pair may not
witness the full class diameter.

The section builds several certificates from a single overlap audit, and their
scope varies along two axes: the \emph{access mode} (exposed decoding
probabilities versus sampled decoding) and the \emph{target quantity} (the
pairwise blind floor, the model-admissible diameter, or the exact core value
$V^\star_\tau$). The reader may use the following map. Under exposed
probabilities, \Cref{prop:master_statistic} certifies the diameter band and the
pairwise floor; \Cref{thm:end_to_end} packages the pairwise floor for the
canonical maximal-spread pair into the headline canonical-pair certificate; and
\Cref{prop:vstar_certificate} sharpens this to the exact value $V^\star_\tau$
using admissible-set multiplicities. \Cref{cor:finite_candidate} lifts the
pairwise floor to a finite family by a union bound. Under sampled decoding, the
same three quantities are certified with an enlarged plug-in radius
(\Cref{cor:plugin_dstar}; \Cref{prop:plugin_vstar_certificate}). \Cref{alg:certificate}
collects the exposed and sampled procedures into one audit.

\subsection{Estimating model-admissible overlap}
\label{ssec:master_statistic}

The admissible-overlap mass satisfies
$D^\star_\tau=\E[\mathbf 1\{|A_\tau(X)|\ge2\}]$. Its Bernoulli integrand records
whether the model admits multiple labels at an input and is observable from
the description and input alone. Consequently, $D^\star_\tau$ can be estimated
from unlabeled deployment inputs at the usual $N^{-1/2}$ rate.

\begin{proposition}[Master ambiguity statistic]
\label{prop:master_statistic}
Let $X_1,\ldots,X_N\overset{\mathrm{i.i.d.}}{\sim}P$, and suppose the
admissible sets $A_\tau(X_i)$ are available. Define
\[
    \widehat D^\star
    :=\frac1N\sum_{i=1}^N
      \mathbf 1\!\left\{|A_\tau(X_i)|\ge2\right\},
    \qquad
    \eps_N:=\sqrt{\frac{\log(2/\delta)}{2N}}.
\]
Then, with probability at least $1-\delta$,
\[
    \bigl|\widehat D^\star-D_\tau^\star\bigr|\le\eps_N.
\]
On the same event,
\[
    \Diam_{\tau,\zeta}(\fNL,P)
    \in
    \left[
        (\widehat D^\star-\eps_N)_+,
        \min\!\left\{1,\widehat D^\star+\eps_N+2\zeta\right\}
    \right].
\]
Under fixed-description target-blind supervision, half the lower endpoint is
also a valid one-sided lower confidence bound on the minimax risk floor
witnessed by the maximal-spread pair.
\end{proposition}

\begin{proof}
Set $Z_i:=\mathbf 1\{|A_\tau(X_i)|\ge2\}$. The variables $Z_1,\ldots,Z_N$ are
i.i.d.\ Bernoulli with mean $D^\star_\tau$, and
$\widehat D^\star=N^{-1}\sum_i Z_i$. Hoeffding's inequality therefore gives,
with probability at least $1-\delta$,
\[
    (\widehat D^\star-\eps_N)_+
    \le D^\star_\tau
    \le \min\{1,\widehat D^\star+\eps_N\}.
\]
On this event, substituting these bounds into
$D^\star_\tau\le\Diam_{\tau,\zeta}(\fNL,P)
\le D^\star_\tau+2\zeta$ and using
$\Diam_{\tau,\zeta}(\fNL,P)\le1$ gives the stated confidence band.

Finally, the maximal-spread pair is admissible and has disagreement
$D^\star_\tau$. Under target-blind supervision its two-point minimax floor is
$D^\star_\tau/2$, which is at least
$(\widehat D^\star-\eps_N)_+/2$ on the same event.
\end{proof}

The geometry places the diameter in the partially identified interval
$ \left[D^\star_\tau,\ \min\{1,D^\star_\tau+2\zeta\}\right],$
where truncation reflects that a disagreement probability cannot exceed one
\citep{imbens2004confidence}. The audit estimates
$D^\star_\tau$, while $\zeta$ remains a fixed tolerance. Propagating a
confidence interval for $D^\star_\tau$ through these deterministic endpoints
therefore yields the diameter band directly; no additional correction for an
estimated, potentially collapsing identification slack is required
\citep{stoye2009more}.

\begin{remark}[Sharper radii via empirical Bernstein]
\label{rem:empirical_bernstein}
We use Hoeffding's inequality for a simple distribution-free certificate. At a
fixed sample size, an empirical-Bernstein bound \citep{maurer2009empirical}
replaces the Hoeffding radius
$\eps_N=\sqrt{\log(2/\delta)/(2N)}$ with a variance-adaptive radius of order
$\sqrt{\widehat D^\star\log(2/\delta)/N}+\log(2/\delta)/N$.
This can substantially tighten the certificate when overlap is rare; the
remaining arguments are unchanged after substituting any valid radius. If the
auditor chooses the sample size adaptively, a fixed-$N$ bound is insufficient.
Betting intervals \citep{waudbysmith2024estimating} or confidence sequences
\citep{howard2021time} provide variance-adaptive, anytime-valid alternatives.
\end{remark}

\subsection{The blind pairwise and finite-family floors}
\label{ssec:floor_certificate}

Target blindness is a structural property of the channel that we assume
throughout; it is not estimated by the audit. Under this assumption the audit
estimates only the disagreement mass of an independently selected admissible
pair, which suffices to certify a risk floor. Fix
$f_a,f_b\in\Fcal_{\tau,\zeta}(\fNL)$ and write $D_{ab}=P(f_a\neq f_b)$ for their
disagreement mass. For $u\in\{a,b\}$, let $P_u$ denote the one-example
observation law induced by target $f_u$ through the fixed-description base
channel. A measurable, possibly
randomized learner maps a sample $S\sim P_u^{\otimes m}$ to a classifier
$\hat h(S)$ of risk $d_P(\hat h(S),f_u)$. Independently of both $S$ and the
choice of $(f_a,f_b)$, draw $z_1,\ldots,z_N\sim P$ and estimate the disagreement
mass by
\[
    \widehat D_{ab}=\frac{1}{N}\sum_{i=1}^N
    \mathbf 1\!\left\{f_a(z_i)\neq f_b(z_i)\right\}.
\]

If the channel is target-blind on $\{f_a,f_b\}$ (so $P_a=P_b$ and
$P_a^{\otimes m}=P_b^{\otimes m}$ for every $m$), the two-point argument of
\Cref{thm:semantic_floor}(b) gives the population floor
$\max_{u\in\{a,b\}}\E_{S\sim P_u^{\otimes m}}[d_P(\hat h(S),f_u)]\ge D_{ab}/2$,
uniformly in $m$ and over all measurable learners. Its empirical version replaces
$D_{ab}$ by the one-sided Hoeffding lower bound
$D_{L,ab}:=[\widehat D_{ab}-\sqrt{\log(2/\delta)/2N}]_+$, valid with probability
at least $1-\delta$, so that the certified floor is $D_{L,ab}/2$; we record both
as \Cref{thm:floor_certificate} in \Cref{app:proofs}. The certificate is
algorithm-independent: $\hat h$ ranges over all measurable rules, so the floor is
a property of the supervision channel, not of any learner. Its only sampling
requirement is independence of the $N$ audit inputs; no repeated oracle call
enters. We report the two-sided radius $\eps_N=\sqrt{\log(2/\delta)/2N}$
throughout under a single confidence budget, which for a one-sided floor is
deliberately conservative.

Lifting the certificate from a single pair to a fixed finite family
$\{f^{(1)},\dots,f^{(K)}\}\subseteq\Fcal_{\tau,\zeta}(\fNL)$ costs only a union
bound over the $\binom K2$ pairs, widening each per-pair radius to
$\sqrt{\log(2\binom K2/\delta)/2N}$ and certifying
\[
    \inf_{\hat h}\max_k\E[d_P(\hat h(S),f^{(k)})]\ge\max_{a<b}D_{L,ab}/2
\]
(\Cref{cor:finite_candidate} in \Cref{app:proofs}). Once the admissible sets are
available, the certified floor is a property of the fixed labelings and the audit
inputs alone; no further reading-specific oracle calls are required.

\begin{remark}[Scope of the minimax certificate]
\label{rem:deployed_minimax_scope}
\Cref{cor:finite_candidate} bounds the worst-case risk over the retained
admissible readings, a maximin (or $\Gamma$-minimax) criterion in the sense of
the maxmin expected-utility model for a set of priors \citep{gilboa1989maxmin}.
It protects against every retained reading and, in particular, places no prior
over them, so it is the appropriate target when no defensible distribution over
intended readings is available. Should external evidence supply reading weights,
one may instead evaluate the prior-weighted average risk; the present
certificate concerns the worst-case quantity, not that average.
\end{remark}

This overlap statistic also characterizes the population zero-overlap regime
and upper-bounds unresolved overlap from data.

\begin{corollary}[Zero-overlap regime]
\label{cor:dstar_zero}
If $D^\star_\tau=0$, then $\Diam_P(\Sel_\tau)=0$ and
$\Diam_{\tau,\zeta}(\fNL,P)\le2\zeta$; conversely,
$\Diam_{\tau,\zeta}(\fNL,P)>2\zeta$ implies
$D^\star_\tau\ge\Diam_{\tau,\zeta}(\fNL,P)-2\zeta>0$. Consequently, the certified
upper bound $\widehat D^\star+\eps_N$ on $D^\star_\tau$ upper-bounds every
pairwise $D/2$ obstruction in $\Fcal_{\tau,\zeta}(\fNL)$ by
$\zeta+\tfrac12(\widehat D^\star+\eps_N)$.
\end{corollary}

\begin{proof}
Both directions follow from the representation sandwich of
\Cref{thm:representation}. If $D^\star_\tau=0$ then
$\Diam_P(\Sel_\tau)=0$, so $\Diam_{\tau,\zeta}\le2\zeta$; if instead
$\Diam_{\tau,\zeta}>2\zeta$ then $D^\star_\tau\ge\Diam_{\tau,\zeta}-2\zeta>0$.
Substituting the confidence bound $\widehat D^\star+\eps_N$ of
\Cref{prop:master_statistic} for $D^\star_\tau$ yields the stated simultaneous
upper bound on pairwise obstructions.
\end{proof}

The statistic $D^\star_\tau$ thus gives a population characterization of ambiguity
at resolution $2\zeta$, backed by a finite-sample upper confidence bound: an
overlap certified below the resolution means the residual disagreement is
coverage slack rather than substantive, while any diameter above it must register
in $D^\star_\tau$.

\subsection{Sampled decoding: the plug-in radius}
\label{ssec:plugin}

When the decoding law is accessible only through samples, the label
probabilities are not observed directly and the admissible sets must be
estimated. 

At each of $N$ i.i.d.\ held-out inputs $z_i\sim P$ the auditor draws
$r$ i.i.d.\ labels
$y_{i,1},\dots,y_{i,r}\sim\pi_{\LLM}(\cdot\mid\fNL,z_i)$, independently across
inputs, and forms the empirical decoding
$\widehat\pi_i(y)=\tfrac1r\sum_{t=1}^r\mathbf 1\{y_{i,t}=y\}$, the plug-in
admissible set $\widehat A_\tau(z_i)=\{y:\widehat\pi_i(y)\ge\tau\}$, and the
plug-in overlap statistic $\widehat D^\star_{\mathrm{plug}}
=\tfrac1N\sum_{i=1}^N\mathbf 1\{|\widehat A_\tau(z_i)|\ge2\}$.

Thresholding empirical frequencies rather than exact probabilities widens the
certificate radius by two terms. A finite-depth error arises from
misclassifying admissibility when a label sits near $\tau$, and it decays with
the number of decodings $r$. What does not decay is the mass of inputs that lie
near the threshold in the first place, measured by the \emph{threshold-margin
mass}
\[
    \kappa(\xi)\;:=\;\Pr_{x\sim P}\Bigl[\exists\,y\in\Ycal:\;
        \bigl|\pi_{\LLM}(y\mid\fNL,x)-\tau\bigr|\le\xi\Bigr],
    \qquad \xi>0,
\]
the probability that some label falls within $\xi$ of $\tau$; this term persists
at any sampling depth.

\begin{corollary}[Plug-in master statistic under a threshold margin]
\label{cor:plugin_dstar}
With $C=|\Ycal|$, for each fixed $\xi>0$ and with probability at least
$1-\delta$ over the audit,
\[
    \bigl|\widehat D^\star_{\mathrm{plug}}-D^\star_\tau\bigr|
    \;\le\;
    \eps_{N,r}(\xi)
    \;:=\;
    \underbrace{\sqrt{\frac{\log(2/\delta)}{2N}}}_{\text{sampling}}
    \;+\;\underbrace{\kappa(\xi)}_{\text{margin}}
    \;+\;\underbrace{2C\,e^{-2r\xi^2}}_{\text{depth}},
\]
and $[\,(\widehat D^\star_{\mathrm{plug}}-\eps_{N,r}(\xi))_+,\;
\min\{1,\widehat D^\star_{\mathrm{plug}}+\eps_{N,r}(\xi)+2\zeta\}\,]$ is a
two-sided confidence band for $\Diam_{\tau,\zeta}$.
\end{corollary}

\begin{proof}[Proof sketch; full proof in \Cref{app:proofs}]
Call an input $\xi$-safe if every label probability lies more than $\xi$ from
$\tau$, so the unsafe mass is $\kappa(\xi)$. At a $\xi$-safe input, a
label-wise Hoeffding union bounds the chance of a wrong admissibility decision
by $2Ce^{-2r\xi^2}$; the plug-in overlap indicator can therefore differ from
the population one only on the unsafe mass or through this depth error, so their
means differ by at most $\kappa(\xi)+2Ce^{-2r\xi^2}$. An outer Hoeffding bound
over the $N$ i.i.d.\ indicators adds the sampling term $\sqrt{\log(2/\delta)/2N}$,
and the triangle inequality assembles the three into $\eps_{N,r}(\xi)$.
Transporting the bound to $\Diam_{\tau,\zeta}$ adds only the deterministic
$2\zeta$ at the upper end.
\end{proof}

The margin mass $\kappa(\xi)$ is the only term of $\eps_{N,r}(\xi)$ that no
sampling depth removes; the residual gap between $D^\star_\tau$ and the full
diameter is the identification tolerance $2\zeta$, fixed by the description
rather than incurred by search or finite samples. A positive threshold margin
is unavoidable for a uniform finite-depth certificate: labels whose
probabilities lie arbitrarily close to $\tau$ cannot be classified as admissible
or inadmissible with uniformly controlled error from finitely many decoding
draws (\Cref{rem:plugin_margin}).

\subsection{Canonical-Pair Certificate}
\label{ssec:endtoend}

We now assemble the preceding certificates into a single audit of the canonical
maximal-spread pair $(f_\tau^+,f_\tau^-)$. This pair is a deterministic
functional of $(\pi_{\LLM},\tau,\prec)$, fixed before any audit data are drawn.
Since the data neither select nor modify it, its disagreement is estimated
directly, with no sample splitting and no data-selection correction of the type
required for adaptively chosen targets \citep{berk2013valid}.
\Cref{alg:certificate} states the procedure.

\begin{algorithm}[t]
\caption{Blind-channel floor certificate.}
\label{alg:certificate}
\small
\raggedright
\textbf{Input:} description $\fNL$; frozen model $\LLM$; prompt template $\Pi$;
threshold $\tau\in(0,\tfrac12)$; coverage tolerance $\zeta$; confidence $\delta$;
held-out inputs $z_1,\dots,z_N\sim P$; and an access mode, either exposed
probabilities or sampled decoding. Under sampled decoding, also fix a margin
$\xi$ and split the budget as $\delta=\delta_{\mathrm{audit}}+\delta_\kappa$;
supply either a prespecified deterministic upper bound on $\kappa(\xi)$ (which
consumes no confidence, so $\delta_{\mathrm{audit}}=\delta$) or one certified at
level $\delta_\kappa$ on an independent split.\\
\textbf{Output:} a one-sided floor certificate valid with probability
$\ge1-\delta$.
\begin{enumerate}[label=\arabic*., leftmargin=2.2em, itemsep=1pt, topsep=3pt]
  \item \textbf{Admissible sets.} For each $z_i$, form
        $A_\tau(z_i)=\{y:\pi_{\LLM}(y\mid\fNL,z_i)\ge\tau\}$ from exposed
        probabilities, or the $r$-sample plug-in $\widehat A_\tau(z_i)$ under
        sampled decoding (\Cref{cor:plugin_dstar}).
  \item \textbf{Statistics.} Form the overlap
        $\widehat D^\star=\tfrac1N\sum_i\mathbf 1\{|A_\tau(z_i)|\ge2\}$ and the
        multiplicity statistic $\widehat V^\star=\tfrac1N\sum_i\phi(A_\tau(z_i))$,
        with $\phi(A)=1-1/|A|$ and $\phi(\varnothing)=0$ (\Cref{eq:phi_def}),
        using $\widehat A_\tau$ in place of $A_\tau$ under sampled decoding.
  \item \textbf{Radius.} Set $\eps=\sqrt{\log(2/\delta)/2N}$ under exposed
        probabilities, or the enlarged plug-in radius
        $\eps=\eps_{\mathrm{audit}}+\kappa(\xi)+2Ce^{-2r\xi^2}$ of
        \Cref{cor:plugin_dstar} under sampled decoding, with
        $\eps_{\mathrm{audit}}=\sqrt{\log(2/\delta_{\mathrm{audit}})/2N}$; a union
        bound then makes the certificate valid at the total level $\delta$.
  \item \textbf{Certificate.} Return the certified pairwise floor
        $\tfrac12(\widehat D^\star-\eps)_+$, and, for the exact core value,
        the certificate of \Cref{prop:vstar_certificate} (exposed) or
        \Cref{prop:plugin_vstar_certificate} (sampled). Report every value
        alongside the access mode, model, template, $N$, and
        $(\delta,\tau,\zeta)$.
\end{enumerate}
\textbf{Note.} The pair $(f_\tau^+,f_\tau^-)$ is a deterministic functional of
$(\pi_{\LLM},\tau,\prec)$; under sampled decoding the corrected statistics of
step~2 estimate its population overlap, and the pair induced by the sampled sets
is not treated as a population target.
\end{algorithm}

\begin{theorem}[Certificate from empirical overlap]
\label{thm:end_to_end}
Under \Cref{a:measurable,a:finite_Y}, fix $\fNL$, the frozen $\LLM$ and template
$\Pi$, a threshold $\tau\in(0,\tfrac12)$, a coverage tolerance
$\zeta\in(0,\tfrac12)$, and a confidence $\delta\in(0,1)$, and suppose the
supervision channel is target-blind. From $N$ i.i.d.\ held-out inputs $z_1,\dots,z_N\sim P$ with
observed admissible sets $A_\tau(z_i)$, form
\[
    \widehat D^\star=\frac1N\sum_{i=1}^N\mathbf 1\{|A_\tau(z_i)|\ge2\},
    \qquad
    \eps_N=\sqrt{\frac{\log(2/\delta)}{2N}} .
\]
Then, with probability at least $1-\delta$ over the audit and simultaneously for
every sample size $m$:
\begin{enumerate}[label=(\alph*), itemsep=2pt]
    \item \emph{(Certified blind-channel floor.)} Every learner $\hat h$ obeys
          \[
              \sup_{f\in\Fcal_{\tau,\zeta}(\fNL)}\;
              \E_{S\sim P_f^{\otimes m}}\bigl[d_P(\hat h(S), f)\bigr]
              \;\ge\;
              \tfrac12\bigl(\widehat D^\star-\eps_N\bigr)_+ ,
          \]
          so in particular the blind-channel minimax risk over the class is at
          least this floor.
    \item \emph{(Diameter bracket.)} The reported floor lies within $\eps_N$ of
          $\tfrac12 D^\star_\tau$, which satisfies
          $\tfrac12 D^\star_\tau\ge\tfrac12(\Diam_{\tau,\zeta}-2\zeta)$;
          consequently the gap between the reported floor and the largest
          pairwise $D/2$ obstruction in $\Fcal_{\tau,\zeta}(\fNL)$ is at most
          $\zeta+\eps_N$.
\end{enumerate}
\end{theorem}

\begin{proof}[Proof sketch; full proof in \Cref{app:proofs}]
Two-sided Hoeffding on the i.i.d.\ indicators $\mathbf 1\{|A_\tau(z_i)|\ge2\}$
gives $|\widehat D^\star-D^\star_\tau|\le\eps_N$ with probability at least
$1-\delta$; work on this event. For \emph{(a)}, the canonical pair
$(f_\tau^+,f_\tau^-)$ is admissible (\Cref{thm:representation}) and satisfies
$\{f_\tau^+\neq f_\tau^-\}=\{|A_\tau|\ge2\}$, so its disagreement mass equals
$D^\star_\tau\ge\widehat D^\star-\eps_N$. Target-blindness renders the two
targets indistinguishable through the channel, so by the two-point floor of
\Cref{thm:semantic_floor}(b) the worst-case risk is at
least half this mass, uniformly in $m$; and since the pair is a deterministic
functional of $(\pi_{\LLM},\tau,\prec)$, fixed independently of the audit, the
whole budget $\delta$ funds the single estimate $\widehat D^\star$ without
splitting (the $K=2$ case of \Cref{thm:floor_certificate}). For \emph{(b)}, the representation bracket gives
$D^\star_\tau\ge\Diam_{\tau,\zeta}-2\zeta$, while any admissible pair has floor
$\tfrac12 d_P\le\tfrac12\Diam_{\tau,\zeta}\le\tfrac12 D^\star_\tau+\zeta$;
together with $|\widehat D^\star-D^\star_\tau|\le\eps_N$ this places every such
floor within $\zeta+\eps_N$ of the reported one.
\end{proof}

\begin{remark}[Scope, sampled decoding, and reading a zero certificate]
\label{rem:end_to_end_scope}
The certificate is a minimax lower bound over $\Fcal_{\tau,\zeta}(\fNL)$: it
rules out any learner that does uniformly better across the class, and requires
no realizability. Should the operative target lie in $\Fcal_{\tau,\zeta}(\fNL)$,
equivalently $f(x)\in A_\tau(x)$ outside a set of $P$-measure below $\zeta$, the
floor applies to the class containing it; but absent further identification it
does not lower-bound risk at that particular target.

Under sampled decoding the plug-in radius of \Cref{cor:plugin_dstar} replaces
$\eps_N$, and validity requires an upper bound on the threshold-margin mass
$\kappa(\xi)$; without one the certificate is not computable from decoding
samples alone. If that bound is certified at level $\delta_1$ on an independent
split and the overlap estimate holds at level $\delta_2$, a union bound makes
the overall certificate valid at level $\delta_1+\delta_2$.

Finally, a zero lower certificate does not establish safety or the absence of
ambiguity. It can arise from genuinely small overlap at the audited resolution,
from insufficient sampling or decoding depth, or from a judge that assigns
negligible probability to admissible alternatives. Distinguishing these causes
requires the reported uncertainty terms, the access diagnostics, and comparisons
across prespecified configurations.
\end{remark}

\subsection{Exact Minimax-Value Certificate}
\label{ssec:exact_value_certificate}

The canonical-pair floor of \Cref{thm:end_to_end} is the headline certificate; this
subsection is an optional sharpening and may be skipped on a first reading. The
overlap floor uses only whether each admissible set is a singleton. Its full
multiplicity carries more information, and exploiting it sharpens the guarantee
from a floor on worst-case risk to a certificate for the exact blind-channel
value $V^\star_\tau$.

The audited admissible sets also yield the multiplicity estimator
\[ \widehat V^\star
   :=\frac1N\sum_{i=1}^N\left(1-\frac1{k(z_i)}\right)
   =\frac1N\sum_{i=1}^N\phi(A_\tau(z_i)) \]
of $V^\star_\tau$, where $k(z_i)=|A_\tau(z_i)|$ and
\begin{equation}\label{eq:phi_def}
    \phi(A)=1-\tfrac1{|A|}\ \ \text{for } A\neq\varnothing,
    \qquad \phi(\varnothing)=0 .
\end{equation}

The exact value $V^\star_\tau$ is pinned to the overlap mass $D^\star_\tau$ from
both sides, which is what lets the audit certify it.

\begin{corollary}[Placement of the exact value]
\label{cor:deployed_minimax_bounds}
With $C=|\Ycal|$, the exact blind-channel value satisfies
\[
    \tfrac12 D^\star_\tau \;\le\; V^\star_\tau \;\le\;
    \bigl(1-\tfrac1C\bigr)D^\star_\tau,
    \qquad
    D^\star_\tau-V^\star_\tau
    = \E_x\!\left[\frac{\mathbf 1\{k(x)\ge2\}}{k(x)}\right],
\]
the lower bound attaining equality if and only if $k(x)\le2$ for $P$-a.e.\ $x$.
Under target-blindness on $\Sel_\tau$ (\Cref{thm:deployed_exact_minimax}), the
neighborhood minimax is bracketed uniformly in $m$,
\[ V^\star_\tau \;\le\; R_m(\Fcal_{\tau,\zeta}) \;\le\; V^\star_\tau+\zeta. \]
\end{corollary}

\begin{proof}
Write $k=k(x)$. The summand $1-1/k$ vanishes at $k=1$ and lies in
$[\tfrac12,1-\tfrac1C]$ for $2\le k\le C$, so $(1-1/k)\mathbf 1\{k\ge2\}$ is
squeezed between $\tfrac12\mathbf 1\{k\ge2\}$ and $(1-\tfrac1C)\mathbf 1\{k\ge2\}$;
taking expectations gives the sandwich. Subtracting termwise,
\[
    \mathbf 1\{k\ge2\}-\Bigl(1-\tfrac1k\Bigr)=\frac{\mathbf 1\{k\ge2\}}{k},
\]
and integrating gives the gap identity; the lower bound is tight exactly when the
factor $\tfrac12$ is attained a.e., that is, $k(x)\le2$.

For the neighborhood bracket, the lower bound is immediate from
$\Sel_\tau\subseteq\Fcal_{\tau,\zeta}(\fNL)$. For the upper bound, every
$f\in\Fcal_{\tau,\zeta}(\fNL)$ admits some $f_0\in\Sel_\tau$ with
$d_P(f,f_0)<\zeta$, so the exact-value-optimal learner $\hat h_\star$ satisfies
\[
    \E\bigl[d_P(\hat h_\star,f)\bigr]
    \le \E\bigl[d_P(\hat h_\star,f_0)\bigr]+d_P(f_0,f)
    < V^\star_\tau+\zeta ;
\]
taking the supremum over $f$ gives the bracket, uniformly in $m$.
\end{proof}

The overlap floor $\tfrac12 D^\star_\tau$ therefore equals $V^\star_\tau$ under
binary overlap and lower-bounds it otherwise, while the neighborhood minimax
stays within $\zeta$ of $V^\star_\tau$. Rendered finite-sample through
$\widehat D^\star$ alone, without the multiplicity estimator, this placement
already yields a two-sided interval for $V^\star_\tau$.

Concretely, for any finite admissible family $\mathcal F_K$ containing the
maximal-spread pair, the blind-channel minimax risk
$V_{\mathrm{blind}}(\mathcal F_K)$ is independent of $m$ and, with probability at
least $1-\delta$, lies in
$[\tfrac12(\widehat D^\star-\eps_N)_+,\ \widehat D^\star+\eps_N+2\zeta]$; the same
interval brackets $V^\star_\tau$ (\Cref{cor:deployed_interval}, proved in
\Cref{app:proofs}). The certificate below sharpens its lower endpoint using the
multiplicity estimator $\widehat V^\star$.

\begin{proposition}[Standalone certificate for the exact blind-channel value]
\label{prop:vstar_certificate}
Suppose the supervision channel is target-blind, and form $\widehat D^\star$ and
$\widehat V^\star$ from $N$ i.i.d.\ held-out inputs. With radius
$\eps_N=\sqrt{\log(2/\delta)/2N}$ and $C=|\Ycal|$, define
\[
    C_{\mathrm{joint}}
    :=
    \max\!\left\{
        \left(\widehat V^\star-\left(1-\tfrac1C\right)\eps_N\right)_+,
        \tfrac12\left(\widehat D^\star-\eps_N\right)_+
    \right\}.
\]
Then, with probability at least $1-\delta$,
\[
    V^\star_\tau\ge C_{\mathrm{joint}}
    \qquad\text{and}\qquad
    R_m(\Fcal_{\tau,\zeta})\ge C_{\mathrm{joint}}\ \text{ for every }m .
\]
The second branch of $C_{\mathrm{joint}}$ is the certified overlap floor, so the
certificate never falls below it; it is strictly larger precisely when
$(\widehat V^\star-(1-1/C)\eps_N)_+>\tfrac12(\widehat D^\star-\eps_N)_+$, which,
when neither branch is clipped at zero, reads
$\widehat V^\star-\tfrac12\widehat D^\star>(\tfrac12-\tfrac1C)\eps_N$.
\end{proposition}

\begin{proof}
Allocate failure probability $\delta/2$ to each of two one-sided Hoeffding
bounds. The first gives $D^\star_\tau\ge\widehat D^\star-\eps_N$: its one-sided
radius at level $\delta/2$ is $\sqrt{\log(1/(\delta/2))/(2N)}=\eps_N$, so the
split incurs no additional radius. The summands $1-1/k(z_i)$ of
$\widehat V^\star$ lie in $[0,1-\tfrac1C]$, so the second gives
$V^\star_\tau\ge\widehat V^\star-(1-\tfrac1C)\eps_N$. A union bound makes both
hold with probability at least $1-\delta$; work on that event. The second bound
is the first branch of $C_{\mathrm{joint}}$, and
$V^\star_\tau\ge\tfrac12 D^\star_\tau\ge\tfrac12(\widehat D^\star-\eps_N)$ is the
second, so $V^\star_\tau\ge C_{\mathrm{joint}}$. Finally
$\Sel_\tau\subseteq\Fcal_{\tau,\zeta}(\fNL)$ transfers the bound to the
neighborhood minimax, $R_m(\Fcal_{\tau,\zeta})\ge C_{\mathrm{joint}}$ for every
$m$.
\end{proof}

When the decoding probabilities are not exposed, the same certificate holds with
each branch widened by the plug-in error $q_r(\xi)$.

\begin{proposition}[Exact-value certificate under sampled decoding]
\label{prop:plugin_vstar_certificate}
Suppose the admissible sets are estimated by the $r$-sample plug-in
$\widehat A_\tau$ at margin $\xi$, and form the plug-in statistics
$\widehat D^\star_{\mathrm{plug}}$ and
$\widehat V^\star_{\mathrm{plug}}=\tfrac1N\sum_i\phi(\widehat A_\tau(z_i))$. With
$C=|\Ycal|$, $\eps_N=\sqrt{\log(2/\delta)/2N}$, and
$q_r(\xi)=\kappa(\xi)+2Ce^{-2r\xi^2}$, the certificate
\[
C_{\mathrm{joint}}^{\mathrm{plug}}
:=\max\!\left\{
 \bigl(\widehat V^\star_{\mathrm{plug}}-(1-\tfrac1C)(\eps_N+q_r(\xi))\bigr)_+,
 \tfrac12\bigl(\widehat D^\star_{\mathrm{plug}}-\eps_N-q_r(\xi)\bigr)_+
\right\}
\]
satisfies, with probability at least $1-\delta$,
\[
    V^\star_\tau\ge C_{\mathrm{joint}}^{\mathrm{plug}}
    \qquad\text{and}\qquad
    R_m(\Fcal_{\tau,\zeta})\ge C_{\mathrm{joint}}^{\mathrm{plug}}
    \ \text{ for every }m .
\]
\end{proposition}

\begin{proof}[Proof sketch; full proof in \Cref{app:proofs}]
Replacing exact admissible sets by their $r$-sample plug-ins introduces a
per-input error of probability at most $q_r(\xi)$. Because $\phi$ and
$\mathbf 1\{|\cdot|\ge2\}$ change only when the set changes and have ranges
$[0,1-\tfrac1C]$ and $[0,1]$, this shifts the two estimator means by at most
$(1-\tfrac1C)q_r(\xi)$ and $q_r(\xi)$. Each bias adds to the corresponding
one-sided Hoeffding radius, inflating the two branches by at most the
$q_r(\xi)$ terms above. The remaining steps are the $\delta/2$ split, the
placement $V^\star_\tau\ge\tfrac12 D^\star_\tau$, and the neighborhood transfer;
these repeat the proof of \Cref{prop:vstar_certificate}.
\end{proof}

\begin{remark}[Using the audit in practice]
\label{rem:practice}
\emph{(i)~Who can run it.} The correction-free audit needs the judge's
admissible-label probabilities, exposed only by a logits-serving backend
(open-weight or locally hosted judges); a sampling-only API instead supports the
sampled-decoding audit, valid under a threshold-margin bound and sufficient
decoding depth (\Cref{cor:plugin_dstar}). \emph{(ii)~Interpretation.} A certified
floor directs intervention at the information structure rather than at sample
size. \emph{(iii)~What to report.} Every certified quantity is
configuration-relative, so a floor is meaningful only with the tuple (model,
prompt template, $\tau$-sweep, $N$, $\delta$, access mode) that
\Cref{alg:certificate} returns alongside it.
\end{remark}

\section{Empirical Probes}
\label{sec:experiments}
\providecommand{\csModel}{qwen2.5:3b}
\providecommand{\csN}{100}
\providecommand{\csR}{3}
\providecommand{\csNkappa}{100}
\providecommand{\csTau}{0.2}
\providecommand{\csDelta}{0.1}
\providecommand{\csXi}{0.05}
\providecommand{\csDesignSeed}{20260703}
\providecommand{\csDhatPos}{0.60}
\providecommand{\csVhatPos}{0.300}
\providecommand{\csFloorPos}{0.0}
\providecommand{\csCjointPos}{0.0}
\providecommand{\csKappaU}{1.0}
\providecommand{\csEpsN}{0.1358}
\providecommand{\csDepthTerm}{1.0}
\providecommand{\csDhatCtrl}{0.05}
\providecommand{\csVhatCtrl}{0.025}
\providecommand{\csFloorCtrl}{0.0}
\providecommand{\csConstancyPos}{0.62}
\providecommand{\csUnvalidatedFrac}{0.0}
\providecommand{\csConfigs}{6}
\providecommand{\csDhatPosMean}{0.1933}
\providecommand{\csDhatPosMin}{0.0}
\providecommand{\csDhatPosMax}{0.60}
\providecommand{\csFloorPosMean}{0.0}
\providecommand{\csFloorPosMin}{0.0}
\providecommand{\csFloorPosMax}{0.0}
\providecommand{\csPositiveConfigs}{0}
\providecommand{\csDetectConfigs}{2}
\providecommand{\csControlFloorMax}{0.0}

\newcommand{\exTau}{0.20}
\newcommand{\exN}{100}
\newcommand{\exDelta}{0.10}
\newcommand{\exDhatPone}{0.29}
\newcommand{\exFloorPone}{0.0838}
\newcommand{\exJointPone}{0.0838}
\newcommand{\exEntropyPone}{0.291}
\newcommand{\exDhatPtwo}{0.10}
\newcommand{\exFloorPtwo}{0.0000}
\newcommand{\exJointPtwo}{0.0000}
\newcommand{\exEntropyPtwo}{0.087}
\newcommand{\exDhatControlPone}{0.00}
\newcommand{\exFloorControlPone}{0.0000}
\newcommand{\exDhatControlPtwo}{0.00}
\newcommand{\exFloorControlPtwo}{0.0000}
\newcommand{\exResultSha}{79c7e370537f}

\newcommand{\etaFitN}{100}
\newcommand{\etaHoldoutN}{100}
\newcommand{\etaLambdaNarrow}{0.50}
\newcommand{\etaHoldoutMean}{-0.28}
\newcommand{\etaUpper}{0.0000}
\newcommand{\etaNarrowViolation}{0.24}
\newcommand{\etaBroadViolation}{0.36}
\newcommand{\etaZetaUpper}{0.4904}
\newcommand{\etaGapUpper}{0.4904}
\newcommand{\etaHoldoutDisagreement}{0.76}
\newcommand{\etaResultSha}{434560247c62}

\providecommand{\cpShallowQ}{300}
\providecommand{\cpShallowRad}{1.0000}
\providecommand{\cpExposedQ}{100}
\providecommand{\cpExposedRad}{0.1224}

By construction, the pairwise certificate of \Cref{alg:certificate} is
positive exactly when the observed model-induced overlap clears its
finite-sample radius (the multiplicity-based exact-value certificate can be
positive through a separate branch), and any positive floor it returns is
sample-size independent. The empirical question is whether a
real judge exhibits such overlap where the specification is genuinely
underspecified and, critically, none where an exact rule leaves no room for
it. We test this on one frozen judge configuration with matched exact-rule
null controls, then delimit the resulting claim: does the pointwise
certificate transfer to supplied candidate reading clauses, and where do
implementation limits bind? These are probes, not deployment-scale benchmarks. A controlled pipeline
check on constructed cases with known disagreement, a model-free ChaosNLI
exact-value calculation, and all implementation details are collected in
\Cref{app:experiment_diagnostics}.

\paragraph{Correction-free exposed-probability audit.}
\label{ssec:qwen_exposed_audit}

We run \Cref{alg:certificate} in its exposed-probability mode on a frozen
Qwen~2.5--3B judge, where both plug-in corrections vanish and
\Cref{thm:end_to_end} applies with the correction-free radius
$\eps_N=\sqrt{\log(2/\delta)/2N}$. The audited channel is the judge's
declared-label conditional first-token law (the class-token probabilities
renormalized over the declared label set), formally defined, with all
tokenization and runtime details, in \Cref{appssec:exposed_channel}. Two prompt
paraphrases (P1, P2), two exact-rule control tasks, and the threshold sweep
$\tau\in\{0.10,\dots,0.40\}$ are fixed before any model call.

\begin{figure}[t]
\centering
\includegraphics[width=0.78\linewidth]{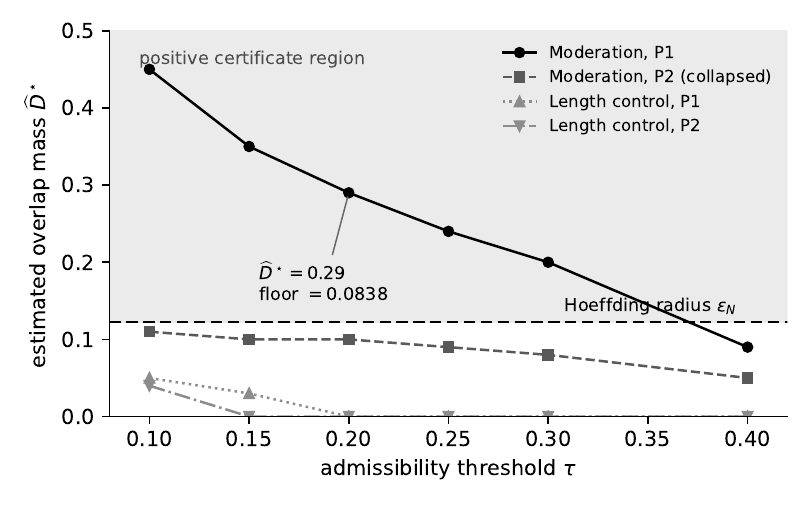}
\caption{Correction-free exposed-probability audit at $\delta=\exDelta$,
$N=\exN$, over the prespecified threshold sweep. The certificate is positive
exactly where the estimated overlap $\widehat D^\star$ clears the Hoeffding
radius $\eps_N$; the certified floor $\tfrac12(\widehat D^\star-\eps_N)_+$ is
valid uniformly over all learners and sample sizes $m$ under the
fixed-description target-blind channel (\Cref{prop:deployed_blindness}), with
blindness supplied by the deployment design, not established by the audit.
Only the P1 moderation prompt enters the positive region; the collapsed
paraphrase P2 and both exact-rule controls stay below the radius at every
$\tau$.}
\Description{Line plot of estimated overlap mass versus admissibility
threshold for four configurations, with a dashed horizontal line marking the
Hoeffding radius and shading above it; only the moderation P1 curve lies above
the radius at the swept thresholds up to 0.30, dropping below it by 0.40.}
\label{fig:qwen_exposed}
\end{figure}

\Cref{fig:qwen_exposed} reports the outcome across the sweep. For the
moderation task under P1, $\widehat D^\star=\exDhatPone$ clears the radius at
the comparison point $\tau=\exTau$ and the audit certifies a floor of
$\exFloorPone$; the certificate is positive at every swept
$\tau\le0.30$, and for the
binary label set the exact-value certificate of \Cref{prop:vstar_certificate}
coincides with this pairwise floor (numerical sweep values in
\Cref{appssec:exposed_sweep}). Both exact-rule controls certify
zero, as predicted: the procedure does not manufacture overlap where none
exists. Paraphrase P2 also certifies zero, but for a different reason: its
declared-label probability collapses onto one label (entropy diagnostic in
\Cref{appssec:exposed_channel}). Prompt sensitivity, not decoding depth, thus
gates the certificate for this pair, and a zero certificate does not establish
that the task is unambiguous (\Cref{rem:end_to_end_scope}).

\paragraph{Transfer boundary.}
\label{ssec:qwen_limits}
\label{ssec:qwen_eta_bridge}
\Cref{thm:coherent_bridge} bounds the gap between the pointwise value
$V^\star_\tau$ and the blind value over a supplied family of coherent readings,
provided the readings track the model's admissible geometry. We test it on two
supplied candidate reading clauses for the moderation task, written
independently of that geometry (they carry no independent human validation).
The fit--holdout coverage estimate is favorable ($\widehat\eta_U=\etaUpper$),
but the certified admissibility violation is large
($\widehat\zeta_U=\etaZetaUpper$). Consequently,
$|V^\star_\tau-V_{\rm blind}(\mathcal C)|\le\etaGapUpper$ is nearly the
largest bound the theorem permits. This is a direct test of the bridge, and
its failure is informative because it identifies which theorem condition
fails: the clauses match the aggregate mixture condition while violating the
separate admissibility condition, so no semantic transfer is certified. The
positive certificate above is therefore pointwise and model-relative;
this is a boundary of the claim, not evidence that the judge is unambiguous.
The clauses, split procedure, mixture optimization, confidence allocation,
and the preregistered external-reading protocol this motivates are in
\Cref{appssec:bridge_audit}.

\paragraph{Implementation boundaries.}
\label{ssec:qwen_audit}
Under sampling-only access at the audited breadth ($r=\csR$), the finite-depth
term of the plug-in radius (\Cref{cor:plugin_dstar}) saturates and every
certified floor is zero; a positive floor at margin $\xi=\csXi$ would require
roughly $r\approx877$ decodings per input, which is why the principal audit
exposes probabilities (design, per-configuration results, and cost--precision
trade-off in \Cref{appssec:sampled_design}). Separately, a controlled pipeline
check on constructed cases with known disagreement recovers the prescribed
certificates exactly: zero where two readings never disagree, and a
conservative floor of $0.166$ where they disagree by $D=0.4$
(\Cref{appssec:controlled_check}). Every audit fixes its design and split
before any model call; source code, frozen designs, versioned outputs, and
SHA-256 hashes are in the supplementary reproducibility archive
(\Cref{appssec:exposed_channel}).

\section{Discussion and Limitations}
\label{sec:discussion}

\paragraph{Summary.}
NL-PAC turns the risk induced by model-admissible ambiguity into an empirically
auditable quantity: for a fixed description, model, decoding threshold, and
deployment distribution, it certifies a finite-sample lower bound on the
worst-case risk over admissible labelings. Because this obstruction is a
property of the specification--channel pair rather than of any estimator, it
persists independently of the learner and its sample size.

\paragraph{Scope of the guarantees.}
The certified quantities are worst-case over admissible labelings and
conditional on the sampled deployment distribution $P$. They characterize the
frozen model--prompt--decoding configuration, not intrinsic human ambiguity, so
model updates, temperature changes, or prompt edits require re-auditing. The
audit requires that admissible sets be observable; under sampled decoding this
means a certified threshold-margin bound and a plug-in decoding depth
(\Cref{ssec:plugin}). The blind-channel floors rely only on target blindness and
the induced observation laws, not on repeated-call independence; the conditional
independence of \Cref{a:kernel} enters only where concentration is applied to
repeated oracle draws, namely the sampled-decoding analysis
(\Cref{ssec:plugin}).

The experiments establish one limited positive claim: the exposed-probability
audit gives a positive model-relative certificate for the prespecified P1
prompt, with zero certificates under the exact-rule controls
(\Cref{ssec:qwen_exposed_audit}). A controlled probe additionally verifies
that the certificate arithmetic recovers known zero and positive disagreement
cases (\Cref{appssec:controlled_check}).

Two negative results delimit that claim. The shallow sampled-decoding audit is
vacuous at the available depth, because the finite-depth correction saturates
(\Cref{ssec:qwen_audit}). The fit--holdout bridge is too loose to transfer the
positive exposed-probability certificate for P1 to the supplied two-reading pool
despite a zero estimated coverage gap: its admissibility bound is
$\widehat\zeta_U=\etaZetaUpper$, hence
$\max\{\eta_U,\zeta_U\}=\etaGapUpper$
(\Cref{ssec:qwen_eta_bridge}). That transfer remains open.

\paragraph{What the certificate licenses.}
A positive floor of $c$ rules out any worst-case guarantee below $c$ over the
retained model-admissible labelings, no matter how many additional labels are collected through
the same description--model--prompt channel. Lowering the floor therefore
requires changing the information structure: clarifying the specification,
introducing reading-revealing supervision, or changing the model or prompt. The
certificate quantifies the floor but does not rank these interventions.
Conversely, a zero lower certificate is not evidence of safety or of the absence
of human ambiguity (\Cref{rem:end_to_end_scope}).

The minimax criterion is the appropriate target when no defensible probability
distribution over the retained readings is available, or when the decision rule
must protect against every retained reading, as in minimax and minimax-regret
decisions under set identification \citep{manski2003partial,stoye2009minimax}.
When external evidence identifies credible reading weights, the corresponding
prior-weighted risk defines a different criterion, which the present certificate
does not estimate (\Cref{rem:deployed_minimax_scope}).

This also explains why common judge diagnostics are not substitutes for the
certificate. Response entropy and probability margins summarize concentration in the judge's
output distribution but do not yield a confidence-qualified lower bound on target
risk. Cross-judge and judge--human disagreement measure evaluator sensitivity,
yet common-mode collapse can drive these rates to zero, and they do not identify
which admissible targets the deployed channel leaves unresolved. NL-PAC supplies
the missing decision-theoretic step: for a prespecified channel, it maps the
admissible target set to a certified minimax risk floor.

\paragraph{Future work.}
Three extensions are directly exposed by the limitations above. First,
\emph{coherent-reading coverage}: a preregistered study should elicit readings
independently of model outputs, freeze the resulting clauses and a held-out
input split, and then test both bridge terms (mixture coverage and
admissibility) separately. A larger independently validated pool would show
whether \Cref{thm:coherent_bridge} becomes informative once held-out
admissibility error is small; the frozen protocol is specified in
\Cref{appssec:bridge_audit}.
Second, \emph{cross-model and natural-distribution replication}: replication
across model families and naturally occurring inputs, with a complete high-depth
sampled-decoding audit, would assess the stability of the exposed-probability
certificate and determine whether the threshold-margin or the finite-depth term
is the binding constraint (\Cref{ssec:qwen_audit}). Third, \emph{non-blind and
intervention channels}: embedding the blind channel as the zero-sensitivity
endpoint of a family of intervention channels would quantify how much
reading-revealing information the channel must expose to push risk below the
blind-channel floor (\Cref{thm:semantic_floor}).

\appendix
\section{Notation}
\label{app:notation}

\Cref{tab:notation} collects the recurring symbols.

\begin{table}[h]
\centering
\small
\begin{tabular}{@{}p{0.25\linewidth}p{0.21\linewidth}p{0.48\linewidth}@{}}
\toprule
\textbf{Symbol} & \textbf{Defined} & \textbf{Meaning} \\
\midrule
$\Xcal,\Ycal,P,f$ & \Cref{ssec:basic_objects} & instance space, finite label space, data distribution, target \\
$\fNL$ & \Cref{ssec:basic_objects} & natural-language task description (fixed string) \\
$\mathcal T,\ \mathcal T^*$ & \Cref{ssec:basic_objects} & token alphabet; string space under the discrete $\sigma$-algebra \\
$\SNL,\ \JNL$ & \Cref{ssec:basic_objects} & measurable string subspaces of $\mathcal T^*$ \\
$\mathscr F_{\SNL},\ \mathscr F_{\JNL}$ & \Cref{ssec:basic_objects} & their induced subspace $\sigma$-algebras \\
$h,\ \hat h$ & \Cref{ssec:basic_objects} & measurable classifier; learner (measurable rule) \\
$g,\ \rho,\ j$ & \Cref{def:oracle,def:decoder} & oracle kernel; justification decoder; justification \\
$\pi_{\LLM}(y\mid\fNL,x)$ & \Cref{ssec:basic_objects} & admissibility score (description-only decoding probability) \\
$\tau,\ \zeta$ & \Cref{def:ambiguity_set} & admissibility threshold; coverage tolerance \\
$\Fcal_{\tau,\zeta}(\fNL)$ & \Cref{def:ambiguity_set} & model-admissible labeling class \\
$\Fcal^{\mathrm{read}}_{\tau,\zeta}(\fNL)$ & \Cref{def:coherent_subclass} & coherent-reading subclass \\
$\Diam_{\tau,\zeta}(\fNL,P)$ & \Cref{def:ambiguity_set} & model-admissible diameter \\
$d_P$ & \Cref{def:ambiguity_set} & disagreement pseudometric $d_P(f,g)=\Pr_P[f\neq g]$ \\
$D,\ B$ & \Cref{thm:semantic_floor} & pairwise disagreement mass; disagreement set \\
$A_\tau(x)$ & \Cref{def:admissible_set} & admissible label set at $x$ \\
$D^\star_\tau$ & \Cref{def:overlap_mass} & admissible-overlap mass $\Pr_x[|A_\tau(x)|\ge2]$ \\
$V^\star_\tau$ & \Cref{def:exact_blind_value} & exact blind-channel minimax on $\Sel_\tau$ (\Cref{thm:deployed_exact_minimax}); expected non-modal admissible mass \\
$V_{\mathrm{blind}}(\cdot)$ & \Cref{thm:coherent_bridge} & blind-channel minimax of a finite target family \\
$\mathcal C$ & \Cref{thm:coherent_bridge} & finite coherent-reading family (distinct from $C=|\Ycal|$) \\
$(f_\tau^+,f_\tau^-)$ & \Cref{def:maximal_spread} & maximal-spread pair \\
$\Sel_\tau$ & \Cref{def:ambiguity_set} & almost-everywhere admissible selection class; $\Fcal_{\tau,\zeta}$ is its open $\zeta$-neighborhood \\
$k(x)$ & \Cref{ssec:exact_blind_values} & number of admissible labels, $|\bar A_\tau(x)|$ \\
$\kappa(\xi)$ & \Cref{cor:plugin_dstar} & threshold-margin mass \\
$\eta$ & \Cref{def:eta_coverage} & coherent-family coverage gap \\
\midrule
$N$ & \Cref{prop:master_statistic} & held-out audit inputs (Hoeffding sample size) \\
$m$ & \Cref{def:oracle_sample} & learner's oracle-labeled sample size \\
$r$ & \Cref{cor:plugin_dstar} & decoding depth (samples per input) \\
$\delta$ & \Cref{prop:master_statistic} & audit confidence level \\
$\eps_N$ & \Cref{prop:master_statistic} & Hoeffding sampling slack $\sqrt{\log(2/\delta)/2N}$ \\
$\xi$ & \Cref{cor:plugin_dstar} & threshold-margin half-width \\
$C$ & \Cref{cor:plugin_dstar} & label-space size $|\Ycal|$ \\
\bottomrule
\end{tabular}
\caption{Recurring notation, with the defining location.}
\label{tab:notation}
\end{table}

\section{Deferred Proofs and Supporting Statements}
\label{app:proofs}

This appendix collects the proofs deferred from the main text, grouped to mirror
the body and ordered within each group by appearance.

\subsection{Proofs for model-admissible geometry}
\label{appssec:geometry_proofs}

\begin{proof}[Proof of \Cref{thm:representation}]
\emph{($\supseteq$).} Suppose $d_P(u,v)<\zeta$ for some
$v\in\Sel_\tau$. Since $v(x)\in A_\tau(x)$ for $P$-a.e.\ $x$,
the event $\{u(x)\in A_\tau(x)\}$ contains $\{u=v\}$ up to a null set, and
disagreement with $v$ has mass $d_P(u,v)$, so
$\Pr_P[\pi_{\LLM}(u(x)\mid\fNL,x)\ge\tau]
=\Pr_P[u(x)\in A_\tau(x)]\ge\Pr_P[u=v]=1-d_P(u,v)>1-\zeta$,
which is membership in $\Fcal_{\tau,\zeta}(\fNL)$ (\Cref{def:ambiguity_set}).
\emph{($\subseteq$).} Let $u\in\Fcal_{\tau,\zeta}(\fNL)$ and set
$G:=\{x:u(x)\in A_\tau(x)\}$; $G$ is measurable because
$x\mapsto\pi_{\LLM}(u(x)\mid\fNL,x)
=\sum_{y\in\Ycal}\mathbf 1\{u(x)=y\}\,\pi_{\LLM}(y\mid\fNL,x)$ is measurable
under \Cref{a:measurable}, and membership gives $\Pr_P[G]>1-\zeta$. Define
$v:=u$ on $G$, $v:=y_{(1)}(x)$ (the top-ranked admissible label of
\Cref{def:maximal_spread}) on $G^{c}\cap\{A_\tau(x)\neq\varnothing\}$, and
$v:=y_0$ on the remaining set $G^{c}\cap\{A_\tau(x)=\varnothing\}$, which is
$P$-null by the nonemptiness hypothesis; $v$ is then a measurable total map into
$\Ycal$. Then $v\in\Sel_\tau$ and
$d_P(u,v)\le\Pr_P[G^{c}]<\zeta$, so $d_P(u,\Sel_\tau)<\zeta$.
\emph{(Diameter of the core.)} Two selections can disagree at $x$ only if $A_\tau(x)$ contains at least two labels, so $d_P(u,v)\le D^\star_\tau$ for all $u,v\in\Sel_\tau$. The two selections taking the top and a second admissible label wherever $|A_\tau(x)|\ge2$ disagree on all of that set, so the bound is attained and $\Diam_P(\Sel_\tau)=D^\star_\tau$.
\emph{(Sandwich.)} Since $\Sel_\tau\subseteq\Fcal_{\tau,\zeta}(\fNL)$, the previous step gives $D^\star_\tau\le\Diam_{\tau,\zeta}$. For the upper bound, take
$u,v\in\Fcal_{\tau,\zeta}(\fNL)$ with selections $\tilde u,\tilde v\in\Sel_\tau$ as in
($\subseteq$); the triangle inequality for $d_P$ gives
$d_P(u,v)\le d_P(\tilde u,\tilde v)
+d_P(u,\tilde u)+d_P(v,\tilde v)<D^\star_\tau+2\zeta$,
and taking the supremum yields $\Diam_{\tau,\zeta}\le D^\star_\tau+2\zeta$.
\end{proof}

\begin{proof}[Proof of \Cref{prop:bracket_tightness}]
\emph{(a)} The upper bound $\Diam_{\tau,\zeta}\le D^\star_\tau+2\zeta$ is
\Cref{thm:representation}; it remains to match it. Fix $\eps\in(0,\zeta)$. Off
the overlap region $B^\star=\{x:|A_\tau(x)|\ge2\}$ the admissible set is a
singleton $\{a(x)\}$, with $x\mapsto a(x)$ measurable (the top admissible label
of \Cref{def:maximal_spread}) and nonempty $P$-a.e.\ by \Cref{a:finite_Y}. Since
$P$ is nonatomic on $\{x:|A_\tau(x)|=1\}$, which has mass at least
$2\zeta>2(\zeta-\eps)$, choose disjoint measurable subsets $S_1,S_2$ of it with
$P(S_1)=P(S_2)=\zeta-\eps$. Let $(f_\tau^+,f_\tau^-)$ be the maximal-spread pair,
and let $b(x)$ be the $\prec$-smallest label distinct from $a(x)$ (measurable,
and well defined since $|\Ycal|\ge2$). Define
\[
    f_1(x)=\begin{cases} b(x) & x\in S_1,\\ f_\tau^+(x) & \text{else},\end{cases}
    \qquad
    f_2(x)=\begin{cases} b(x) & x\in S_2,\\ f_\tau^-(x) & \text{else}.\end{cases}
\]
Each $f_j$ deviates from an almost-everywhere admissible selection only on a set
of mass $\zeta-\eps<\zeta$, so $f_1,f_2\in\Fcal_{\tau,\zeta}(\fNL)$ by
\Cref{thm:representation}. Off $B^\star$ both selections equal $a(x)$, so
$f_1$ and $f_2$ disagree on all of $S_1\cup S_2$ (on each $S_j$ exactly one takes
the value $b(x)\neq a(x)$) and, disjointly, on the overlap event
$\{f_\tau^+\neq f_\tau^-\}=B^\star$; hence
$d_P(f_1,f_2)\ge D^\star_\tau+2(\zeta-\eps)$. Letting $\eps\downarrow0$ gives
$\Diam_{\tau,\zeta}\ge D^\star_\tau+2\zeta$, so with the upper bound the diameter
equals $D^\star_\tau+2\zeta$. The supremum is not attained by any single pair,
since each admissible labeling deviates on mass strictly below $\zeta$, whence
every pairwise disagreement is strictly below $D^\star_\tau+2\zeta$.

\emph{(b)} Let $f'\in\Fcal_{\tau,\zeta}(\fNL)$ with coverage complement
$N=\{x:\pi_{\LLM}(f'(x)\mid\fNL,x)<\tau\}$, so $P(N)<\zeta$. If every atom of the
purely atomic $P$ has mass at least $\zeta$, then $N$ contains no atom and hence
$P(N)=0$, so $f'\in\Sel_\tau$. With the immediate reverse inclusion
$\Sel_\tau\subseteq\Fcal_{\tau,\zeta}(\fNL)$, the two classes agree up to
$P$-null sets, and $\Diam_{\tau,\zeta}=\Diam_P(\Sel_\tau)=D^\star_\tau$ by
\Cref{thm:representation}.
\end{proof}

\subsection{Proofs for blind-channel floors and exact blind values}
\label{appssec:floor_proofs}

\begin{proof}[Proof of \Cref{thm:semantic_floor}]
\emph{(a)} This is Le Cam's two-point method with the error rate as
semi-distance~\citep{tsybakov2009introduction}. For any classifier $h$ and
any $x$, the indicators satisfy
$\mathbf{1}[h(x) \neq f_1(x)] + \mathbf{1}[h(x) \neq f_2(x)] \geq
\mathbf{1}[f_1(x) \neq f_2(x)]$, since $h(x)$ can equal at most one of two
distinct labels; taking $\mathbb{E}_{x \sim P}$ gives
$d_P(h, f_1) + d_P(h, f_2) \geq D$ for
every fixed $h$. Hence, writing $p_j$ for densities with respect to a common
dominating measure (e.g.\ $P_1^{\otimes m}+P_2^{\otimes m}$),
\begin{align*}
    \mathbb{E}_{P_1^{\otimes m}}[d_P(\hat h, f_1)] + \mathbb{E}_{P_2^{\otimes m}}[d_P(\hat h, f_2)]
    &\;\geq\; \int \bigl(d_P(\hat h, f_1) + d_P(\hat h, f_2)\bigr) \min(p_1, p_2) \\
    &\;\geq\; D \int \min(p_1, p_2)
    \;=\; D\,\bigl(1 - \mathrm{TV}(P_1^{\otimes m}, P_2^{\otimes m})\bigr).
\end{align*}
Since the maximum is at least the
average, dividing by two yields part~(a).

\emph{(b)} By \Cref{eq:unresolved}, the conditional law of the decoded label
$\rho(J_j,x)$ given $x$ agrees under $f_1$ and $f_2$ for $P$-almost every $x$
(off $B$ this is immediate, since the targets share the label there). The input
marginal $P$ is common to both experiments, so the one-example observation laws
coincide, $P_1 = P_2$; hence $P_1^{\otimes m} = P_2^{\otimes m}$ and
$\mathrm{TV}(P_1^{\otimes m}, P_2^{\otimes m}) = 0$ for every $m$. Substituting
into part~(a) gives $\max_j \mathbb{E}[d_P(\hat h, f_j)] \geq D/2$, uniformly in
$m$. For the fixed-description base channel the hypothesis is automatic, since
$J_j \sim g(\cdot\mid\fNL,x)$ does not depend on the target
(\Cref{prop:deployed_blindness}).
\end{proof}

\begin{corollary}[Tightness and exact minimax value of the two-point floor]
\label{cor:floor_tightness}
In the two-experiment construction of \Cref{thm:semantic_floor}, the minimax mediated
risk over the target pair $\{f_1,f_2\}$ is exactly the least-favorable-prior Bayes risk of the induced binary test~\citep{wald1950statistical,ferguson1967decision,blackwell1954games}; because this is a finite two-action test against a two-point prior, the equality of its minimax and least-favorable-prior values is elementary and requires no general minimax theorem. For every sample size $m$,
\[
    \inf_{\hat h} \max_{j\in\{1,2\}}
    \mathbb{E}_{S\sim P_j^{\otimes m}}\bigl[d_P(\hat h(S), f_j)\bigr]
    \;=\;
    D \sup_{q\in[0,1]}
    \int \min\bigl(q\,\mathrm{d}P_1^{\otimes m},(1-q)\,\mathrm{d}P_2^{\otimes m}\bigr).
\]
Consequently:
\begin{enumerate}[label=(\alph*), itemsep=2pt]
    \item \emph{(Blind tightness.)} In the unresolved case of \Cref{thm:semantic_floor}(b),
          the minimax risk equals $D/2$ for every sample size $m$, attained by the
          data-independent learner that returns $f_1$ or $f_2$ by a fair coin flip.
    \item \emph{(Symmetric tightness.)} If the least-favorable prior is uniform,
          $q^\star=\tfrac12$, then the Le Cam lower bound is tight,
          \[
              \inf_{\hat h} \max_{j\in\{1,2\}}
              \mathbb{E}_{S\sim P_j^{\otimes m}}\bigl[d_P(\hat h(S), f_j)\bigr]
              \;=\;
              \frac{D}{2}\Bigl(1-\mathrm{TV}\bigl(P_1^{\otimes m},P_2^{\otimes m}\bigr)\Bigr).
          \]
    \item \emph{(Total-variation sandwich.)} For general channels,
          \[
              \frac{D}{2}\Bigl(1-\mathrm{TV}\bigl(P_1^{\otimes m},P_2^{\otimes m}\bigr)\Bigr)
              \;\leq\;
              \inf_{\hat h} \max_{j\in\{1,2\}}
              \mathbb{E}_{S\sim P_j^{\otimes m}}\bigl[d_P(\hat h(S), f_j)\bigr]
              \;\leq\;
              D\Bigl(1-\mathrm{TV}\bigl(P_1^{\otimes m},P_2^{\otimes m}\bigr)\Bigr).
          \]
\end{enumerate}
\end{corollary}

\begin{proof}[Proof of \Cref{cor:floor_tightness}]
\emph{Reduction to a two-point test.} Off $B$ the two targets coincide, and any
label outside $\{f_1(x),f_2(x)\}$ errs against both targets at $x$; projecting a
rule pointwise onto the common label off $B$ and onto whichever of
$f_1(x),f_2(x)$ it already matches on $B$ (breaking ties arbitrarily) therefore
weakly lowers $d_P(\cdot,f_1)$ and $d_P(\cdot,f_2)$ simultaneously. Hence we may
restrict to rules $\hat h$ valued pointwise in $\{f_1(x),f_2(x)\}$. Such rules
need not be global (they may follow $f_1$ on part of $B$ and $f_2$ elsewhere), so
we do \emph{not} assert this reduction collapses onto the two functions
$f_1,f_2$; instead we identify the minimax value through the least-favorable
prior.

We first show that \emph{global} rules form a sufficient class. Fix a prior with
weight $q\in[0,1]$ on $f_1$, and conditional on the sample $S$ write $q(S)$ for
the posterior weight on $f_1$. A Bayes rule minimizes the posterior risk
$q(S)\,d_P(\hat h,f_1)+(1-q(S))\,d_P(\hat h,f_2)$, which pointwise on $B$ is
minimized by $f_1(x)$ when $q(S)\ge\tfrac12$ and by $f_2(x)$ otherwise. This is
the \emph{same} choice at every $x$, since $q(S)$ does not vary with $x$. Hence the
Bayes rule is the global test $\hat\jmath\colon S\mapsto\{1,2\}$ returning
$f_{\hat\jmath(S)}$, so global tests are sufficient even though the feasible class
is larger. The problem therefore reduces to a two-action test of $f_1$ against
$f_2$; being a finite test against a two-point prior, its minimax risk equals the
least-favorable-prior Bayes risk by the elementary finite minimax identity
\citep{wald1950statistical,ferguson1967decision}, with no general minimax theorem
required. Such a test errs only on $B$ and only when $\hat\jmath(S)\neq j$, so
$d_P(f_{\hat\jmath(S)},f_j)=D\,\mathbf 1[\hat\jmath(S)\neq j]$. Hence
\[
    \inf_{\hat h}\max_{j\in\{1,2\}}\mathbb{E}_{P_j^{\otimes m}}[d_P(\hat h,f_j)]
    = D\,\inf_{\hat\jmath}\max_{j\in\{1,2\}}\Pr_{P_j^{\otimes m}}[\hat\jmath\neq j]
    = D\sup_{q\in[0,1]}\int\min\bigl(q\,\mathrm dP_1^{\otimes m},(1-q)\,\mathrm dP_2^{\otimes m}\bigr),
\]
the last equality being the minimax (least-favorable-prior) value of the Bayes test between $P_1^{\otimes m}$ and $P_2^{\otimes m}$~\citep{wald1950statistical,tsybakov2009introduction}.
\emph{(a)} If $P_1=P_2$ the integrand is $\min(q,1-q)\,\mathrm dP_1^{\otimes m}$, with
supremum $\tfrac12$ over $q$, so the value is $D/2$; the fair-coin learner returning
$f_1$ or $f_2$ has risk $\tfrac12\cdot0+\tfrac12\cdot D=D/2$ against either target and
attains it. \emph{(b)} When the least-favorable prior is uniform, evaluating at
$q=\tfrac12$ gives $\tfrac12\int\min(\mathrm dP_1^{\otimes m},\mathrm dP_2^{\otimes m})
=\tfrac12\bigl(1-\mathrm{TV}(P_1^{\otimes m},P_2^{\otimes m})\bigr)$, so the value is
$\tfrac{D}{2}(1-\mathrm{TV})$. \emph{(c)} Evaluating the supremum at $q=\tfrac12$ gives the
lower bound; $\min(qa,(1-q)b)\le\min(a,b)$ pointwise gives
$\sup_q\int\min\le\int\min(\mathrm dP_1^{\otimes m},\mathrm dP_2^{\otimes m})=1-\mathrm{TV}$,
the upper bound.
\end{proof}

\begin{remark}[Average oracle accuracy is compatible with localized unresolvedness]
\label{rem:floor_vs_epssys}
The unresolved floor concerns only the portion of the ambiguity the oracle
cannot separate, not the full diameter, so a nonzero floor coexists with high
average oracle accuracy. Call the oracle \emph{$\gamma$-approximately correct},
for a fixed $\gamma\in(0,\tfrac12]$, when its aggregate error $\egora$ lies a
margin $\gamma$ below uniform labeling, $\egora\le1-\tfrac1{|\Ycal|}-\gamma$
(in the binary case $\egora\le\tfrac12-\gamma$). Writing $\egora^{\mathrm{off}}$
for the error rate conditional on $x\notin B$, a binary oracle that errs at the
chance rate $\tfrac12$ on $B$ has average error
\[
    \egora = \frac{D}{2} + (1-D)\,\egora^{\mathrm{off}},
\]
which satisfies the margin condition exactly when
$\tfrac{D}{2}+(1-D)\,\egora^{\mathrm{off}}\le\tfrac12-\gamma$; for an oracle
accurate off $B$ this holds for all sufficiently small $D$. Thus an oracle can
be $\gamma$-approximately correct on average while remaining uninformative
precisely where the candidate targets disagree.
\end{remark}

\begin{proof}[Proof of \Cref{thm:deployed_exact_minimax}]
By hypothesis the channel is target-blind on $\Sel_\tau$, so the oracle-labeled
sample has the same law for every $f\in\Sel_\tau$; any randomized learner
therefore induces a target-independent randomization over measurable classifiers.

Let $C=|\Ycal|$ and $L=\operatorname{lcm}(1,\ldots,C)$. Enumerate the repaired
set $\bar A_\tau(x)=\{a_1(x),\ldots,a_{k(x)}(x)\}$ in the fixed order $\prec$.
For each $j\le C$, extend $a_j(x)$ by the default value $y_0$ on
$\{x:k(x)<j\}$. The maps $k$ and these total maps $a_j$ are measurable because
$\Ycal$ is finite and each label's admissibility event is measurable under
\Cref{a:measurable}. For
$i=1,\ldots,L$, define
\[
    f^{(i)}(x)
    :=a_{1+((i-1)\bmod k(x))}(x).
\]
Each $f^{(i)}$ is measurable and belongs to $\Sel_\tau$ (the set on which
$\bar A_\tau$ repairs $A_\tau$ is $P$-null). Since $k(x)$ divides $L$, a uniform choice
of $i$ makes $f^{(i)}(x)$ uniform on $\bar A_\tau(x)$ for every $x$.

\emph{Achievability (upper bound).} Let $\hat h_\star$ ignore the sample, draw
$i$ uniformly, and output $f^{(i)}$. For every $f\in\Sel_\tau$,
\[
    \E\!\left[d_P(\hat h_\star,f)\right]
    =\E_x\!\left[\Pr_i\{f^{(i)}(x)\ne f(x)\}\right]
    =\E_x\!\left[1-\frac1{k(x)}\right]
    =V^\star_\tau,
\]
where the equality is unaffected by the repaired null set. Hence the minimax is
at most $V^\star_\tau$.

\emph{Lower bound.} Regard $\{f^{(i)}\}_{i=1}^L$ as an indexed family
(repetitions are harmless, since a supremum dominates any weighted average).
Fix any randomized learner and write
$H=\hat h(S,W)$ for its output, including its internal random seed $W$. Draw a uniform random index $i$ independently of $(X,S,W)$ by construction, and
note that by target blindness the sample law of $S$ does not depend on $i$;
these two facts together make $H=\hat h(S,W)$ independent of $i$, which
justifies interchanging the expectation over $(x,S,W)$ with the probability
over $i$ below. Therefore
\begin{align*}
    \sup_{f\in\Sel_\tau}\E\!\left[d_P(H,f)\right]
    &\ge \frac1L\sum_{i=1}^L
        \E\!\left[d_P(H,f^{(i)})\right] \\
    &=\E_{x,S,W}\!\left[
        1-\Pr_i\{f^{(i)}(x)=H(x)\mid x,H\}
      \right] \\
    &\ge \E_x\!\left[1-\frac1{k(x)}\right]
     =V^\star_\tau.
\end{align*}
Indeed, the conditional match probability is $1/k(x)$ when
$H(x)\in\bar A_\tau(x)$ and zero otherwise. Taking the infimum over learners
gives the lower bound. The two bounds coincide, and $\hat h_\star$ attains the
value for every sample size $m$.
\end{proof}

\begin{proof}[Proof of \Cref{thm:coherent_bridge}]
Target blindness on $\mathcal C$ makes the sample law common to all
$f\in\mathcal C$. Averaging a randomized learner over this common law produces
a data-independent randomized classifier with the same risk against every
target in $\mathcal C$. It therefore suffices to optimize over randomized
classifiers.

For any distribution $\lambda$ on the finite family $\mathcal C$, a plurality
classifier minimizes the $\lambda$-average risk:
\[
    \inf_{\hat h}\E_{f\sim\lambda}[d_P(\hat h,f)]
    =\inf_{\hat h}\E_x\bigl[1-\mu^\lambda_x(\hat h(x))\bigr]
    =\E_x\bigl[1-\max_y\mu^\lambda_x(y)\bigr],
\]
the minimizer $h^\star(x)\in\arg\max_y\mu^\lambda_x(y)$ being measurable ($\Ycal$
finite, $\mu^\lambda_x$ measurable in $x$), and randomization not helping since
$1-\mu^\lambda_x(y)\ge1-\max_{y'}\mu^\lambda_x(y')$ for every $y$.

\emph{Upper bound.}
Consider the data-independent randomized classifier that, at each $x$, draws
uniformly from $\bar A_\tau(x)$. This kernel is measurable because $\Ycal$ is
finite and $x\mapsto\bar A_\tau(x)$ is measurable. For
$f_a\in\mathcal C$, let
$N_a:=\{x:f_a(x)\notin A_\tau(x)\}$. Since
$f_a\in\Fcal_{\tau,\zeta}(\fNL)$, $P(N_a)<\zeta$. On $N_a^c$ the classifier
errs with probability $1-1/k(x)$, while on $N_a$ its error is at most one.
Consequently,
\[
\begin{aligned}
    \E[d_P(\hat h,f_a)]
    &\le \E_x\!\left[1-\frac1{k(x)}\right]+P(N_a) \\
    &< V^\star_\tau+\zeta.
\end{aligned}
\]
Taking the supremum over $a$ and then the infimum over learners yields
$V_{\mathrm{blind}}(\mathcal C)\le V^\star_\tau+\zeta$.

\emph{Lower bound.}
Suppose that $\mathcal C$ is $\eta$-uniformly covering, and fix the covering
distribution $\lambda^\star$ from
\Cref{def:eta_coverage}. For every learner $\hat h$, a maximum dominates an
average, $\sup_{f\in\mathcal C}\E[d_P(\hat h,f)]\ge\E_{f\sim\lambda^\star}[d_P(\hat h,f)]$,
so taking $\inf_{\hat h}$ and using the display,
\[
\begin{aligned}
    V_{\mathrm{blind}}(\mathcal C)
    &=\inf_{\hat h}\sup_{f\in\mathcal C}\E[d_P(\hat h,f)] \\
    &\ge\inf_{\hat h}\E_{f\sim\lambda^\star}[d_P(\hat h,f)]
    &=\E_x\bigl[1-\max_y\mu^{\lambda^\star}_x(y)\bigr] \\
    &=V^\star_\tau
      -\E_x\bigl[\max_y\mu^{\lambda^\star}_x(y)-\tfrac1{k(x)}\bigr]
    \;\ge\;V^\star_\tau-\eta,
\end{aligned}
\]
using $V^\star_\tau=\E_x[1-1/k(x)]$ (\Cref{thm:deployed_exact_minimax}) and the
$\eta$-coverage condition.
\end{proof}

We now record the finite-sample pairwise and finite-family floor certificates
deferred from \Cref{ssec:floor_certificate}. Throughout, fix
$f_a,f_b\in\Fcal_{\tau,\zeta}(\fNL)$ with disagreement mass $D_{ab}=P(f_a\neq
f_b)$; for $u\in\{a,b\}$ let $P_u$ be the one-example observation law induced by
target $f_u$ through the fixed-description base channel; and, independently of the
learner's sample and of the choice of pair, draw $z_1,\dots,z_N\sim P$ and set
$\widehat D_{ab}=\frac1N\sum_{i=1}^N\mathbf 1\{f_a(z_i)\neq f_b(z_i)\}$.

\begin{theorem}[Finite-sample certificate for the blind-channel floor]
\label{thm:floor_certificate}
Fix an admissible pair $(a,b)$, chosen independently of the audit sample, and a
confidence level $\delta\in(0,1)$. Suppose the supervision channel is
target-blind on $\{f_a,f_b\}$, that is, $P_a=P_b$, so that the two observation
laws coincide, $P_a^{\otimes m}=P_b^{\otimes m}$, for every $m$.
\begin{enumerate}[label=(\alph*), itemsep=2pt]
    \item \emph{(Population certificate.)} For every $m$ and every measurable,
          possibly randomized learner $\hat h$ receiving $m$ examples,
          \begin{equation}
          \label{eq:cert_population}
              \max_{u\in\{a,b\}}\;
              \mathbb{E}_{S\sim P_u^{\otimes m}}\bigl[d_P(\hat h(S), f_u)\bigr]
              \;\geq\;
              \frac{D_{ab}}{2}.
          \end{equation}
          The bound is independent of $m$: no amount of supervision through the
          blind channel can distinguish $f_a$ from $f_b$, and the worst of the
          two risks is at least half their disagreement mass.
    \item \emph{(Empirical certificate.)} With the audit estimator
          $\widehat D_{ab}$ above, with probability at least $1-\delta$ over the
          audit sample,
          \[
              D_{ab} \;\geq\; D_{L,ab}
              \;:=\;
              \Bigl[\widehat D_{ab} - \sqrt{\tfrac{\log(2/\delta)}{2N}}\Bigr]_+,
              \qquad\text{so}\qquad
              \max_{u\in\{a,b\}}\;
              \mathbb{E}\bigl[d_P(\hat h(S), f_u)\bigr]
              \;\geq\;
              \frac{D_{L,ab}}{2}
          \]
          for every $m$.
\end{enumerate}
\end{theorem}

\begin{proof}
Write $R_u:=\mathbb E_{S\sim P_u^{\otimes m}}[d_P(\hat h(S),f_u)]$ for
$u\in\{a,b\}$.

\emph{(a)} Since $d_P(f,g)=P(f\neq g)$ is a pseudometric on labelings, it obeys
the triangle inequality $d_P(h,f_a)+d_P(h,f_b)\ge d_P(f_a,f_b)$ for every $h$;
for randomized $\hat h$ the inequality holds conditionally on its realized
output and hence in expectation. The hypothesis $P_a^{\otimes m}=P_b^{\otimes m}$
lets both risks be written under the single law $P_a^{\otimes m}$, so
\[
    R_a+R_b
    =\mathbb E_{S\sim P_a^{\otimes m}}
        \bigl[d_P(\hat h(S),f_a)+d_P(\hat h(S),f_b)\bigr]
    \;\ge\; d_P(f_a,f_b)
    \;=\; D_{ab}.
\]
Therefore $\max\{R_a,R_b\}\ge\tfrac12(R_a+R_b)\ge D_{ab}/2$, which is
\Cref{eq:cert_population}. The argument uses no property of $m$, so the bound
holds for every $m$.

\emph{(b)} The audit indicators $\mathbf 1\{f_a(z_i)\neq f_b(z_i)\}$ are i.i.d.\
Bernoulli with mean $D_{ab}$, so the one-sided Hoeffding inequality gives
$\Pr[\widehat D_{ab}-D_{ab}\ge t]\le e^{-2Nt^2}$. Setting
$t=\sqrt{\log(2/\delta)/(2N)}$ yields $D_{ab}\ge D_{L,ab}$ with probability at
least $1-\delta$, and substituting this bound into part~(a) certifies the floor
$D_{L,ab}/2$.
\end{proof}

\begin{corollary}[Blind finite-family floor]
\label{cor:finite_candidate}
Let $\{f^{(1)},\dots,f^{(K)}\}\subseteq\Fcal_{\tau,\zeta}(\fNL)$, $K\ge2$, be a
finite family of admissible labelings chosen independently of the audit sample,
and let the supervision channel be target-blind on this family, so that the
observation laws satisfy $P_a=P_b$ for every pair. For each pair $a<b$ set
$\widehat D_{ab}=\frac1N\sum_{i=1}^N\mathbf 1\{f^{(a)}(z_i)\neq f^{(b)}(z_i)\}$
on the $N$ i.i.d.\ held-out audit inputs $z_1,\dots,z_N$, and define the
per-pair floor
\[
    D_{L,ab}
    :=\Bigl[\widehat D_{ab}-\sqrt{\tfrac{\log(2\binom{K}{2}/\delta)}{2N}}\Bigr]_+ .
\]
Then, with probability at least $1-\delta$ over the audit sample and for every
sample size $m$,
\[
    \inf_{\hat h}\;
    \max_{1\le k\le K}\;
    \mathbb{E}_{S\sim P_k^{\otimes m}}\bigl[d_P(\hat h(S), f^{(k)})\bigr]
    \;\geq\;
    \max_{1\le a<b\le K}\;\frac{D_{L,ab}}{2} ,
\]
where the infimum is over all measurable, possibly randomized learners.
\end{corollary}

\begin{proof}
Fix a pair $a<b$. The audit indicators $\mathbf 1\{f^{(a)}(z_i)\neq f^{(b)}(z_i)\}$
are i.i.d.\ Bernoulli with mean $D_{ab}$, so the one-sided Hoeffding inequality
with radius $\sqrt{\log(2\binom{K}{2}/\delta)/(2N)}$ gives $D_{ab}<D_{L,ab}$ with
probability at most $\delta/\binom{K}{2}$. A union bound over the $\binom{K}{2}$
pairs then makes the event
\[
    \mathcal E:=\bigl\{\,D_{ab}\ge D_{L,ab}\ \text{for all }a<b\,\bigr\}
\]
hold with probability at least $1-\delta$.

Work on $\mathcal E$, and fix any learner $\hat h$, any $m$, and any pair
$a<b$. Since $\{f^{(a)},f^{(b)}\}$ is a subset of the family, retaining only
these two targets can only decrease the maximum,
\[
    \max_{1\le k\le K}\mathbb E\bigl[d_P(\hat h(S),f^{(k)})\bigr]
    \;\ge\;\max\{R_a,R_b\}.
\]
Target-blindness on $\{f^{(a)},f^{(b)}\}$ gives $\max\{R_a,R_b\}\ge D_{ab}/2$ by
the two-point argument of \Cref{thm:floor_certificate}(a), and $D_{ab}\ge D_{L,ab}$
on $\mathcal E$, so the left side is at least $D_{L,ab}/2$. The bound holds for
every $\hat h$ and every $m$; taking the infimum over $\hat h$ and the maximum
over pairs yields the claim.
\end{proof}

\subsection{Proofs for the plug-in overlap statistic}
\label{appssec:search_proofs}

We first record why the threshold margin cannot be removed, then prove
\Cref{cor:plugin_dstar}.

\begin{remark}[Necessity of the threshold margin]
\label{rem:plugin_margin}
Without a threshold margin the plug-in statistic is not certifiable at any
finite depth. Deciding whether $\pi_{\LLM}(y\mid\fNL,x)$ lies above or below
$\tau$ at an input where it sits within $\xi$ of the threshold requires
$\Omega(\xi^{-2})$ draws. Thus no finite depth uniformly resolves inputs
arbitrarily close to $\tau$. The term $\kappa(\xi)$ bounds the deployment mass
of these inputs. The audit must fix $\xi$ in advance; optimizing it over a
data-dependent grid would require a union bound over the grid. When the API exposes the decoding probabilities exactly, $r$ is
irrelevant and \Cref{prop:master_statistic} applies as stated.
The margin mass $\kappa(\xi)$ is the analogue, at the admissibility threshold
$\tau$, of a Mammen--Tsybakov low-noise/margin quantity
\citep{mammen1999smooth}: it measures the deployment mass sitting where the
plug-in decision is statistically hard, and the same margin geometry governs
the fast rates available for plug-in classifiers
\citep{audibert2007fast}.
\end{remark}

\begin{proof}[Proof of \Cref{cor:plugin_dstar}]
Fix $\xi>0$. Call an input $x$ \emph{$\xi$-safe} if every label satisfies
$|\pi_{\LLM}(y\mid\fNL,x)-\tau|>\xi$; by definition the unsafe inputs carry mass
$\kappa(\xi)$.

\emph{Step 1 (per-input decision error).} At a $\xi$-safe input, each empirical
frequency $\widehat\pi_i(y)$ concentrates around $\pi_{\LLM}(y\mid\fNL,x)$, so a
two-sided Hoeffding bound over the $r$ decoding draws, taken in union across the
$C=|\Ycal|$ labels, gives
$\Pr[\widehat A_\tau(x)\neq A_\tau(x)]\le 2Ce^{-2r\xi^2}$; on the complementary
event the plug-in overlap indicator agrees with the population one. For the
joint draw of $z\sim P$ and its decoding samples, the indicators therefore
disagree only on the unsafe mass or through this depth error,
\[
    \Pr\bigl[\mathbf 1\{|\widehat A_\tau(z)|\ge2\}
        \neq\mathbf 1\{|A_\tau(z)|\ge2\}\bigr]
    \;\le\;\kappa(\xi)+2Ce^{-2r\xi^2},
\]
so their means satisfy
$|\E[\widehat D^\star_{\mathrm{plug}}]-D^\star_\tau|\le\kappa(\xi)+2Ce^{-2r\xi^2}$.

\emph{Step 2 (aggregation).} The $N$ plug-in indicators are i.i.d.\
$\{0,1\}$-valued, each a function of the independent pair (input, decoding
draws), so a two-sided Hoeffding bound gives
$|\widehat D^\star_{\mathrm{plug}}-\E[\widehat D^\star_{\mathrm{plug}}]|
\le\sqrt{\log(2/\delta)/2N}$ with probability at least $1-\delta$. Combining the
two steps by the triangle inequality yields the stated deviation bound
$|\widehat D^\star_{\mathrm{plug}}-D^\star_\tau|\le\eps_{N,r}(\xi)$, and
transporting it to $\Diam_{\tau,\zeta}$ through \Cref{thm:representation} adds
only the deterministic $2\zeta$ at the upper end, giving the confidence band.

\emph{Step 3 (depth prescription).} When $\kappa(\xi)>0$, taking
$r\ge\ln(2C/\kappa(\xi))/(2\xi^2)$ forces $2Ce^{-2r\xi^2}\le\kappa(\xi)$, so the
band inflates over the observed-$A_\tau$ band of \Cref{prop:master_statistic} by
at most $2\kappa(\xi)$; when $\kappa(\xi)=0$, taking
$r\ge\ln(2C/b)/(2\xi^2)$ makes the depth term at most any fixed budget $b>0$.
\end{proof}

\subsection{Proofs for the end-to-end audit}
\label{appssec:endtoend_proofs}

\begin{proof}[Proof of \Cref{thm:end_to_end}]
Two-sided Hoeffding on the i.i.d.\ indicators $\mathbf 1\{|A_\tau(z_i)|\ge2\}$
gives $|\widehat D^\star - D^\star_\tau| \le \eps_N$ with probability at least
$1-\delta$. We argue on this event.

\emph{(a)} The maximal-spread pair $(f_\tau^+, f_\tau^-)$ consists of
almost-everywhere admissible selections, so both members lie in
$\Fcal_{\tau,\zeta}(\fNL)$ (\Cref{thm:representation}). As a deterministic
functional of $(\pi_{\LLM}, \tau, \prec)$, the pair is fixed independently of
the audit inputs, so no sample splitting is required and the full budget
$\delta$ funds the single estimate $\widehat D^\star$; this is the $K=2$ case
of \Cref{thm:floor_certificate}. Since
$\{f_\tau^+(z) \neq f_\tau^-(z)\} = \{|A_\tau(z)| \ge 2\}$, the pair's
disagreement mass is $D_{+-} = D^\star_\tau \ge \widehat D^\star - \eps_N$.
Target-blindness (\Cref{prop:deployed_blindness}) makes the two observation
laws coincide, so by the two-point floor of \Cref{thm:semantic_floor}(b) the
worst-case risk is at least $D_{+-}/2 \ge \tfrac12(\widehat D^\star - \eps_N)$,
uniformly in $m$; since risk is nonnegative, it is therefore at least the
positive part $\tfrac12(\widehat D^\star - \eps_N)_+$ reported by the theorem.

\emph{(b)} The representation bracket (\Cref{thm:representation}) gives
$D^\star_\tau \ge \Diam_{\tau,\zeta} - 2\zeta$. Since the reported floor
$\tfrac12(\widehat D^\star - \eps_N)_+$ is at least the untruncated quantity
$\tfrac12(\widehat D^\star - \eps_N)$, it suffices to bound the latter:
\[
    \tfrac12(\widehat D^\star - \eps_N)
    \;\ge\; \tfrac12 D^\star_\tau - \eps_N
    \;\ge\; \tfrac12(\Diam_{\tau,\zeta} - 2\zeta) - \eps_N .
\]
In the other direction, any pair $f_1, f_2 \in \Fcal_{\tau,\zeta}(\fNL)$ has
floor $\tfrac12 d_P(f_1, f_2) \le \tfrac12 \Diam_{\tau,\zeta}
\le \tfrac12 D^\star_\tau + \zeta$, which exceeds the reported floor by at most
$\zeta + \eps_N$. This is the diameter-bracket consequence of part~(b); it
bounds every admissible pair's floor at once and involves no data-dependent
search over pairs.
\end{proof}

\begin{corollary}[Two-sided certified interval for the blind-channel value]
\label{cor:deployed_interval}
Let $\mathcal F_K\subseteq\Fcal_{\tau,\zeta}(\fNL)$ be a finite admissible family
containing the maximal-spread pair, with blind-channel minimax risk
$V_{\mathrm{blind}}(\mathcal F_K)$. Then $V_{\mathrm{blind}}(\mathcal F_K)$ is
independent of $m$ and
\[
    \tfrac12 D^\star_\tau
    \;\le\;V_{\mathrm{blind}}(\mathcal F_K)\;\le\;
    D^\star_\tau+2\zeta ;
\]
consequently, with probability at least $1-\delta$,
\[
    \tfrac12\bigl(\widehat D^\star-\eps_N\bigr)_+
    \;\le\;V_{\mathrm{blind}}(\mathcal F_K)\;\le\;
    \widehat D^\star+\eps_N+2\zeta .
\]
\end{corollary}

\begin{proof}
Write $D:=\Diam_P(\mathcal F_K)$. Target-blindness makes the observation law the
same for every target in $\mathcal F_K$, so $V_{\mathrm{blind}}(\mathcal F_K)$ is
independent of $m$; the two-point floor gives
$V_{\mathrm{blind}}(\mathcal F_K)\ge\tfrac12 D$, while the data-free learner
returning $f_\tau^+$ has risk at most $D$. Hence
$\tfrac12 D\le V_{\mathrm{blind}}(\mathcal F_K)\le D$. The maximal-spread pair
lies in $\mathcal F_K$, so $D\ge d_P(f_\tau^+,f_\tau^-)=D^\star_\tau$, and the
representation bracket gives $D\le D^\star_\tau+2\zeta$; combining with the
previous display yields the population interval
$\tfrac12 D^\star_\tau\le V_{\mathrm{blind}}(\mathcal F_K)\le D^\star_\tau+2\zeta$.
The empirical interval follows by substituting
$|D^\star_\tau-\widehat D^\star|\le\eps_N$, which holds with probability at least
$1-\delta$.
\end{proof}

\begin{proof}[Proof of \Cref{prop:plugin_vstar_certificate}]
At each input the plug-in and population admissible sets differ with probability
at most $q_r(\xi)$, by the argument proving \Cref{cor:plugin_dstar}. Both
$\phi(A)=1-1/|A|$ (with $\phi(\varnothing)=0$), of range $[0,1-\tfrac1C]$, and
$\mathbf 1\{|A|\ge2\}$, of range $[0,1]$, change only when the set changes, so
the plug-in means are biased by
\[
    \bigl|\E\phi(\widehat A_\tau(Z))-V^\star_\tau\bigr|\le(1-\tfrac1C)q_r(\xi),
    \qquad
    \bigl|\E\mathbf 1\{|\widehat A_\tau(Z)|\ge2\}-D^\star_\tau\bigr|\le q_r(\xi).
\]
The per-input decoding experiments are independent, so the plug-in summands are
i.i.d.; one-sided Hoeffding bounds at level $\delta/2$ each, combined with these
biases, give
\[
    V^\star_\tau\ge\widehat V^\star_{\mathrm{plug}}-(1-\tfrac1C)(\eps_N+q_r(\xi)),
    \qquad
    D^\star_\tau\ge\widehat D^\star_{\mathrm{plug}}-\eps_N-q_r(\xi),
\]
simultaneously with probability at least $1-\delta$ by a union bound. On this
event $V^\star_\tau\ge\tfrac12 D^\star_\tau$ (\Cref{cor:deployed_minimax_bounds})
dominates the second branch and the first display the first, so
$V^\star_\tau\ge C^{\mathrm{plug}}_{\mathrm{joint}}$ after taking positive parts
and the maximum; the same corollary transfers the bound to
$R_m(\Fcal_{\tau,\zeta})$.
\end{proof}

\section{Additional Experimental Diagnostics}
\label{app:experiment_diagnostics}

This appendix collects diagnostics that support interpretation and
reproducibility but do not enter the certified floors reported in
\Cref{sec:experiments}: the controlled pipeline check, a
threshold-sensitivity example, the model-free
ChaosNLI check and its complete table, the sampled-decoding audit design and
cost--precision trade-off, the margin-mass estimation procedure, the
per-configuration Qwen results, the exposed-probability threshold sweep, the
coherent-reading bridge audit design and results, and the exposed-probability
channel definition and audit metadata.

\subsection{Controlled pipeline check}
\label{appssec:controlled_check}
\label{ssec:controlled_audits}

This probe checks the certificate arithmetic on constructed cases; it does not
validate the model-induced admissibility construct. On $6{,}000$ binary inputs we
fix two pairs of known disagreement: a divisibility control whose two readings
never disagree ($D=0$), and a pedagogy pair with two coherent readings
(immediate performance versus long-term retention) disagreeing by $D=0.4$. On the
fixed-description target-blind channel the observed law is independent of the
operative reading (\Cref{prop:deployed_blindness}), so the certified floor is
$D_L/2$ (\Cref{thm:floor_certificate}; $N=400$, $\delta=0.10$).
\Cref{tab:audit_certificate} recovers the prescribed outcome: a zero certificate
for the control and a conservative positive floor of $0.166$ for the pedagogy
pair, valid at every learner sample size $m$. A threshold sweep on a related
controlled construction follows in \Cref{appssec:threshold_sensitivity}.

\begin{table}[t]
\centering
\small
\begin{tabular}{@{}llrrr@{}}
\toprule
\textbf{Task} & \textbf{Pair} & $D_{\mathrm{pop}}$ & $D_L$ &
\textbf{certified floor } $D_L/2$ \\
\midrule
Divisibility (control) & $6$ vs.\ $2{\wedge}3$
    & $0.000$ & $0.000$ & $0.000$ \\
Pedagogy (ambiguous) & $f_1$ vs.\ $f_2$
    & $0.400$ & $0.332$ & $0.166$ \\
\bottomrule
\end{tabular}
\caption{Controlled arithmetic and pipeline check of the finite-sample
blind-channel floor certificate ($\tau=0.20$, $\delta=0.10$). $D_{\mathrm{pop}}$
is the known population disagreement fixed by construction, not an empirical
estimate; $D_L$ is the audited finite-sample lower bound.}
\label{tab:audit_certificate}
\end{table}

\subsection{Threshold sensitivity}
\label{appssec:threshold_sensitivity}

Because the overlap mass $D^\star_\tau=\Pr_x[|A_\tau(x)|\ge2]$ is a threshold
functional of the decoding law, the certified floor depends on the admissibility
threshold $\tau$, and the two should be reported together. On a controlled
population in which a $0.40$ fraction of inputs are two-way ambiguous by construction
(decoding law $(0.45,0.45,0.10)$) and the rest near-confident
$(0.85,0.10,0.05)$, sweeping $\tau$ at $N=4000$, $\delta=0.10$ yields a stable
plateau: the certified floor is $0.49$ for $\tau\le0.10$ (where even the
near-confident second label clears the threshold and the count over-reads), holds
at $0.19$ across $\tau\in[0.12,0.44]$ (where it tracks exactly the
constructed ambiguous mass of $0.40$), and falls to $0$ for $\tau\ge0.50$ (where no second
label clears). The body's choice $\tau=0.20$ lies in this plateau. The
construction illustrates the threshold dependence established in the body; it
does not imply that empirical decoding laws generally exhibit a comparable
plateau.

\subsection{Model-free arithmetic check on ChaosNLI}
\label{ssec:construct_validity}

As a check against ambiguity documented independently of any model, we use
ChaosNLI~\citep{nie2020chaosnli}, which re-annotates SNLI and MNLI with
${\approx}100$ human labels per item, designed to characterize persistent
human disagreement~\citep{pavlick2019inherent,plank2022problem}. We form a
human-label analogue $q_{\mathrm{human}}(\cdot\mid x)$ from each item's empirical
label distribution and apply the exact-value functional to it in place of
$\pi_{\LLM}$; this exercises the certificate arithmetic on documented human
disagreement rather than a model-induced estimand. Under $q_{\mathrm{human}}$ the
overlap statistic and human entropy agree in ordering: MNLI is more ambiguous
than SNLI. The three-label
distributions also exercise the $k\ge3$ branch of the exact-value certificate,
absent in binary problems; at $\delta=0.10$ it raises the certified floor by
$0.005$ on MNLI and leaves it unchanged on SNLI. This model-free analogue is
conditional on the published label distributions; it neither validates the
model-relative admissibility construct nor constitutes a deployment study.

\subsection{ChaosNLI multiclass check: full table}
\label{appssec:chaosnli_full}

\Cref{tab:chaosnli} gives the complete multiclass calculation, computed from
the public ${\approx}100$-annotation ChaosNLI label
distributions~\citep{nie2020chaosnli} at the prespecified $\tau=0.20$,
$\delta=0.10$. Three labels are admissible on $1.3\%$ of SNLI items and $5.9\%$
of MNLI-matched items, yielding $V^\star_\tau=0.263$ and $0.386$, respectively.
The corresponding half-overlap values are $0.261$ and $0.376$. At confidence
$1-\delta=0.90$, the exact-value branch certifies $0.242$ on SNLI and $0.366$ on
MNLI, while the half-overlap branch certifies $0.245$ and $0.361$. Sampling
uncertainty therefore favors the half-overlap branch on SNLI, whereas the
greater three-way admissibility on MNLI makes the exact-value certificate larger
by $0.005$.

\begin{table}[h]
\centering
\small
\begin{tabular}{@{}lrrrrrrr@{}}
\toprule
\textbf{Split} & $n$ & $\bar H$ & $\Pr(k=3)$ & $\tfrac12D^\star$ &
$V^\star$ & $C_{\rm half}$ & $C_V$ \\
\midrule
SNLI          & $1514$ & $0.798$ & $0.013$ & $0.261$ & $0.263$ & $0.245$ & $0.242$ \\
MNLI-matched  & $1599$ & $1.072$ & $0.059$ & $0.376$ & $0.386$ & $0.361$ & $0.366$ \\
\bottomrule
\end{tabular}
\caption{ChaosNLI multiclass check at $\tau=0.20$ and $\delta=0.10$, computed
from the public ${\approx}100$-annotation label distributions (no model).
$C_{\rm half}=\tfrac12(\widehat D^\star-\eps_N)_+$ and
$C_V=(\widehat V^\star-\tfrac23\eps_N)_+$ are the two branches of
\Cref{prop:vstar_certificate}. MNLI has enough three-way admissibility for the
exact-value branch to improve the certified floor; SNLI does not. The comparison
is conditional on the published human label distributions and is not a model
calibration study.}
\label{tab:chaosnli}
\end{table}

\subsection{Sampled-decoding audit: design and cost}
\label{appssec:sampled_design}

This appendix gives the full design of the sampled-decoding depth-barrier
diagnostic summarized in \Cref{ssec:qwen_audit}. The audit fixes Qwen~2.5--3B,
two prompt paraphrases, $\tau=\csTau$, and $\delta=\csDelta$. Each paraphrase is
evaluated under three decoding seeds, giving $\csConfigs$ configurations per
task. In each, the model receives the description and input but neither candidate
reading; we draw $r=\csR$ labels at each of $N=\csN$ audit inputs and set
\[ \widehat A_\tau(x)=\{y:\widehat\pi(y\mid\fNL,x)\ge\tau\}. \]
The model's sampled decoding law thus determines admissibility: the experiment
scores no researcher-supplied label set.

\emph{Task.} The candidate-positive task is borderline content moderation under
two prespecified coherent readings (a narrow one limited to explicit threats,
slurs, or targeted abuse, and a broader one covering likely escalation of
personal harm), and the control is the exact rule ``LONG iff length exceeds 10.''

\emph{Sampling design.} The $N+N_\kappa=\csN+\csNkappa$ moderation inputs are
programmatically constructed borderline scenarios. The split was fixed
independently of model outputs and subsequent analysis choices: a recorded
design seed partitions the declared finite input population into disjoint audit
and margin subsets before any model call. Because the inputs are drawn without
replacement from this finite population, the Hoeffding radius stays valid, and
is indeed conservative %
\citep[\S6]{hoeffding1963probability,bardenet2015concentration}, and every
estimand is read as a finite-population mean over the synthetic, deliberately
borderline distribution, whose overlap mass plausibly exceeds that of natural
moderation traffic. Exact per-label Clopper--Pearson intervals on the $r$
decodings at each of the disjoint $N_\kappa$ margin inputs give the margin bound
$\kappa_U=\csKappaU$, constructed in \Cref{appssec:margin_estimator}.

\emph{Confidence allocation.} Following \Cref{rem:end_to_end_scope}, the
confidence budget $\delta=\csDelta$ splits as $\delta_1=\delta_2=\delta/2$
between the margin-mass bound and the overlap estimate, and the reported
certificate is valid at their union $\delta$.

\begin{table}[t]
\centering
\small
\begin{tabular}{@{}lrrrrrr@{}}
\toprule
\textbf{Task} & $\widehat D^\star$ & $\widehat V^\star$ &
$\eps_N$ & $\kappa_U$ & $2Ce^{-2r\xi^2}$ & \textbf{certified floor} \\
\midrule
Moderation (representative) & \csDhatPos & \csVhatPos & \csEpsN & \csKappaU & \csDepthTerm & \csFloorPos \\
Length control (matched) & \csDhatCtrl & \csVhatCtrl & \csEpsN & \csKappaU & \csDepthTerm & \csFloorCtrl \\
\bottomrule
\end{tabular}
\caption{Model-relative sampled-decoding audit of \csModel\ ($N=\csN$ audit and
$N_\kappa=\csNkappa$ disjoint margin inputs, $r=\csR$, $\tau=\csTau$,
$\xi=\csXi$, $\delta=\csDelta$), one matched prompt--seed configuration
(paraphrase P1, seed~$0$); all twelve per-configuration cells are in
\Cref{tab:qwen_per_config}. The certified floor is zero in every
configuration because the finite-depth term is capped at $1$ at $r=\csR$.}
\label{tab:qwen_model_relative}
\end{table}

The certified floor is zero in every configuration
(\Cref{tab:qwen_model_relative}) because at $r=\csR$ the finite-depth term
$2Ce^{-2r\xi^2}$ of the radius in \Cref{cor:plugin_dstar} caps at $1$; these
zeros are the vacuous-correction case of \Cref{rem:end_to_end_scope}, not
evidence that the task or judge is benign.
A positive floor at $\xi=\csXi$ requires $r\ge\ln(2C/b)/(2\xi^2)$ decodings
for depth budget $b$, about $r\approx877$ at $b=0.05$, and the margin
estimator is depth-limited in the same way. The audit does yield observed
plug-in overlap values: $\widehat D^\star$ clears
$\eps_N=\csEpsN$ in $\csDetectConfigs$ of the $\csConfigs$ moderation
configurations, while no control configuration does; the positive cells appear in
\Cref{tab:qwen_per_config}. These are prompt-sensitive point estimates, not
certificates; the
per-configuration cells, the constant-output check, and the paraphrase-collapse
pattern appear in \Cref{appssec:per_config}.

At $N=\csN$ the sampling slack $\eps_N=\csEpsN$ prevents the audit from
resolving overlap masses below roughly one eighth of the deployment
distribution, regardless of decoding depth: under the confidence split
(\Cref{rem:end_to_end_scope}), slack $\eps$ requires
$N\ge\log(4/\delta)/(2\eps^2)$, so $\eps=0.05$ needs about $738$ inputs.
\Cref{fig:cost_precision} summarizes the resulting cost--precision trade-off:
for each total call budget $Q$ it plots the smallest certificate radius
achievable by any breadth--depth split, alongside the exposed-probability
radius, which needs one call per input. The shallow audit has a capped, hence
vacuous, radius at $Q=\cpShallowQ$, whereas the exposed-probability audit
reaches radius $\cpExposedRad$ with $Q=\cpExposedQ$ calls; matching it by sampled
decoding at $\xi=0.05$ takes roughly three orders of magnitude more calls.
Exposed admissible-label probabilities remove both plug-in corrections
(\Cref{prop:master_statistic}), giving the more call-efficient access mode for
open-weight judges with logits-exposing backends.

\begin{figure}[t]
\centering
\pgfplotstableread[col sep=comma]{%
Q,exposed,sampled_xi0.05,N_xi0.05,r_xi0.05,sampled_xi0.15,N_xi0.15,r_xi0.15
100,0.12238734153404082,1.0,1,100,1.0,2,50
121,0.11126121957640075,1.0,1,121,1.0,2,60
147,0.10094337797154447,1.0,1,147,1.0,2,73
178,0.09173319081172296,1.0,1,178,1.0,2,89
215,0.08346747414905542,1.0,1,215,0.9479663772040172,3,71
261,0.07575586258251352,1.0,1,261,0.8638624702822847,3,87
316,0.06884825859322519,1.0,1,316,0.7933762575354346,4,79
383,0.0625370101756983,1.0,1,383,0.7346989527460603,4,95
464,0.05681689693688514,1.0,1,464,0.671052867926352,5,92
562,0.05162598411902672,1.0,1,562,0.6153314458451513,6,93
681,0.04689893117795282,1.0,1,681,0.5641727443909349,7,97
825,0.04260980306579049,1.0,1,825,0.5189858238133588,8,103
1000,0.038702275602049495,1.0,2,500,0.4739053944997068,10,100
1212,0.03515484531548869,1.0,2,606,0.4345308287148518,12,101
1468,0.031942836126756836,1.0,2,734,0.4000839532882429,14,104
1778,0.029024892702428917,0.9913759443905887,3,592,0.3665040555312013,17,104
2154,0.026370213852634994,0.8944935973926826,3,718,0.3361096192657634,20,107
2610,0.023956107187146813,0.832604349940518,4,652,0.3079392901459041,23,113
3162,0.021764844489536302,0.7560695649718574,4,790,0.28183089257587984,29,109
3831,0.019773357811404258,0.6941999244276994,5,766,0.25804444491577117,33,116
4642,0.017963209426103105,0.6382943118256958,6,773,0.23636402067913262,40,116
5623,0.016321214011657748,0.5854859418929367,7,803,0.2167026410382027,48,117
6813,0.014827478630133983,0.536933758642271,8,851,0.19860566846357552,54,126
8254,0.01347113749935963,0.49351550439786845,9,917,0.1815969930178675,66,125
10000,0.012238734153404082,0.4519637049490621,11,909,0.16626661802113923,80,125
12115,0.011119232006339855,0.41472524160540586,13,931,0.15234816346867244,93,130
14678,0.010101899872364158,0.3803403780861502,16,917,0.13934123629474648,112,131
17783,0.009177702738940663,0.348869440890226,18,987,0.12782719126636566,133,133
21544,0.008338219643992414,0.31948385293513487,22,979,0.11677209405399819,156,138
26102,0.007575296022584397,0.2928961485412067,26,1003,0.1068242990200207,189,138
31623,0.006882321672525712,0.26830921917967726,31,1020,0.0977639854058669,224,141
38312,0.006252721555659168,0.24589538150144102,37,1035,0.08940578934877129,266,144
46416,0.0056807103452896695,0.22526542111219058,43,1079,0.08178545932670976,320,145
56234,0.005161037480029973,0.20631101192350945,52,1081,0.07482245196698559,385,146
68129,0.004688894854377528,0.18897734278273454,61,1116,0.06842238080914335,454,150
82540,0.004259947717128148,0.17302382651073436,73,1130,0.06256212809839495,543,152
100000,0.0038702275602049492,0.15839870406734496,87,1149,0.05721506043010864,645,155
121153,0.0035161663625869243,0.14499685971800108,103,1176,0.052328305144250284,776,156
146780,0.003194501229163498,0.1327240681834556,123,1193,0.0478262541209449,923,159
177828,0.0029022607547703303,0.1214589168339311,146,1218,0.04372898899849941,1104,161
215443,0.002636758212224531,0.11115418867925762,175,1231,0.03997442845554394,1313,164
261016,0.002395537293409507,0.10171599963199268,206,1267,0.03653329277290331,1572,166
316228,0.0021763880898177717,0.09305401167231524,248,1275,0.03338912927049791,1882,168
383119,0.001977286749578891,0.08513031246609032,294,1303,0.030515519813882685,2240,171
464159,0.0017964002769896162,0.07787421263763857,354,1311,0.027883057909622328,2683,173
562341,0.0016320619015052368,0.07122225722258044,419,1342,0.025479967306588112,3213,175
681292,0.0014827565684936346,0.0651374679193256,502,1357,0.02328027773702403,3849,177
825404,0.0013471104858001409,0.05956804268897503,598,1380,0.021270256971701304,4611,179
1000000,0.0012238734153404082,0.054472308117440874,712,1404,0.019432289530596794,5494,182
1211528,0.001111910351558477,0.049796826488150016,849,1427,0.017751522884199858,6584,184
1467799,0.0010101903313536331,0.04552512028707482,1020,1439,0.01621540919857645,7891,186
1778279,0.0009177756929309138,0.041616657640583817,1213,1466,0.014811604902640673,9408,189
2154435,0.0008338151914199496,0.038043035478630145,1443,1493,0.013527946908496772,11279,191
2610157,0.0007575358420437749,0.03476934191888531,1732,1507,0.012354714414560514,13524,193
3162278,0.0006882345612674784,0.03177732815964549,2079,1521,0.011283049019784878,16134,196
3831187,0.0006252732164030839,0.02904011433128697,2478,1546,0.010303569806823332,19349,198
4641589,0.0005680717076581128,0.026536542603411584,2962,1567,0.009408667652810947,23207,200
5623413,0.0005161031514477912,0.024247478799660296,3550,1584,0.008590942612986152,27838,202
6812921,0.0004688887627898032,0.02215402748121302,4229,1611,0.007843968021212941,33396,204
8254042,0.0004259936878921138,0.020239954832765236,5070,1628,0.007161480369695928,39874,207
10000000,0.0003870227560204949,0.018490206354401555,6079,1645,0.006538070864790368,47846,209
}\costprecisiontable
\begin{tikzpicture}
\begin{loglogaxis}[
    width=0.78\textwidth,
    height=6.2cm,
    xlabel={Total model calls $Q$},
    ylabel={Certificate radius (smaller certifies more)},
    xmin=90, xmax=1.2e7,
    ymin=3e-4, ymax=1.6,
    grid=both,
    grid style={line width=.1pt, draw=gray!10},
    major grid style={line width=.2pt, draw=gray!25},
    legend cell align={left},
    legend columns=2,
    legend style={font=\small,
      at={(0.5,-0.28)}, anchor=north, /tikz/every even column/.append style={column sep=0.4cm}},
    axis lines=left,
]
\addplot[color=blue, thick]
  table [x=Q, y expr=\thisrow{sampled_xi0.05}, col sep=comma]
  {\costprecisiontable};
\addlegendentry{Sampled decoding, $\xi=0.05$ (best $N{\cdot}r$ split)}
\addplot[color=blue!50, thick, dashed]
  table [x=Q, y expr=\thisrow{sampled_xi0.15}, col sep=comma]
  {\costprecisiontable};
\addlegendentry{Sampled decoding, $\xi=0.15$ (margin term omitted)}
\addplot[color=red, thick]
  table [x=Q, y=exposed, col sep=comma]
  {\costprecisiontable};
\addlegendentry{Exposed probabilities}
\addplot[only marks, mark=square*, color=blue, mark size=2.5pt]
  coordinates {(\cpShallowQ, 1.0)};
\addlegendentry{Shallow audit ($N{=}100$, $r{=}3$; vacuous)}
\addplot[only marks, mark=*, color=red, mark size=2.5pt]
  coordinates {(\cpExposedQ, \cpExposedRad)};
\addlegendentry{Exposed audit run ($N{=}100$)}
\end{loglogaxis}
\end{tikzpicture}
\Description{A log-log line plot of certificate radius versus total model
calls. The exposed-probability curve decreases steadily from about 0.12 at
100 calls to below 0.001 at ten million calls. The two sampled-decoding
curves stay at the vacuous value 1 until several hundred calls, then decrease
but remain one to three orders of magnitude above the exposed-probability
curve at every budget. Two markers show the manuscript's actual runs: the
shallow sampled audit at 300 calls with a vacuous radius, and the exposed
audit at 100 calls with radius 0.1224.}
\caption{Cost--precision trade-off of the audit at $\delta=0.10$ for a binary
label set ($C=2$). For each total call budget $Q$ the sampled-decoding curves
plot the smallest radius of \Cref{cor:plugin_dstar} over breadth--depth splits;
the margin term $\kappa(\xi)$ and the disjoint margin-split calls are
task-dependent and omitted, so these curves are \emph{optimistic} for sampled
decoding. The exposed-probability curve plots the correction-free radius of
\Cref{prop:master_statistic}, which needs one call per input. Markers show the
two audits actually run in this section.}
\label{fig:cost_precision}
\end{figure}
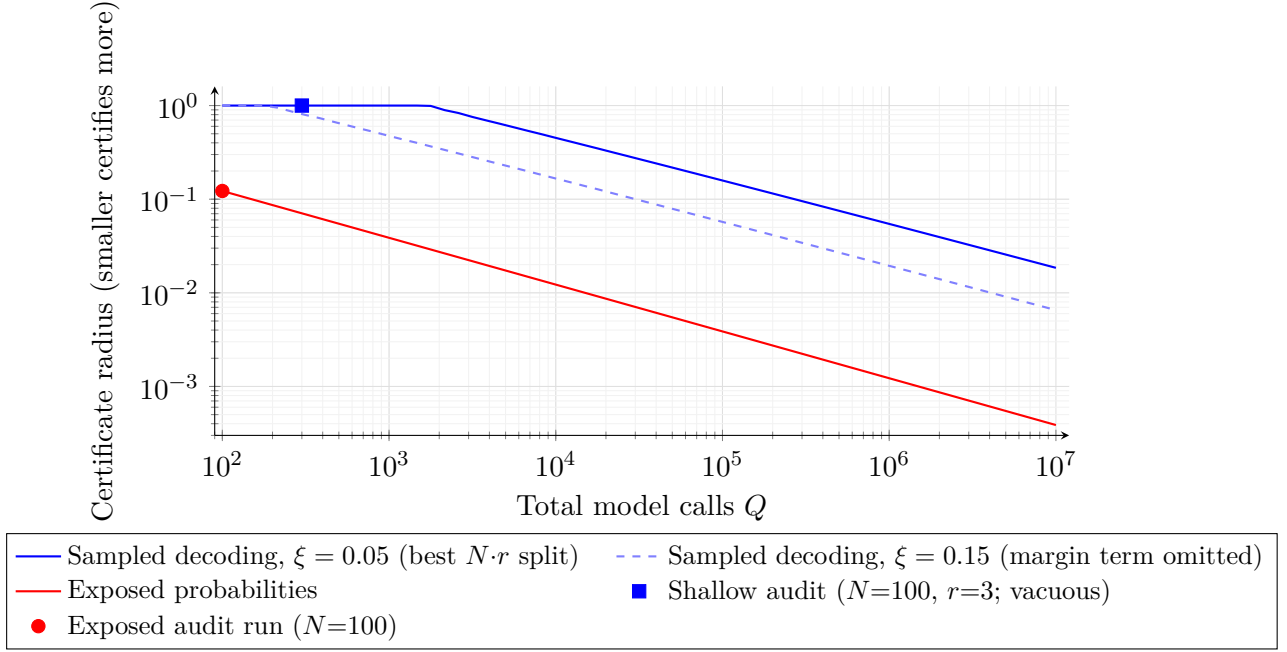

\subsection{Margin-mass estimator}
\label{appssec:margin_estimator}

The margin bound $\kappa_U$ reported in \Cref{ssec:qwen_audit} is estimated on
the disjoint $N_\kappa=\csNkappa$ margin inputs, with the confidence budget
split between it and the overlap estimate. The audit implementation bounds the
margin event using exact per-label Clopper--Pearson intervals. An input is
classified as potentially within the margin if any label-probability interval
intersects $[\tau-\xi,\tau+\xi]$. A union bound allocates the noncoverage budget
across labels. Alternatively, the auditor may supply an externally justified
upper bound on $\kappa(\xi)$. At
$r=\csR$, the exact three-draw intervals classify nearly every input as
potentially within the margin, yielding the near-vacuous upper bound
$\kappa_U=\csKappaU$ reported in the body. The partial high-depth cache is
incomplete and therefore cannot support a reproducible margin-mass estimate.
Determining whether the margin or finite-depth term dominates at
certificate-valid depth requires a completed high-depth run.

\subsection{Per-configuration Qwen audit results}
\label{appssec:per_config}

\Cref{tab:qwen_per_config} reports every cell of the two-prompt,
three-seed Qwen~2.5--3B audit of \Cref{ssec:qwen_audit}: all six moderation
configurations and the six matched length-control configurations, with
certificates recomputed under the full radius of \Cref{cor:plugin_dstar}. As
elsewhere, the finite-depth correction is vacuous at $r=\csR$
(\Cref{ssec:qwen_audit}), so the informative columns are the plug-in statistics
rather than the (zero) certified floor. The configurations show two patterns.
First, moderation produces large plug-in
overlap under P1 with seeds $0$ and $2$. P1 with seed $1$ has overlap $0.01$,
while P2 produces constant output under all three seeds (majority fraction
$1.00$). The estimated admissible sets are therefore sensitive to both prompt
paraphrase and shallow-decoding seed. Second, the control has small nonzero
plug-in overlap in five of six configurations, with $\widehat D^\star\le0.08$,
and never approaches the two large moderation signals. These control values are
consistent with finite-decoding noise at $r=\csR$. At $r=\csR$ decodings per input,
per-cell plug-in probabilities are coarse (multiples of $1/\csR$), so the
seed-to-seed variation within P1 may largely reflect Monte Carlo resolution
rather than decoding-law instability.

\begin{table}[t]
\centering
\small

\caption{Per-configuration results for the model-relative Qwen~2.5--3B audit
with $N=\csN$, $N_\kappa=\csNkappa$, $r=\csR$, $\tau=\csTau$, $\xi=\csXi$, and
$\delta=\csDelta$. ``Majority frac.''\ denotes the pooled majority-label
fraction; a value of $1.00$ denotes constant output. Every certified floor is
zero, as explained in the text, and the independently estimated bound $\kappa_U$
enters the same radius.}
\label{tab:qwen_per_config}
\end{table}

\subsection{Exposed-probability threshold sweep}
\label{appssec:exposed_sweep}

The exposed-probability audit of \Cref{ssec:qwen_exposed_audit} plots the
full sweep in \Cref{fig:qwen_exposed} and reports its correction-free
certified floor at the comparison point $\tau=\exTau$; this subsection records
the numerical values. The
probabilities swept here are the declared-label conditional probabilities
obtained by renormalizing the raw next-token log-probabilities over the
prespecified class tokens, so every floor below is relative to that constrained
first-token channel (\Cref{ssec:qwen_exposed_audit}), not an unrestricted
continuation channel. The full prespecified sweep for the moderation prompt P1 is
monotone decreasing in $\tau$. Across $\tau\in\{0.10,0.15,0.20,0.25,0.30,0.40\}$ the P1 floor is
$0.1638$, $0.1138$, $0.0838$, $0.0588$, $0.0388$, and $0$, respectively. At
$N=\exN$ the Hoeffding radius is $\eps_N\approx0.12$, so the positive floors
arise only where $\widehat D^\star$ clears that margin; the floor vanishes once
$\tau$ is large enough that fewer inputs place both class tokens above
threshold.

\subsection{Coherent-reading bridge audit: design and results}
\label{appssec:bridge_audit}

This subsection gives the full design behind the transfer boundary of
\Cref{ssec:qwen_eta_bridge}, which tests the two conditions of
\Cref{thm:coherent_bridge} separately on the moderation task.

\emph{Readings.} Two policy clauses are prespecified. The \emph{narrow}
reading: ``Remove a comment only when it contains an explicit threat, slur, or
targeted abusive statement. Allow it otherwise.'' The \emph{broad} reading:
``Remove a comment when it is likely to escalate personal harm, including
veiled or implicit targeted harassment even without explicit prohibited words.
Allow it otherwise.'' Each clause is inserted into a fixed reading template
(``Apply this single policy reading consistently to every input: \dots''), and
the induced labeling $f_a(x)$ is the argmax of the exposed declared-label law
under that clause. The base admissible sets $A_\tau(x)$ come from the frozen
exposed-probability P1 audit of \Cref{ssec:qwen_exposed_audit} at
$\tau=\exTau$.

\emph{Split and optimization.} A design-seeded permutation of the declared
input population assigns $\etaFitN$ fit and $\etaHoldoutN$ held-out inputs; the
fit inputs are verified (by hash) to coincide with the frozen P1 audit inputs.
The mixture weight is fitted on the fit split alone by a grid search over
$\lambda\in\{0,10^{-4},\dots,1\}$ minimizing the empirical coverage gap
$\E_x[\max_y\mu^\lambda_x(y)-1/k(x)]$, yielding
$\widehat\lambda_{\rm narrow}=\etaLambdaNarrow$. All reported quantities are
then evaluated on the held-out split only.

\emph{Confidence allocation.} The budget $\delta=\exDelta$ funds three
one-sided Hoeffding bounds (one for the coverage gap and one per-reading
admissibility violation) via the shared radius
$\sqrt{\log(3/\delta)/(2\,\etaHoldoutN)}\approx0.1304$.

\emph{Results.} The two conditions of the bridge come apart. The held-out
coverage gap is $\etaHoldoutMean$, giving the clipped upper bound
$\widehat\eta_U=\etaUpper$: the fitted mixture \emph{satisfies} the
$\eta$-uniform-coverage condition (\Cref{def:eta_coverage}), helped by the
large held-out reading disagreement ($\etaHoldoutDisagreement$). The
admissibility condition fails: the held-out violation rates are
$\etaNarrowViolation$ (narrow) and $\etaBroadViolation$ (broad), giving
$\widehat\zeta_U=\etaZetaUpper$ and hence the nearly vacuous gap bound
$|V^\star_\tau-V_{\rm blind}(\mathcal C)|\le\etaGapUpper$. The obstruction is
thus not mixture coverage but that each clause, applied globally, frequently
labels outside the model's admissible set; the clauses were written
independently of the model's admissible geometry and carry no independent
human validation.

\emph{Preregistered external-reading validation (next-step design).} The
protocol this result motivates is fixed as follows. (i)~Elicit a pool of
$K\ge5$ candidate policy readings from annotators or policy documents,
independently of any model output. (ii)~Freeze the clauses, the prompt
template, the split sizes, $\tau$, $\delta$, and the analysis code, and
register hashes before any model call. (iii)~On the fit split, fit the mixture
weights; on the held-out split, report $\widehat\eta_U$ and per-reading
$\widehat\zeta_U$ at radius $\sqrt{\log((K{+}1)/\delta)/(2N_{\rm holdout})}$.
(iv)~The bridge is informative iff both bounds are small; readings with large
violation rates are reported, not discarded post hoc. This design tests
whether \Cref{thm:coherent_bridge} becomes informative once held-out
admissibility error is small, as anticipated in the discussion.

\subsection{Exposed-probability channel and audit metadata}
\label{appssec:exposed_channel}

The exposed-probability audit of \Cref{ssec:qwen_exposed_audit} serves a frozen
Qwen~2.5--3B (\texttt{llama.cpp} b9870) and instantiates the verbalizer $V$ of
\Cref{ssec:basic_objects} as the \emph{constrained} first-token label kernel:
$V(\cdot\mid s)$ places all mass on the label whose class token leads $s$,
conditioned on that token lying in the declared class-token set
$\{v_y\}_{y\in\Ycal}$. Under this $V$ the admissibility channel is exactly the
constrained first-token law
\[
    \pi_{\LLM}(y\mid\fNL,x)
    \;=\;
    \frac{\Pr(T_1=v_y\mid u)}{\sum_{y'\in\Ycal}\Pr(T_1=v_{y'}\mid u)},
    \qquad u=\Pi(\fNL,x),
\]
i.e.\ the raw class-token probabilities renormalized over the declared label
set, not an unrestricted continuation probability. The certificate is stated
relative to this declared-label channel and needs neither the finite-depth nor
the threshold-margin correction.

As a collapse diagnostic, not itself a certified quantity, we report the mean
binary entropy (nats) of the declared-label law: $\exEntropyPone$ for
paraphrase P1 against $\exEntropyPtwo$ for P2. The order-of-magnitude gap shows
that P2's zero certificate arises from the declared-label probability
concentrating on one label (\emph{prompt collapse}), not from small estimated
overlap at comparable spread.

Each model-relative audit fixes its design and split before any model call and
records the seeds, prompts, per-input probability vectors, label-token
identifiers, model hash, and platform and runtime revision. The recorded
vectors are the raw next-token class-token probabilities renormalized over the
declared label set (the constrained first-token channel above), not
full-vocabulary softmax values. The supplementary reproducibility archive
provides the source code, frozen designs, versioned outputs, and SHA-256 hashes
for the exposed-probability and bridge result artifacts.

\section{Further Related Work}
\label{app:related}

\paragraph{Decision-theoretic status of the exact blind-channel value.}
Once the supervision experiment is target-independent,
the reduction to a randomized rule and a least-favorable prior is classical
\citep{wald1950statistical,blackwell1951comparison,blackwell1954games,ferguson1967decision,sion1958minimax}.
Modern adversarial and minimax-risk classifiers solve the analogous robust
$0$--$1$ decision problem for broad uncertainty sets of joint distributions
\citep{fathony2016adversarial,mazuelas2023minimax}. Most notably, the adversarial zero-one loss of
\citet[Theorem~1, Eq.~(4)]{fathony2016adversarial} is the maximum over
nonempty label subsets $S$ of hyperplanes whose constant term is
$(|S|-1)/|S|$. The same expression arises here as the value of the finite
zero--one game on $A_\tau(x)$, namely $1-1/|A_\tau(x)|$. Candidate-label learning is closer
to our pointwise set-valued object: partial-label learning asks when training
candidate sets can be disambiguated~\citep{cour2011partial}, but it is not
formulated as the target-blind, sample-flat selector-class identity used here. The NL-PAC-specific
step is the cyclic finite family whose marginal at each $x$ is uniform on
$A_\tau(x)$; it exposes the least-favorable value as
$\E[1-1/|A_\tau(X)|]$ and makes that value estimable from the same held-out
admissibility audit.

\paragraph{Superset learning and conformal prediction sets.}
The admissible set $A_\tau(x)$ is formally a candidate-label (superset) set, so
the superset/imprecise-label line is a near neighbor of our selection-class
view. Beyond the average-loss formulation of \citet{cour2011partial}, the
generalized-loss and data-disambiguation program of
\citet{huellermeier2014imprecise} learns by minimizing a loss over the set of
observation-consistent labelings, the same set-valued object we analyze. Its aim
is to recover a single sharp model, whereas NL-PAC asks whether the
mediated channel can distinguish the set's members at all. A superficially
similar set, the conformal prediction set, is a threshold set of a score
function and hence formally resembles $A_\tau(x)$
\citep{vovk2005algorithmic,angelopoulos2023conformal}; the two are conceptually
opposite, however. A conformal set is a property of the \emph{predictor} and
drives a coverage guarantee, while $A_\tau(x)$ is a property of the
\emph{supervision channel} and drives an identification floor: the former
certifies that the truth is contained, the latter certifies that the truth
cannot be pinned down.

\paragraph{Relation to non-identifiable noise and partial monitoring.}
The blind-channel floor (\Cref{thm:semantic_floor}) is an LLM-mediated instantiation of information-theoretic noise-floor lower bounds, such as those in the malicious-error~\citep{kearnsli1993malicious} and nasty-noise~\citep{bshouty2002nasty} models. However, while those floors are budgeted (vanishing as the corruption rate $\eta\to0$), the NL-PAC floor is set by the indistinguishability of two admissible readings under the oracle channel (\Cref{eq:unresolved}).

Instance-dependent noise (IDN) similarly makes identifiability depend on
structural assumptions: bounded instance- and label-dependent noise can be
learned under additional conditions~\citep{cheng2020bounded}, whereas the
unrestricted target--noise decomposition is not identified by labels alone. The
systematic component of the NL-PAC oracle error is a structured instance of IDN
where the noise is generated by an LLM interpreting an ambiguous prompt. This
also distinguishes NL-PAC from weak-supervision frameworks like
Snorkel~\citep{ratner2017snorkel}, which aggregate multiple heuristics by
estimating their accuracies, whereas we analyze whether the labeling channel can
distinguish the admissible readings at all.

A supervision-channel separation between direct and mediated labels (a natural follow-up outside this paper's scope) has a direct precedent in partial monitoring, where "hopeless" games with linear regret arise when feedback cannot separate target outcomes~\citep{bartok2011minimax}; NL-PAC transfers this logic to an offline specification setting.

\paragraph{Harm specification and LMaaS opacity.}
The blind-channel floor has an independent information-theoretic counterpart on the alignment side: Young~\citeyearpar{young2025harm} argues that no specification $I$ can pin down an external target $O$ if the conditional entropy $H(O\mid I)>0$, so an information gap forces a residual error. Our ambiguity diameter $\Diam_{\tau,\zeta}(\fNL,P)>0$ is the operational witness of this gap, and the master statistic (\Cref{prop:master_statistic}) estimates it from data. Judge-mediated supervision is increasingly delivered through proprietary, API-only models, the Language-Models-as-a-Service (LMaaS) paradigm, whose opacity is documented as a systematic obstacle to evaluation, reproducibility, reliability, and trustworthiness~\citep{lamalfa2024lmaas}; the operative reading is then often unobservable by construction, which is exactly the reading-blind regime the floor and its audit are designed to certify.

\paragraph{Rating indeterminacy and judge bias.}
Within the LLM-as-a-judge literature (e.g., G-Eval~\citealp{liu2023geval}, Prometheus~\citealp{kim2023prometheus}, and surveys like \citealp{gu2024survey}), NL-PAC treats judge-mediated evaluation as a partially identified target-selection problem rather than a judge-alignment problem. Recent large-scale audits report that judge reliability and validity can come apart: high test--retest consistency can coexist with severe positional bias~\citep{norman2026reliability}. Instrument-centric diagnostics likewise treat judge reliability as a property of the measurement instrument, not of individual outputs~\citep{choi2026irtjudge}. This parallels NL-PAC's treatment of the floor as a property of the tuple $(\LLM,\Pi,\tau,P)$, not of any single judgment.

Guerdan et al.~\citeyearpar{guerdan2025rating} document \emph{rating indeterminacy}, in which rating criteria admit multiple valid interpretations, and show that forced-choice judge validation is systematically biased relative to multi-label response sets; their contribution is a theoretical and empirical validation framework for judge systems, whereas NL-PAC proves and certifies the learnability floor that such indeterminacy imposes. Dorner et al.~\citeyearpar{dorner2025limits} prove that, when the judge is no more accurate than the evaluated model, judge-based debiasing cannot reduce the required ground-truth sample size by more than half, and Feuer et al.~\citeyearpar{feuer2026biasbounded} study bias-bounded guarantees for measurable judge bias. Both quantify \emph{exogenous} judge error; NL-PAC measures the endogenous ambiguity limit generated by the model's own interpretation.

\paragraph{Learning from disagreement.}
The JAIR survey of \citet{uma2021disagreement} organizes methods that retain
multiple human judgments rather than reducing them to one gold label, and
\citet{baan2022calibration} show why majority-label calibration is problematic
when human disagreement is inherent. Perspectivist NLP develops the same
contention into annotator- and distribution-aware modeling and
evaluation~\citep{xu2026perspectivist}. These approaches learn from or evaluate against
human judgment variation. NL-PAC instead asks what can be guaranteed when the
supervision channel does not reveal which retained reading is operative. Its
validated reading family is therefore conceptually distinct from the
model-admissible selector set: agreement between them is an empirical bridge
condition, not a definition.

\paragraph{Prompt underspecification.}
Systems work on prompt sensitivity shows that prompts omit requirements that models fill with fragile defaults: underspecified prompts regress markedly across model or prompt changes~\citep{yang2025underspec}, and even subjective evaluations can vary across prompt formulations. This continues the broader observation that underspecified pipelines admit many equally predictive solutions that diverge under deployment shift~\citep{damour2020underspecification}. Active task disambiguation~\citep{kobalczyk2025active} and elicitation~\citep{li2023eliciting} clarify intent by asking questions or generating edge cases. NL-PAC supplies the complementary accounting layer by pricing the remaining ambiguity via $D^\star_\tau$.

\paragraph{Partial concept classes.}
Partial-concept PAC learning allows a concept to be undefined away from the
support of its source distribution and studies which such classes remain
learnable~\citep{alon2022partial}. NL-PAC has a different partiality: each
admissible reading is a total target on the deployment domain, but the
natural-language specification identifies a set of such targets and the mediated
channel may fail to distinguish them. Thus the blind-channel floor is a
channel-induced identification obstruction, not a consequence of undefined
off-support labels.

\paragraph{Privileged information, semantic uncertainty, and partial-identification inference.}
Vapnik's learning-using-privileged-information (LUPI) program treats teacher comments as training-only signals that accelerate learning~\citep{vapnik2009new}. While it asks how much trusted explanations help, our question is dual: how much error remains when the explanation channel is systematically ambiguous. The semantic-uncertainty literature (e.g., \citealp{kuhn2023semantic}) measures uncertainty over meanings by clustering answers; such concentration signals yield only one-sided diameter diagnostics, and the sharp handle for floors remains the master statistic $D^\star_\tau$. Finally, the certificate in \Cref{thm:floor_certificate} draws on classical inference for partially identified parameters~\citep{imbens2004confidence} to certify the Le Cam floor from finite data; simultaneous-inference corrections over searched families in the sense of \citet{berk2013valid} are deliberately unnecessary here, because the canonical candidate pair is a deterministic functional of the frozen model rather than a data-selected object.


\begin{thebibliography}{64}
\providecommand{\natexlab}[1]{#1}
\providecommand{\url}[1]{\texttt{#1}}
\expandafter\ifx\csname urlstyle\endcsname\relax
  \providecommand{\doi}[1]{doi: #1}\else
  \providecommand{\doi}{doi: \begingroup \urlstyle{rm}\Url}\fi

\bibitem[Alon et~al.(2022)Alon, Hanneke, Holzman, and Moran]{alon2022partial}
Noga Alon, Steve Hanneke, Ron Holzman, and Shay Moran.
\newblock A theory of {PAC} learnability of partial concept classes.
\newblock In \emph{2021 {IEEE} 62nd Annual Symposium on Foundations of Computer
  Science ({FOCS})}, pages 658--671. IEEE, 2022.
\newblock \doi{10.1109/FOCS52979.2021.00070}.

\bibitem[Angelopoulos and Bates(2023)]{angelopoulos2023conformal}
Anastasios~N. Angelopoulos and Stephen Bates.
\newblock Conformal prediction: A gentle introduction.
\newblock \emph{Foundations and Trends in Machine Learning}, 16\penalty0
  (4):\penalty0 494--591, 2023.
\newblock \doi{10.1561/2200000101}.

\bibitem[Audibert and Tsybakov(2007)]{audibert2007fast}
Jean-Yves Audibert and Alexandre~B. Tsybakov.
\newblock Fast learning rates for plug-in classifiers.
\newblock \emph{The Annals of Statistics}, 35\penalty0 (2):\penalty0 608--633,
  2007.
\newblock \doi{10.1214/009053606000001217}.

\bibitem[Baan et~al.(2022)Baan, Aziz, Plank, and
  Fern{\'a}ndez]{baan2022calibration}
Joris Baan, Wilker Aziz, Barbara Plank, and Raquel Fern{\'a}ndez.
\newblock Stop measuring calibration when humans disagree.
\newblock In \emph{Proceedings of the 2022 Conference on Empirical Methods in
  Natural Language Processing}, pages 1892--1915, 2022.
\newblock \doi{10.18653/v1/2022.emnlp-main.124}.

\bibitem[Bardenet and Maillard(2015)]{bardenet2015concentration}
R{\'e}mi Bardenet and Odalric-Ambrym Maillard.
\newblock Concentration inequalities for sampling without replacement.
\newblock \emph{Bernoulli}, 21\penalty0 (3):\penalty0 1361--1385, 2015.
\newblock \doi{10.3150/14-BEJ605}.

\bibitem[Bart{\'o}k et~al.(2011)Bart{\'o}k, P{\'a}l, and
  Szepesv{\'a}ri]{bartok2011minimax}
G{\'a}bor Bart{\'o}k, D{\'a}vid P{\'a}l, and Csaba Szepesv{\'a}ri.
\newblock Minimax regret of finite partial-monitoring games in stochastic
  environments.
\newblock In Sham~M. Kakade and Ulrike von Luxburg, editors, \emph{Proceedings
  of the 24th Annual Conference on Learning Theory}, volume~19 of
  \emph{Proceedings of Machine Learning Research}, pages 133--154, Budapest,
  Hungary, 2011. PMLR.
\newblock URL \url{https://proceedings.mlr.press/v19/bartok11a.html}.

\bibitem[Berk et~al.(2013)Berk, Brown, Buja, Zhang, and Zhao]{berk2013valid}
Richard Berk, Lawrence Brown, Andreas Buja, Kai Zhang, and Linda Zhao.
\newblock Valid post-selection inference.
\newblock \emph{The Annals of Statistics}, 41\penalty0 (2):\penalty0 802--837,
  2013.
\newblock \doi{10.1214/12-aos1077}.

\bibitem[Blackwell(1951)]{blackwell1951comparison}
David Blackwell.
\newblock Comparison of experiments.
\newblock In Jerzy Neyman, editor, \emph{Proceedings of the Second Berkeley
  Symposium on Mathematical Statistics and Probability}, pages 93--102,
  Berkeley, 1951. University of California Press.
\newblock \doi{10.1525/9780520411586-009}.

\bibitem[Blackwell and Girshick(1954)]{blackwell1954games}
David Blackwell and M.~A. Girshick.
\newblock \emph{Theory of Games and Statistical Decisions}.
\newblock Wiley, New York, 1954.

\bibitem[Bshouty et~al.(2002)Bshouty, Eiron, and Kushilevitz]{bshouty2002nasty}
Nader~H. Bshouty, Nadav Eiron, and Eyal Kushilevitz.
\newblock {PAC} learning with nasty noise.
\newblock \emph{Theoretical Computer Science}, 288\penalty0 (2):\penalty0
  255--275, 2002.
\newblock \doi{10.1016/S0304-3975(01)00403-0}.

\bibitem[Cheng et~al.(2020)Cheng, Liu, Ramamohanarao, and
  Tao]{cheng2020bounded}
Jiacheng Cheng, Tongliang Liu, Kotagiri Ramamohanarao, and Dacheng Tao.
\newblock Learning with bounded instance and label-dependent label noise.
\newblock In Hal Daum{\'e}~III and Aarti Singh, editors, \emph{Proceedings of
  the 37th International Conference on Machine Learning}, volume 119 of
  \emph{Proceedings of Machine Learning Research}, pages 1789--1799. PMLR,
  2020.
\newblock URL \url{https://proceedings.mlr.press/v119/cheng20c.html}.

\bibitem[Choi et~al.(2026)Choi, Park, Cho, Park, and Kim]{choi2026irtjudge}
Junhyuk Choi, Sohhyung Park, Chanhee Cho, Hyeonchu Park, and Bugeun Kim.
\newblock Diagnosing the reliability of {LLM}-as-a-judge via item response
  theory, 2026.

\bibitem[Cour et~al.(2011)Cour, Sapp, and Taskar]{cour2011partial}
Timothee Cour, Ben Sapp, and Ben Taskar.
\newblock Learning from partial labels.
\newblock \emph{Journal of Machine Learning Research}, 12\penalty0
  (42):\penalty0 1501--1536, 2011.
\newblock URL \url{http://jmlr.org/papers/v12/cour11a.html}.

\bibitem[D'Amour et~al.(2022)D'Amour, Heller, Moldovan, Adlam, Alipanahi,
  Beutel, et~al.]{damour2020underspecification}
Alexander D'Amour, Katherine Heller, Dan Moldovan, Ben Adlam, Babak Alipanahi,
  Alex Beutel, et~al.
\newblock Underspecification presents challenges for credibility in modern
  machine learning.
\newblock \emph{Journal of Machine Learning Research}, 23\penalty0
  (226):\penalty0 1--61, 2022.
\newblock URL \url{https://jmlr.org/papers/v23/20-1335.html}.

\bibitem[Dawid and Skene(1979)]{dawid1979maximum}
A.~P. Dawid and A.~M. Skene.
\newblock Maximum likelihood estimation of observer error-rates using the {EM}
  algorithm.
\newblock \emph{Journal of the Royal Statistical Society. Series C (Applied
  Statistics)}, 28\penalty0 (1):\penalty0 20--28, 1979.
\newblock \doi{10.2307/2346806}.

\bibitem[Dorner et~al.(2025)Dorner, Nastl, and Hardt]{dorner2025limits}
Florian~Eddie Dorner, Vivian Nastl, and Moritz Hardt.
\newblock Limits to scalable evaluation at the frontier: {LLM} as judge won't
  beat twice the data.
\newblock In \emph{International Conference on Learning Representations}, pages
  26467--26491, 2025.
\newblock URL
  \url{https://proceedings.iclr.cc/paper_files/paper/2025/hash/4264ee4376776907c0b87ed70b959585-Abstract-Conference.html}.

\bibitem[Fathony et~al.(2016)Fathony, Liu, Asif, and
  Ziebart]{fathony2016adversarial}
Rizal Fathony, Anqi Liu, Kaiser Asif, and Brian Ziebart.
\newblock Adversarial multiclass classification: A risk minimization
  perspective.
\newblock In D.~Lee, M.~Sugiyama, U.~von Luxburg, I.~Guyon, and R.~Garnett,
  editors, \emph{Advances in Neural Information Processing Systems}, volume~29.
  Curran Associates, Inc., 2016.
\newblock URL
  \url{https://proceedings.neurips.cc/paper_files/paper/2016/file/ad13a2a07ca4b7642959dc0c4c740ab6-Paper.pdf}.

\bibitem[Ferguson(1967)]{ferguson1967decision}
Thomas~S. Ferguson.
\newblock \emph{Mathematical Statistics: A Decision Theoretic Approach}.
\newblock Academic Press, New York, 1967.

\bibitem[Feuer et~al.(2026)Feuer, Rosenblatt, and
  Elachqar]{feuer2026biasbounded}
Benjamin Feuer, Lucas Rosenblatt, and Oussama Elachqar.
\newblock Towards provably unbiased llm judges via bias-bounded evaluation,
  2026.

\bibitem[Gilboa and Schmeidler(1989)]{gilboa1989maxmin}
Itzhak Gilboa and David Schmeidler.
\newblock Maxmin expected utility with non-unique prior.
\newblock \emph{Journal of Mathematical Economics}, 18\penalty0 (2):\penalty0
  141--153, 1989.
\newblock \doi{10.1016/0304-4068(89)90018-9}.

\bibitem[Gu et~al.(2026)Gu, Jiang, Shi, Tan, Zhai, Xu, Li, Shen, Ma, Liu, Wang,
  Zhang, Lin, Zhang, Ni, Gao, Wang, and Guo]{gu2024survey}
Jiawei Gu, Xuhui Jiang, Zhichao Shi, Hexiang Tan, Xuehao Zhai, Chengjin Xu, Wei
  Li, Yinghan Shen, Shengjie Ma, Honghao Liu, Saizhuo Wang, Kun Zhang, Zhouchi
  Lin, Bowen Zhang, Lionel Ni, Wen Gao, Yuanzhuo Wang, and Jian Guo.
\newblock A survey on {LLM}-as-a-judge.
\newblock \emph{The Innovation}, 7\penalty0 (6):\penalty0 101253, 2026.
\newblock \doi{10.1016/j.xinn.2025.101253}.

\bibitem[Guerdan et~al.(2025)Guerdan, Barocas, Holstein, Wallach, Wu, and
  Chouldechova]{guerdan2025rating}
Luke Guerdan, Solon Barocas, Kenneth Holstein, Hanna Wallach, Steven Wu, and
  Alexandra Chouldechova.
\newblock Validating {LLM}-as-a-judge systems under rating indeterminacy.
\newblock In D.~Belgrave, C.~Zhang, H.~Lin, R.~Pascanu, P.~Koniusz,
  M.~Ghassemi, and N.~Chen, editors, \emph{Advances in Neural Information
  Processing Systems}, volume~38, pages 112282--112350. Curran Associates,
  Inc., 2025.
\newblock URL
  \url{https://proceedings.neurips.cc/paper_files/paper/2025/file/a309239c11a28c597d050bd4a1752d32-Paper-Conference.pdf}.

\bibitem[Hanneke(2014)]{hanneke2014theory}
Steve Hanneke.
\newblock Theory of disagreement-based active learning.
\newblock \emph{Foundations and Trends in Machine Learning}, 7\penalty0
  (2--3):\penalty0 131--309, 2014.
\newblock \doi{10.1561/2200000037}.

\bibitem[Hoeffding(1963)]{hoeffding1963probability}
Wassily Hoeffding.
\newblock Probability inequalities for sums of bounded random variables.
\newblock \emph{Journal of the American Statistical Association}, 58\penalty0
  (301):\penalty0 13--30, 1963.
\newblock \doi{10.1080/01621459.1963.10500830}.

\bibitem[Howard et~al.(2021)Howard, Ramdas, McAuliffe, and
  Sekhon]{howard2021time}
Steven~R. Howard, Aaditya Ramdas, Jon McAuliffe, and Jasjeet Sekhon.
\newblock Time-uniform, nonparametric, nonasymptotic confidence sequences.
\newblock \emph{The Annals of Statistics}, 49\penalty0 (2):\penalty0
  1055--1080, 2021.
\newblock \doi{10.1214/20-aos1991}.

\bibitem[H{\"u}llermeier(2014)]{huellermeier2014imprecise}
Eyke H{\"u}llermeier.
\newblock Learning from imprecise and fuzzy observations: Data disambiguation
  through generalized loss minimization.
\newblock \emph{International Journal of Approximate Reasoning}, 55\penalty0
  (7):\penalty0 1519--1534, 2014.
\newblock \doi{10.1016/j.ijar.2013.09.003}.

\bibitem[Imbens and Manski(2004)]{imbens2004confidence}
Guido~W. Imbens and Charles~F. Manski.
\newblock Confidence intervals for partially identified parameters.
\newblock \emph{Econometrica}, 72\penalty0 (6):\penalty0 1845--1857, 2004.
\newblock \doi{10.1111/j.1468-0262.2004.00555.x}.

\bibitem[Kearns and Li(1993)]{kearnsli1993malicious}
Michael Kearns and Ming Li.
\newblock Learning in the presence of malicious errors.
\newblock \emph{SIAM Journal on Computing}, 22\penalty0 (4):\penalty0 807--837,
  1993.
\newblock \doi{10.1137/0222052}.

\bibitem[Kearns and Vazirani(1994)]{kearns1994introduction}
Michael~J. Kearns and Umesh~V. Vazirani.
\newblock \emph{An Introduction to Computational Learning Theory}.
\newblock MIT Press, 1994.

\bibitem[Kim et~al.(2024)Kim, Shin, cho, Jang, Longpre, Lee, Yun, Shin, Kim,
  Thorne, and Seo]{kim2023prometheus}
Seungone Kim, Jay Shin, yejin cho, Joel Jang, Shayne Longpre, Hwaran Lee,
  Sangdoo Yun, S~Shin, Ryan, Sungdong Kim, James Thorne, and Minjoon Seo.
\newblock Prometheus: Inducing fine-grained evaluation capability in language
  models.
\newblock In B.~Kim, Y.~Yue, S.~Chaudhuri, K.~Fragkiadaki, M.~Khan, and Y.~Sun,
  editors, \emph{International Conference on Learning Representations}, volume
  2024, pages 29927--29962, 2024.
\newblock URL
  \url{https://proceedings.iclr.cc/paper_files/paper/2024/file/803485352e61e3ebf41221e4776c9fd4-Paper-Conference.pdf}.

\bibitem[Kobalczyk et~al.(2025)Kobalczyk, Astorga, Liu, and van~der
  Schaar]{kobalczyk2025active}
Katarzyna Kobalczyk, Nicol{\'a}s Astorga, Tennison Liu, and Mihaela van~der
  Schaar.
\newblock Active task disambiguation with {LLM}s.
\newblock In Y.~Yue, A.~Garg, N.~Peng, F.~Sha, and R.~Yu, editors,
  \emph{International Conference on Learning Representations}, volume 2025,
  pages 37823--37847, 2025.
\newblock URL
  \url{https://proceedings.iclr.cc/paper_files/paper/2025/file/5e07476b6bd2497e1fbd11b8f0b2de3c-Paper-Conference.pdf}.

\bibitem[Kuhn et~al.(2023)Kuhn, Gal, and Farquhar]{kuhn2023semantic}
Lorenz Kuhn, Yarin Gal, and Sebastian Farquhar.
\newblock Semantic uncertainty: Linguistic invariances for uncertainty
  estimation in natural language generation.
\newblock In \emph{International Conference on Learning Representations}, 2023.
\newblock Spotlight.

\bibitem[La~Malfa et~al.(2024)La~Malfa, Petrov, Frieder, Weinhuber, Burnell,
  Nazar, Cohn, Shadbolt, and Wooldridge]{lamalfa2024lmaas}
Emanuele La~Malfa, Aleksandar Petrov, Simon Frieder, Christoph Weinhuber, Ryan
  Burnell, Raza Nazar, Anthony Cohn, Nigel Shadbolt, and Michael Wooldridge.
\newblock Language-models-as-a-service: Overview of a new paradigm and its
  challenges.
\newblock \emph{Journal of Artificial Intelligence Research}, 80:\penalty0
  1497--1523, 2024.
\newblock \doi{10.1613/jair.1.15865}.

\bibitem[Li et~al.(2025)Li, Tamkin, Goodman, and Andreas]{li2023eliciting}
Belinda Li, Alex Tamkin, Noah Goodman, and Jacob Andreas.
\newblock Eliciting human preferences with language models.
\newblock In Y.~Yue, A.~Garg, N.~Peng, F.~Sha, and R.~Yu, editors,
  \emph{International Conference on Learning Representations}, volume 2025,
  pages 80984--81013, 2025.
\newblock URL
  \url{https://proceedings.iclr.cc/paper_files/paper/2025/file/c9867d5a22653ce98b02595061e40f12-Paper-Conference.pdf}.

\bibitem[Liu et~al.(2023{\natexlab{a}})Liu, Wu, Michael, Suhr, West, Koller,
  Swayamdipta, Smith, and Choi]{liu2023ambiguity}
Alisa Liu, Zhaofeng Wu, Julian Michael, Alane Suhr, Peter West, Alexander
  Koller, Swabha Swayamdipta, Noah~A. Smith, and Yejin Choi.
\newblock We're afraid language models aren't modeling ambiguity.
\newblock In \emph{Conference on Empirical Methods in Natural Language
  Processing (EMNLP)}, 2023{\natexlab{a}}.
\newblock \doi{10.18653/v1/2023.emnlp-main.51}.

\bibitem[Liu and Dietterich(2014)]{liu2014learnability}
Liping Liu and Thomas Dietterich.
\newblock Learnability of the superset label learning problem.
\newblock In Eric~P. Xing and Tony Jebara, editors, \emph{Proceedings of the
  31st International Conference on Machine Learning}, volume~32 of
  \emph{Proceedings of Machine Learning Research}, pages 1629--1637, Bejing,
  China, 2014. PMLR.
\newblock URL \url{https://proceedings.mlr.press/v32/liug14.html}.

\bibitem[Liu et~al.(2023{\natexlab{b}})Liu, Iter, Xu, Wang, Xu, and
  Zhu]{liu2023geval}
Yang Liu, Dan Iter, Yichong Xu, Shuohang Wang, Ruochen Xu, and Chenguang Zhu.
\newblock {G-Eval}: {NLG} evaluation using {GPT}-4 with better human alignment.
\newblock In \emph{Proceedings of the 2023 Conference on Empirical Methods in
  Natural Language Processing (EMNLP)}, pages 2511--2522, 2023{\natexlab{b}}.
\newblock \doi{10.18653/v1/2023.emnlp-main.153}.

\bibitem[Mammen and Tsybakov(1999)]{mammen1999smooth}
Enno Mammen and Alexandre~B. Tsybakov.
\newblock Smooth discrimination analysis.
\newblock \emph{The Annals of Statistics}, 27\penalty0 (6):\penalty0
  1808--1829, 1999.
\newblock \doi{10.1214/aos/1017939240}.

\bibitem[Manski(2003)]{manski2003partial}
Charles~F. Manski.
\newblock \emph{Partial Identification of Probability Distributions}.
\newblock Springer Series in Statistics. Springer, New York, 2003.
\newblock \doi{10.1007/b97478}.

\bibitem[Maurer and Pontil(2009)]{maurer2009empirical}
Andreas Maurer and Massimiliano Pontil.
\newblock Empirical {Bernstein} bounds and sample-variance penalization.
\newblock In \emph{Proceedings of the 22nd Annual Conference on Learning Theory
  (COLT)}, 2009.
\newblock URL \url{https://www.cs.mcgill.ca/~colt2009/papers/012.pdf}.

\bibitem[Mazuelas et~al.(2023)Mazuelas, Romero, and
  Gr{\"u}nwald]{mazuelas2023minimax}
Santiago Mazuelas, Mauricio Romero, and Peter Gr{\"u}nwald.
\newblock Minimax risk classifiers with 0--1 loss.
\newblock \emph{Journal of Machine Learning Research}, 24\penalty0
  (208):\penalty0 1--48, 2023.
\newblock URL \url{https://www.jmlr.org/papers/volume24/22-0339/22-0339.pdf}.

\bibitem[Mitchell(1982)]{mitchell1982generalization}
Tom~M. Mitchell.
\newblock Generalization as search.
\newblock \emph{Artificial Intelligence}, 18\penalty0 (2):\penalty0 203--226,
  1982.
\newblock \doi{10.1016/0004-3702(82)90040-6}.

\bibitem[Molchanov and Molinari(2018)]{molchanov2018random}
Ilya Molchanov and Francesca Molinari.
\newblock \emph{Random Sets in Econometrics}.
\newblock Cambridge University Press, 2018.
\newblock \doi{10.1017/9781316392973}.

\bibitem[Natarajan et~al.(2013)Natarajan, Dhillon, Ravikumar, and
  Tewari]{natarajan2013learning}
Nagarajan Natarajan, Inderjit~S Dhillon, Pradeep~K Ravikumar, and Ambuj Tewari.
\newblock Learning with noisy labels.
\newblock In C.~J. Burges, L.~Bottou, M.~Welling, Z.~Ghahramani, and
  K.~Weinberger, editors, \emph{Advances in Neural Information Processing
  Systems}, volume~26. Curran Associates, Inc., 2013.
\newblock URL
  \url{https://proceedings.neurips.cc/paper_files/paper/2013/file/3871bd64012152bfb53fdf04b401193f-Paper.pdf}.

\bibitem[Nie et~al.(2020)Nie, Zhou, and Bansal]{nie2020chaosnli}
Yixin Nie, Xiang Zhou, and Mohit Bansal.
\newblock What can we learn from collective human opinions on natural language
  inference data?
\newblock In \emph{Proceedings of the 2020 Conference on Empirical Methods in
  Natural Language Processing (EMNLP)}, pages 9131--9143, 2020.
\newblock \doi{10.18653/v1/2020.emnlp-main.734}.

\bibitem[Norman et~al.(2026)Norman, Rivera, and Hughes]{norman2026reliability}
Justin~D. Norman, Michael~U. Rivera, and D.~Alex Hughes.
\newblock Reliability without validity: A systematic, large-scale evaluation of
  {LLM}-as-a-judge models across agreement, consistency, and bias, 2026.

\bibitem[Pan et~al.(2022)Pan, Bhatia, and Steinhardt]{pan2022reward}
Alexander Pan, Kush Bhatia, and Jacob Steinhardt.
\newblock The effects of reward misspecification: Mapping and mitigating
  misaligned models.
\newblock In \emph{International Conference on Learning Representations
  (ICLR)}, 2022.

\bibitem[Pavlick and Kwiatkowski(2019)]{pavlick2019inherent}
Ellie Pavlick and Tom Kwiatkowski.
\newblock Inherent disagreements in human textual inferences.
\newblock \emph{Transactions of the Association for Computational Linguistics},
  7:\penalty0 677--694, 2019.
\newblock \doi{10.1162/tacl_a_00293}.

\bibitem[Plank(2022)]{plank2022problem}
Barbara Plank.
\newblock The ``problem'' of human label variation: On ground truth in data,
  modeling and evaluation.
\newblock In \emph{Conference on Empirical Methods in Natural Language
  Processing (EMNLP)}, 2022.
\newblock \doi{10.18653/v1/2022.emnlp-main.731}.

\bibitem[Ratner et~al.(2017)Ratner, Bach, Ehrenberg, Fries, Wu, and
  R{\'e}]{ratner2017snorkel}
Alexander Ratner, Stephen~H. Bach, Henry Ehrenberg, Jason Fries, Sen Wu, and
  Christopher R{\'e}.
\newblock Snorkel: Rapid training data creation with weak supervision.
\newblock \emph{Proceedings of the VLDB Endowment}, 11\penalty0 (3):\penalty0
  269--282, 2017.
\newblock \doi{10.14778/3157794.3157797}.

\bibitem[Singleton and Booth(2024)]{singleton2024truth}
Joseph Singleton and Richard Booth.
\newblock Truth-tracking with non-expert information sources.
\newblock \emph{Journal of Artificial Intelligence Research}, 81:\penalty0
  619--641, 2024.
\newblock \doi{10.1613/jair.1.15273}.

\bibitem[Sion(1958)]{sion1958minimax}
Maurice Sion.
\newblock On general minimax theorems.
\newblock \emph{Pacific Journal of Mathematics}, 8\penalty0 (1):\penalty0
  171--176, 1958.
\newblock \doi{10.2140/pjm.1958.8.171}.

\bibitem[Stoye(2009{\natexlab{a}})]{stoye2009minimax}
J{\"o}rg Stoye.
\newblock Minimax regret treatment choice with finite samples.
\newblock \emph{Journal of Econometrics}, 151\penalty0 (1):\penalty0 70--81,
  2009{\natexlab{a}}.
\newblock \doi{10.1016/j.jeconom.2009.02.013}.

\bibitem[Stoye(2009{\natexlab{b}})]{stoye2009more}
J{\"o}rg Stoye.
\newblock More on confidence intervals for partially identified parameters.
\newblock \emph{Econometrica}, 77\penalty0 (4):\penalty0 1299--1315,
  2009{\natexlab{b}}.
\newblock \doi{10.3982/ECTA7347}.

\bibitem[Tsybakov(2009)]{tsybakov2009introduction}
Alexandre~B. Tsybakov.
\newblock \emph{Introduction to Nonparametric Estimation}.
\newblock Springer Series in Statistics. Springer New York, New York, NY, 1
  edition, 2009.
\newblock ISBN 978-0-387-79052-7.
\newblock \doi{10.1007/b13794}.
\newblock URL \url{https://link.springer.com/book/10.1007/b13794}.

\bibitem[Uma et~al.(2021)Uma, Fornaciari, Hovy, Paun, Plank, and
  Poesio]{uma2021disagreement}
Alexandra~N. Uma, Tommaso Fornaciari, Dirk Hovy, Silviu Paun, Barbara Plank,
  and Massimo Poesio.
\newblock Learning from disagreement: A survey.
\newblock \emph{Journal of Artificial Intelligence Research}, 72:\penalty0
  1385--1470, 2021.
\newblock \doi{10.1613/jair.1.12752}.

\bibitem[Valiant(1984)]{valiant1984theory}
Leslie~G. Valiant.
\newblock A theory of the learnable.
\newblock \emph{Communications of the {ACM}}, 27\penalty0 (11):\penalty0
  1134--1142, 1984.
\newblock \doi{10.1145/1968.1972}.

\bibitem[Vapnik and Vashist(2009)]{vapnik2009new}
Vladimir Vapnik and Akshay Vashist.
\newblock A new learning paradigm: Learning using privileged information.
\newblock \emph{Neural Networks}, 22\penalty0 (5--6):\penalty0 544--557, 2009.
\newblock \doi{10.1016/j.neunet.2009.06.042}.

\bibitem[Vovk et~al.(2005)Vovk, Gammerman, and Shafer]{vovk2005algorithmic}
Vladimir Vovk, Alexander Gammerman, and Glenn Shafer.
\newblock \emph{Algorithmic Learning in a Random World}.
\newblock Springer New York, New York, NY, 1 edition, 2005.
\newblock ISBN 978-0-387-25061-8.
\newblock \doi{10.1007/b106715}.
\newblock URL \url{https://link.springer.com/book/10.1007/b106715}.

\bibitem[Wald(1950)]{wald1950statistical}
Abraham Wald.
\newblock \emph{Statistical Decision Functions}.
\newblock Wiley, New York, 1950.

\bibitem[Waudby-Smith and Ramdas(2024)]{waudbysmith2024estimating}
Ian Waudby-Smith and Aaditya Ramdas.
\newblock Estimating means of bounded random variables by betting.
\newblock \emph{Journal of the Royal Statistical Society Series B: Statistical
  Methodology}, 86\penalty0 (1):\penalty0 1--27, 2024.
\newblock \doi{10.1093/jrsssb/qkad009}.

\bibitem[Xu and Jurgens(2026)]{xu2026perspectivist}
Yinuo Xu and David Jurgens.
\newblock Beyond consensus: Perspectivist modeling and evaluation of annotator
  disagreement in nlp, 2026.

\bibitem[Yang et~al.(2026)Yang, Shi, Ma, Liu, K{\"a}stner, and
  Wu]{yang2025underspec}
Chenyang Yang, Yike Shi, Qianou Ma, Michael~Xieyang Liu, Christian K{\"a}stner,
  and Tongshuang Wu.
\newblock What prompts don't say: Understanding and managing underspecification
  in {LLM} prompts.
\newblock In \emph{Findings of the Association for Computational Linguistics:
  ACL 2026}, pages 9072--9101. Association for Computational Linguistics, 2026.
\newblock \doi{10.18653/v1/2026.findings-acl.441}.

\bibitem[Young(2025)]{young2025harm}
Robin Young.
\newblock Information-theoretic distinctions between deception and confusion.
\newblock In Kentaro Inui, Sakriani Sakti, Haofen Wang, Derek~F. Wong, Pushpak
  Bhattacharyya, Biplab Banerjee, Asif Ekbal, Tanmoy Chakraborty, and
  Dhirendra~Pratap Singh, editors, \emph{Proceedings of the 14th International
  Joint Conference on Natural Language Processing and the 4th Conference of the
  Asia-Pacific Chapter of the Association for Computational Linguistics}, pages
  258--268, Mumbai, India, December 2025. The Asian Federation of Natural
  Language Processing and The Association for Computational Linguistics.
\newblock \doi{10.18653/v1/2025.findings-ijcnlp.15}.
\newblock URL \url{https://aclanthology.org/2025.findings-ijcnlp.15/}.

\end{thebibliography}
\end{document}